\documentclass[a4paper,11pt]{article}
\usepackage[margin=1in]{geometry}
\usepackage{import}
\usepackage{amsmath}
\usepackage{mathtools}
\usepackage{multirow}
\usepackage{makecell}
\usepackage{dsfont}
\usepackage{tablefootnote}
\usepackage{cite}
\usepackage[utf8]{inputenc}
\usepackage{csquotes}

\newcommand{\printfnsymbol}[1]{%
  \textsuperscript{\@fnsymbol{#1}}%
}

\usepackage[utf8]{inputenc} % allow utf-8 input
\usepackage[T1]{fontenc}    % use 8-bit T1 fonts
\usepackage{hyperref}       % hyperlinks
\usepackage{nicefrac}       % compact symbols for 1/2, etc.
\usepackage{microtype}      % microtypography
\usepackage[table]{xcolor}
\usepackage{framed}
\colorlet{shadecolor}{pink}
\usepackage{authblk}

\definecolor{Gray}{rgb}{0.8,0.8,0.8}

\usepackage{graphicx}
\usepackage{soul}
\usepackage{subcaption}
\usepackage{booktabs} % for professional tables
\usepackage{tablefootnote}

\usepackage{amsmath,amsthm,amssymb,amsfonts}
\usepackage{algorithm}
\usepackage{algorithmic}
\usepackage{enumerate}
\usepackage{cleveref}
\usepackage{comment}
\usepackage{bm}
\usepackage{pifont}

\theoremstyle{plain}
\newtheorem{theorem}{Theorem}[]
\newtheorem{corollary}{Corollary}[theorem]
\newtheorem{lemma}[]{Lemma}
\newtheorem{proposition}{Proposition}

\newtheorem{assumption}{Assumption}[]

\newtheorem{remark}{Remark}[]

\begin{document}
\title{Understanding Self-Distillation in the Presence of Label Noise}
\date{}
\author[*]{Rudrajit Das}
\author[*]{Sujay Sanghavi}
\affil[*]{UT Austin}
\maketitle

\begin{abstract}
\noindent Self-distillation (SD) is the process of first training a \enquote{teacher} model and then using its predictions to train a \enquote{student} model with the \textit{same} architecture. Specifically, the student's objective function is $\big(\xi*\ell(\text{teacher's predictions}, \text{ student's predictions}) + (1-\xi)*\ell(\text{given labels}, \text{ student's predictions})\big)$, where $\ell$ is some loss function and $\xi$ is some parameter $\in [0,1]$. Empirically, SD has been observed to provide performance gains in several settings. In this paper, we theoretically characterize the effect of SD in two supervised learning problems with \textit{noisy labels}. We first analyze SD for regularized linear regression and show that in the high label noise regime, the optimal value of $\xi$ that minimizes the expected error in estimating the ground truth parameter is surprisingly greater than 1. Empirically, we show that $\xi > 1$ works better than $\xi \leq 1$ even with the cross-entropy loss for several classification datasets when 50\% or 30\% of the labels are corrupted. Further, we quantify when optimal SD is better than optimal regularization. Next, we analyze SD in the case of logistic regression for binary classification with random label corruption and quantify the range of label corruption in which the student outperforms the teacher in terms of accuracy. To our knowledge, this is the first result of its kind for the cross-entropy loss.
\end{abstract}

\section{Introduction}
\label{intro}
The core idea of \textit{knowledge distillation} (KD), introduced in \cite{hinton2015distilling}, is to train a student model with a teacher model's \textit{predicted soft labels} (i.e., the output probability distribution over the classes for classification problems) in addition to the original hard labels (one-hot vectors for classification problems) on which the teacher is trained. 
%The rationale is to transfer \enquote{dark knowledge} embedded in the teacher's soft labels with the hope of improving the student's performance. 
The original rationale was to use a teacher with large statistical capacity to better model the underlying label distribution compared to the provided hard labels, and have the student with smaller capacity learn some mixture of the teacher's predicted label distribution (a.k.a. \enquote{dark knowledge}) and the provided label distribution. Specifically, the student's per-sample objective function in the KD framework is:
%\vspace{-0.1cm}
\begin{equation}
    \label{eq:1-intro}
    \xi * \ell\big(\bm{y}_T, \bm{y}_S(\bm{\theta})\big) + (1-\xi) * \ell\big(\bm{y}, \bm{y}_S(\bm{\theta})\big),
\end{equation}
where $\ell$ is some loss function (usually, regularized cross-entropy loss for classification problems), $\bm{y}_T$ is the teacher's predicted label, $\bm{y}$ is the given label on which the teacher is trained, $\bm{y}_S(\bm{\theta})$ is the prediction of the student model parameterized by $\bm{\theta}$, and $\xi \in [0,1]$ is known as the imitation parameter \cite{lopez2015unifying}\footnote{In this work, we set the temperature parameter suggested in \cite{hinton2015distilling} equal to 1.}. KD and its variants have been shown to be beneficial for model compression (i.e., distilling a bigger teacher model's knowledge into a smaller student model), semi-supervised learning, making models robust and improving performance in general \cite{li2017learning,furlanello2018born,sun2019patient,ahn2019variational,chen2020big,xie2020self,sarfraz2021knowledge,li2021align,pham2021meta,beyer2022knowledge,baykal2022robust}; see \cite{gou2021knowledge} for a survey on KD. 

{The focus of this work is on the special case of the student and teacher having the same architecture, which is known as \textbf{self-distillation} (following \cite{mobahi2020self}); we abbreviate it as \textbf{SD} henceforth. Since the teacher and student have the same capacity, one would expect the utility of the teacher's dark knowledge to be very limited, if any at all. However, surprisingly, \cite{furlanello2018born} show that SD (with ensembling) yields performance gains in both vision and language tasks with extensive experiments. Further, \cite{li2017learning} empirically demonstrate that SD can ameliorate learning in the presence of noisy labels. There are also a few works that theoretically investigate %the effects and benefits of 
SD, such as \cite{mobahi2020self,dong2019distillation}; we discuss these in detail in \Cref{rel-wrk}. The results of these papers are only with the squared loss and \textit{not} the \textit{cross-entropy loss} which is the de facto loss function for classification problems.}

{In this work, we theoretically analyze SD in the presence of label corruption (in the supervised setting) for the cross-entropy loss as well as the squared loss, characterizing its utility and unveiling some new insights including a recommendation for use in practice. We summarize our contributions next and survey the landscape of pertinent theoretical works on KD and SD in \Cref{rel-wrk}.
}
\\
\\
\textbf{Contributions:}
\\
\textbf{(a)} First, we consider linear regression with $\ell_2$-regularized squared loss in \Cref{lin-reg}. Here, the observed label $y$ for a sample $\bm{x}$ is: $y = \langle \bm{\theta}^{*}, \bm{x} \rangle + \eta$, where $\bm{\theta}^{*}$ is the underlying parameter and $\eta$ is zero-mean random label noise.
\vspace{-0.1 cm}
\begin{itemize}
    \item We show that self-distillation (SD) is associated with a \textbf{bias-variance tradeoff} in that increasing $\xi$ in \cref{eq:1-intro} reduces the variance but increases the bias in estimating $\bm{\theta}^{*}$ with respect to the randomness in label noise; see \Cref{thm-est-err} and \Cref{dist-tradeoff}.
    \vspace{-0.2 cm}
    \item {A \textbf{surprising algorithmic insight} from our analysis is that the value of $\xi$ that optimally balances this {bias-variance tradeoff} can be $> 1$, especially in the \textit{high label noise regime} (i.e., when $\mathbb{E}[\eta^2]$ is large); see \Cref{opt-xi} and \Cref{rmk-xi}. This can be interpreted as actively \textbf{anti-learning} (or going against) the observed (possibly noisy) labels. But as discussed after \cref{eq:1-intro}, $\xi$ is tuned in $[0,1]$ in practice. In \Cref{sec-expts-2}, we empirically corroborate our insight for multi-class classification with linear probing\footnote{i.e., learning a softmax layer on top of a pre-trained network} using the \textit{cross-entropy} loss by showing that $\xi > 1$ works better than $\xi \leq 1$ for several datasets with 50\% or 30\% of the training set's labels being corrupted in different ways.}
    \vspace{-0.2 cm}
    \item In \Cref{dist-utility}, we show that as the degree of label noise increases, the utility of the teacher's predictions in training the student increases. Intuitively, this happens because the noise component in the teacher's predictions is smaller compared to the original labels. We also empirically verify this insight for the \textit{cross-entropy} loss in {\Cref{dist-utility-ver}}. %where SD with $\xi=1$ (i.e., only using the teacher's predictions) consistently yields higher gains as the label corruption level increases.
    \vspace{-0.2 cm}
    \item {In \Cref{thm-opt-lambda}, we provide a condition when \textit{optimal} SD is 
    \textbf{better} than \textit{optimal} $\ell_2$ regularization ({optimal} means with the best parameters); this is the \textbf{first} such result. 
    }
\end{itemize}
%\vspace{-0.2 cm}
\textbf{(b)} Next, we look at logistic regression with $\ell_2$-regularized cross-entropy loss in \Cref{sec:log_reg}. We consider a balanced binary classification problem where some fraction, say $p < 0.5$, of the training set's labels are randomly flipped. Under some assumptions on the data geometry and the kernel function, \textit{we quantify the range of $p$ in which the student outperforms the teacher in terms of accuracy}; see \Cref{thm1}. To our knowledge, this is the \textbf{first result} that \textit{provably establishes the utility of SD in the presence of label noise for the \textbf{cross-entropy} loss}. The main technical challenge in the analysis is dealing with non-linear equations involving the sigmoid function. We tackle this by employing the first-order Maclaurin series expansion of the sigmoid function and by bounding the corresponding approximation errors; see Step 3 in the proof outline of \Cref{thm1}. Moreover, in \Cref{rmk-3}, we show that the student's predictions have smaller variability than the teacher's predictions which is akin to SD reducing variance in linear regression.

\section{Related Work}
\label{rel-wrk}
There is a growing body of works trying to theoretically explain KD/SD and its benefits. \cite{mobahi2020self} look at regression with the squared loss in Hilbert space, showing that SD essentially amplifies regularization. {However, unlike us, they do not explicitly consider the case of noisy labels/observations or discuss the bias-variance tradeoff associated with SD in the presence of label noise.} Moreover, they restrict their analysis to $\xi=1$; {so unlike us, they do not have any results on when optimal SD is better than optimal $\ell_2$ regularization}. \cite{dong2019distillation} claim that KD is effective in transferring dark knowledge by mimicking early stopping. Further, they propose their own SD algorithm that uses dynamically updated soft labels, and show that in the presence of noisy labels, their algorithm is able to learn the correct labels. In this work, we focus on the standard SD algorithm with fixed soft labels, and moreover, we quantify the range of label corruption in which SD improves accuracy. Unlike our work, \cite{dong2019distillation} do not quantify when their proposed algorithm improves upon the standard approach of using just hard labels. An important difference between our work and \cite{dong2019distillation} as well as \cite{mobahi2020self} is that the results of these two papers are with the squared loss, whereas we provide results with the cross-entropy loss in addition to squared loss. The cross-entropy loss is the customary choice for classification problems in practice and is also more challenging to analyze. On the note of cross-entropy loss, \cite{phuong2019towards} analyze the convergence of linear student networks trained with the cross-entropy loss, and also bound the expected difference between the predictions of the student and teacher. \cite{ji2020knowledge} also bound the expected difference between the predictions of the student and teacher for wide neural networks that evolve as linear networks under the NTK assumption. However, \cite{phuong2019towards} and \cite{ji2020knowledge} do not consider how the student might have better generalization than the teacher in the presence of noisy labels. \cite{menon2021statistical} statistically characterize \enquote{good} teachers for distilling knowledge to a student. \cite{kaplun2022knowledge} show that an ensemble of teachers trained with noisy labels can be used to label a new unlabeled dataset, which can be then employed to train a student with good performance. We focus on the (common) case of only one teacher and the student being trained on the same dataset as the teacher. {There are also some works such as \cite{cheng2020explaining,stanton2021does,pham2022revisiting} that empirically provide some insights on KD.}

\section{Linear Regression}
\label{lin-reg}
\textbf{Setting:} The \textit{observed} label $y \in \mathbb{R}$ is linearly related to the data $\bm{x} \in \mathcal{X} \subseteq \mathbb{R}^d$ as:
\begin{equation}
    \label{eq:1}
    y = \langle \bm{\theta}^{*}, \bm{x} \rangle + \eta,
\end{equation}
where $\bm{\theta}^{*} \in \mathbb{R}^d$ and $\eta \in \mathbb{R}$ is label noise. Here, $\langle \bm{\theta}^{*}, \bm{x} \rangle$ is the \textit{actual} label of $\bm{x}$.

The training set consists of $n$ pairs of data points (drawn from $\mathcal{X}$) and noisy labels $\{(\bm{x}_i, y_i)\}_{i=1}^n$. Let $\bm{X} := [\bm{x}_1, \ldots, \bm{x}_n] \in \mathbb{R}^{d \times n}$ be the data matrix and $\bm{Y} := [y_1, \ldots, y_n]^T \in \mathbb{R}^{n}$ be the label vector. Then, as per the above linear model (\cref{eq:1}):
\begin{equation}
    \label{eq:3}
    \bm{Y} = \bm{X}^T \bm{\theta}^{*} + \bm{\eta},
\end{equation}
for some noise vector $\bm{\eta} \in \mathbb{R}^{n}$. We make some standard assumptions on the noise vector $\bm{\eta}$.
\begin{assumption}
\label{asmp-noise}
$\bm{\eta}$ is independent of $\bm{X}$. Further, each coordinate of $\bm{\eta}$ has mean 0 and variance $\gamma^2$, and is independent of the other coordinates.
\end{assumption}

\noindent \textbf{Teacher Model:} The teacher tries to learn the underlying model, parameterized by $\bm{\theta} \in \mathbb{R}^d$, from $(\bm{X}, \bm{Y})$ by applying the squared loss with $\ell_2$ regularization. Specifically, the teacher's objective function is:
\begin{equation}
    \label{eq:4}
    f_T(\bm{\theta}) = \frac{1}{2}\|\bm{Y} - \bm{X}^T \bm{\theta}\|^2 + \frac{\lambda}{2} \|\bm{\theta}\|^2,
\end{equation}
where $\lambda > 0$ is the $\ell_2$-regularization parameter. Now, the \textit{model learned by the teacher} is\footnote{Throughout this work, we shall assume that we can converge to the exact optimum of the objective function. All the objective functions in this work are convex, and hence (stochastic) gradient descent will converge to the optimum in all the cases.}:
\begin{equation}
    \label{eq:4-2}
    \hat{\bm{\theta}}_{T} := \textup{arg min}_{\bm{\theta} \in \mathbb{R}^d} \text{ }f_T(\bm{\theta}) = (\bm{X}\bm{X}^T + \lambda \bm{I}_d)^{-1} \bm{X} \bm{Y},
\end{equation} 
where $\bm{I}_d$ is the identity matrix of size $d \times d$. %Plugging in the value of $\bm{Y}$ from \cref{eq:3} in \cref{eq:4-2}, we can express $\hat{\bm{\theta}}_{T}$ in terms of $\bm{\theta}^{*}$:
Plugging in $\bm{Y}$ from \cref{eq:3} in \cref{eq:4-2}, we get:
\begin{equation}
    \label{prop-1}
    \hat{\bm{\theta}}_{T} = (\bm{X}\bm{X}^T + \lambda \bm{I}_d)^{-1} \bm{X}(\bm{X}^T \bm{\theta}^{*} + \bm{\eta}).
\end{equation}

\noindent \textbf{Student Model Trained with Self-Distillation:} Following \cref{eq:1-intro}, here the student is trained with a weighted sum of (i) the $\ell_2$-regularized squared loss between the student's predictions and the \textit{teacher's predictions}, and (ii) the $\ell_2$-regularized squared loss between the student's predictions and the \textit{original labels} on which the teacher was trained. For the $i^{\text{th}}$ sample, the teacher's prediction is $\hat{y}_i = \langle \hat{\bm{\theta}}_{T}, \bm{x}_i \rangle$. Define $\hat{\bm{Y}} := [\hat{y}_1, \ldots, \hat{y}_n]^T \in \mathbb{R}^{n}$; note that $\hat{\bm{Y}} = \bm{X}^T \hat{\bm{\theta}}_{T}$. 

The student's objective function is:
\begin{flalign}
    \nonumber
    {f}_S({\bm{\theta}}; \xi) & = \xi\Big(\frac{1}{2}\|\hat{\bm{Y}} - {\bm{X}}^T \bm{\theta}\|^2 + \frac{\lambda}{2} \|\bm{\theta}\|^2\Big) + (1-\xi)\Big(\frac{1}{2}\|{\bm{Y}} - {\bm{X}}^T \bm{\theta}\|^2 + \frac{\lambda}{2} \|\bm{\theta}\|^2\Big)
    \\
    \label{eq:18-oct30}
    & = \xi\Big(\frac{1}{2}\|\hat{\bm{Y}} - {\bm{X}}^T \bm{\theta}\|^2\Big) + (1-\xi)\Big(\frac{1}{2}\|{\bm{Y}} - {\bm{X}}^T \bm{\theta}\|^2\Big) + \frac{\lambda}{2} \|\bm{\theta}\|^2,
\end{flalign}
where $\xi \in \mathbb{R}$ is known as the imitation parameter \cite{lopez2015unifying} and $\lambda > 0$ is the same regularization parameter that was used by the teacher. Even though it is standard practice to restrict $\xi \in [0,1]$, we do not impose this condition. Now, the \textit{model learned by the student} is:
\begin{flalign}
    \nonumber
    \hat{\bm{\theta}}_{S}(\xi) := \textup{arg min}_{\bm{\theta} \in \mathbb{R}^{{d}}} \text{ } {f}_S({\bm{\theta}}; \xi) & = ({\bm{X}}{\bm{X}}^T + \lambda \bm{I}_{{d}})^{-1} {\bm{X}} (\xi\hat{\bm{Y}} + (1-\xi)\bm{Y}) %= \xi \hat{\bm{\theta}}_{S} + (1-\xi) \hat{\bm{\theta}}_{T},
    \\
    \label{eq:19-oct30-new}
    & = \xi ({\bm{X}}{\bm{X}}^T + \lambda \bm{I}_{{d}})^{-1} {\bm{X}} \bm{X}^T \hat{\bm{\theta}}_{T} + (1-\xi) \hat{\bm{\theta}}_{T},
\end{flalign}
where \cref{eq:19-oct30-new} is obtained by using $\hat{\bm{Y}} = \bm{X}^T \hat{\bm{\theta}}_{T}$ and \cref{eq:4-2}. Note that $\xi = 0$ corresponds to the teacher, i.e. $\hat{\bm{\theta}}_{S}(0) = \hat{\bm{\theta}}_{T}$.

Finally, plugging in $\hat{\bm{\theta}}_{T}$ from \cref{prop-1} in \cref{eq:19-oct30-new}, we get:
\begin{equation}
    \label{prop-2}
    \hat{\bm{\theta}}_{S}(\xi) = \Big(\xi ({\bm{X}}{\bm{X}}^T + \lambda \bm{I}_{{d}})^{-1} {\bm{X}} \bm{X}^T + (1-\xi)\bm{I}_{{d}}\Big) (\bm{X}\bm{X}^T + \lambda \bm{I}_d)^{-1} \bm{X}(\bm{X}^T \bm{\theta}^{*} + \bm{\eta}).
\end{equation}

\subsection{Estimation Error Comparison: Bias-Variance Tradeoff}
%\subsection{Comparison of Teacher and Student}
Let us denote the student's error in estimating the ground truth parameter $\bm{\theta}^{*}$ with imitation parameter $\xi$ as ${\bm{\epsilon}}_{S}(\xi) := \hat{\bm{\theta}}_{S}(\xi) - \bm{\theta}^{*}$. Note that ${\bm{\epsilon}}_{S}(0) := \hat{\bm{\theta}}_{S}(0) - \bm{\theta}^{*} = \hat{\bm{\theta}}_{T} - \bm{\theta}^{*}$ is the teacher's estimation error. We shall analyze the expected squared norm of the estimation error w.r.t. the random label noise $\bm{\eta}$, i.e. $\mathbb{E}_{\bm{\eta}}[\|{\bm{\epsilon}}_{S}(\xi)\|^2]$, as a function of $\xi$\footnote{We do not analyze the expected squared prediction error, i.e.  $\mathbb{E}_{\bm{\eta}, \bm{x}}\big[\big(\langle \hat{\bm{\theta}}_{S}(\xi), \bm{x} \rangle - \langle \bm{\theta}^{*}, \bm{x} \rangle\big)^2\big]$, because that would force us to make assumptions on the distribution of $\bm{x}$ (the data) as well. However, it is worth noting that with the standard assumption of $\bm{x} \sim \mathcal{N}(\vec{0}_d, \textup{I}_d)$, the expected squared prediction error is the same as the expected squared norm of the error in estimating $\bm{\theta}^{*}$.}.

It will be illustrative to analyze $\mathbb{E}_{\bm{\eta}}[\|{\bm{\epsilon}}_{S}(\xi)\|^2]$ in terms of the SVD of $\bm{X}$. Let $\text{rank}(\bm{X}) = r$ (note that $r \leq \min(d, n)$) and the SVD decomposition of $\bm{X}$ be $\sum_{j=1}^r \sigma_j \bm{u}_j \bm{v}_j^T$, where $\sigma_1 \geq \ldots \geq \sigma_r > 0$, and each $\bm{u}_j \in \mathbb{R}^d$ and $\bm{v}_j \in \mathbb{R}^n$. Also, let $\{\bm{u}_1,\ldots,\bm{u}_d\}$ be the full set of left singular vectors of $\bm{X}$ (i.e., even those corresponding to the zero singular values); note that this forms an orthonormal basis for $\mathbb{R}^d$.

Following standard bias-variance decomposition, we have:
\begin{equation}
    \label{eq:13-jan2}
    \mathbb{E}_{\bm{\eta}}\Big[\big\|\bm{\epsilon}_{S}(\xi)\big\|^2\Big] = \underbrace{\big\|\mathbb{E}_{\bm{\eta}}[\bm{\epsilon}_{S}(\xi)]\big\|^2}_\text{squared bias} + \underbrace{\mathbb{E}_{\bm{\eta}}\Big[\big\|\bm{\epsilon}_{S}(\xi) - \mathbb{E}_{\bm{\eta}}[\bm{\epsilon}_{S}(\xi)]\big\|^2\Big]}_\text{variance}.
\end{equation}
Now we shall quantify the squared bias and variance in \cref{eq:13-jan2} as a function of $\xi$.
\begin{theorem}[\textbf{Bias$^2$ and Variance}]
\label{thm-est-err}
Suppose \Cref{asmp-noise} holds. Then,
\\
\textbf{\textup{(i)}} the squared bias is:
\begin{equation}
    \label{eq:bias}
    \big\|\mathbb{E}_{\bm{\eta}}[\bm{\epsilon}_{S}(\xi)]\big\|^2 = \sum_{j=1}^r \big(\langle \bm{\theta}^{*}, \bm{u}_j \rangle\big)^2 \Bigg(\frac{\nicefrac{\lambda}{\sigma_j^2}}{1 + \nicefrac{\lambda}{\sigma_j^2}}\Bigg)^2 \Bigg(1 + \frac{\xi}{1 + \nicefrac{\lambda}{\sigma_j^2}}\Bigg)^2 + \sum_{j=r+1}^d \big(\langle \bm{\theta}^{*}, \bm{u}_j \rangle\big)^2.\footnote{Note the $\sum_{j=r+1}^d \big(\langle \bm{\theta}^{*}, \bm{u}_j \rangle\big)^2$ term. If $r < d$, then this quantity is equal to the squared norm of the component of $\bm{\theta}^{*}$ along the non-empty null-space of $\bm{X}^T$; this component is not recoverable by any algorithm.}
\end{equation}
\textbf{\textup{(ii)}} the variance is:
\begin{equation}
    \label{eq:var}
    \mathbb{E}_{\bm{\eta}}\Big[\big\|\bm{\epsilon}_{S}(\xi) - \mathbb{E}_{\bm{\eta}}[\bm{\epsilon}_{S}(\xi)]\big\|^2\Big] = {\frac{\gamma^2}{\lambda} \Bigg\{\sum_{j=1}^r \frac{\nicefrac{\lambda}{\sigma_j^2}}{\big(1 + \nicefrac{\lambda}{\sigma_j^2}\big)^2} \Bigg(1 - \xi \Bigg(\frac{\nicefrac{\lambda}{\sigma_j^2}}{1 + \nicefrac{\lambda}{\sigma_j^2}}\Bigg)\Bigg)^2\Bigg\}},
\end{equation}
where $\gamma^2$ is the per-coordinate label noise variance (as per \Cref{asmp-noise}).
\end{theorem}
\noindent The proof of \Cref{thm-est-err} is in \Cref{pf-thm-est-err}. 

\begin{remark}[\textbf{Bias-Variance Tradeoff as a Function of $\xi$}] 
\label{dist-tradeoff}
Let us restrict our attention to $\xi \in [0,1]$ which is the range of $\xi$ used in practice \cite{lopez2015unifying,li2017learning,sun2019patient}. From \cref{eq:bias}, note that $\big\|\mathbb{E}_{\bm{\eta}}[\bm{\epsilon}_{S}(\xi)]\big\|^2$ is an \textit{increasing} function of $\xi$, i.e. the bias increases as the student tries to imitate the teacher more. However, from \cref{eq:var}, we see that $\mathbb{E}_{\bm{\eta}}\Big[\big\|\bm{\epsilon}_{S}(\xi) - \mathbb{E}_{\bm{\eta}}[\bm{\epsilon}_{S}(\xi)]\big\|^2\Big]$ is a \textit{decreasing} function of $\xi$, i.e., the variance (due to label noise) reduces as the student tries to imitate the teacher more. Thus, SD is associated with a \textbf{bias-variance tradeoff} -- a higher value of the imitation parameter $\xi$ mitigates the impact of label noise variance at the cost of increasing the estimation bias (and vice versa).
\end{remark}

Plugging in \cref{eq:bias} and \cref{eq:var} in \cref{eq:13-jan2}, we obtain $\mathbb{E}_{\bm{\eta}}[\|{\bm{\epsilon}}_{S}(\xi)\|^2]$; note that it is a quadratic function of $\xi$. \Cref{opt-xi} provides the optimal value of $\xi$, say $\xi^{*}$, that minimizes $\mathbb{E}_{\bm{\eta}}[\|{\bm{\epsilon}}_{S}(\xi)\|^2]$ (obtained by simple differentiation).  
\begin{corollary}
\label{opt-xi}
Let $c_j := \nicefrac{\lambda}{\sigma_j^2}$ and $\theta_j^{*} := \big(\langle \bm{\theta}^{*}, \bm{u}_j \rangle\big)^2$. Then:
\begin{equation}
    \xi^{*} = \textup{arg min}_{\xi \in \mathbb{R}} \mathbb{E}_{\bm{\eta}}[\|{\bm{\epsilon}}_{S}(\xi)\|^2] = \frac{\sum_{j=1}^r\big(\frac{\gamma^2}{\lambda} - \theta_j^{*}\big)\frac{c_j^2}{(1+c_j)^3}}{\sum_{j=1}^r\big(\frac{\gamma^2}{\lambda}c_j + \theta_j^{*}\big)\frac{c_j^2}{(1+c_j)^4}}.
\end{equation}
\end{corollary}
\noindent Thus, setting $\xi = \xi^{*}$ yields the optimal balance between the squared bias and variance. %Unfortunately, however, $\xi^{*}$ cannot be computed apriori (as it depends on the $\langle \bm{\theta}^{*}, \bm{u}_j \rangle$'s and we may not know $\gamma$ either). 

\begin{remark}[\textbf{Anti-Learning Observed Labels in Noisy Settings}]
\label{rmk-xi}
There are scenarios when $\xi^{*}$ obtained in \Cref{opt-xi} is more than 1\footnote{$\xi^{*}$ can be negative too, but we shall not focus on this case in this work.}, especially when $\gamma$ is large, i.e., there is a lot of label noise. For e.g., note that $\lim_{\gamma \to \infty}\xi^{*} = \frac{\sum_{j=1}^r {c_j^2}/{(1+c_j)^3}}{\sum_{j=1}^r {c_j^3}/{(1+c_j)^4}} > 1$\footnote{This is because $\sum_{j=1}^r \frac{c_j^3}{(1+c_j)^4} = \sum_{j=1}^r \underbrace{\frac{c_j}{(1+c_j)}}_{<1} \Big(\frac{c_j^2}{(1+c_j)^3}\Big) < \sum_{j=1}^r \frac{c_j^2}{(1+c_j)^3}$.}. However, the imitation parameter $\xi$ is restricted to and tuned in $[0,1]$ \cite{lopez2015unifying,li2017learning,sun2019patient}. Based on our analysis, we advocate not restricting $\xi \in [0,1]$ and also trying $\xi > 1$ in the high noise regime. Setting $\xi > 1$ can be interpreted as \enquote{anti-learning} (or going against) the observed labels.
\end{remark}
\noindent {In \Cref{sec-expts-2}, we provide empirical evidence showing that $\xi > 1$ works better than $\xi \leq 1$ even with the \textit{cross-entropy loss} for several noisy datasets; see \Cref{tab-expts2}.}

\begin{remark}[\textbf{Utility of Teacher's Predicted Labels}] 
\label{dist-utility}
In \Cref{prop_xi_increase} (\Cref{xi_increase}), we show that $\xi^{*}$ is an increasing function of the label noise variance $\gamma^2$, i.e., we should assign more weight to the teacher's predicted labels as $\gamma^2$ increases. So in linear regression, the benefit of using the teacher's predictions (which is the core idea of SD) increases with the degree of label noise.
\end{remark}
\noindent {We make a similar observation in our experiments on multi-class classification in \Cref{dist-utility-ver}}, where SD with $\xi=1$ -- which corresponds to \textit{only using the teacher's predictions} (and completely ignoring the original labels) -- does not yield any gains (over the teacher) with zero label corruption but it consistently yields higher gains as the amount of label corruption increases.
\\
\\
\textbf{Is Optimal Self-Distillation Better than Optimal $\ell_2$ Regularization?} Let $e(\lambda, \xi) := \mathbb{E}_{\bm{\eta}}\big[\|\bm{\epsilon}_{S}(\xi)\|^2\big]$ (recall $\bm{\epsilon}_{S}(\xi)$ is a function of the $\ell_2$-regularization parameter $\lambda$ too). Since $\xi = 0$ corresponds to using plain $\ell_2$ regularization, we define $e_\text{reg}(\lambda) :=  e(\lambda, 0)$ as the estimation error obtained using only $\ell_2$ regularization (and no SD) with parameter $\lambda$. Next, let us define $e_\text{sd}(\lambda)$ as the error obtained using SD with $\ell_2$-regularization parameter = $\lambda$ and the optimal value of $\xi = \xi^{*}$ from \Cref{opt-xi} (which is itself a function of $\lambda$), i.e., $e_\text{sd}(\lambda) := e(\lambda, \xi^{*})$. By definition, $e_\text{sd}(\lambda) \leq e_\text{reg}(\lambda)$ $\forall$ $\lambda$; we wish to know when and if {\color{blue}$\min_{\lambda} e_\text{sd}(\lambda) \bm{<} \min_{\lambda} e_\text{reg}(\lambda)$} (note the \textbf{strict} inequality), i.e., when and if \textbf{optimal} SD is better than \textbf{optimal} $\ell_2$-regularization \textit{by tuning over $\lambda$}.

\begin{theorem}
    \label{thm-opt-lambda}
    Let $\lambda^{*}_\text{reg} := \text{arg min}_{\lambda} e_\textup{reg}(\lambda)$. It holds that $e_\textup{sd}(\lambda^{*}_\text{reg}) = e_\textup{reg}(\lambda^{*}_\text{reg})$ and $\frac{d e_\textup{sd}(\lambda)}{d \lambda}\big|_{\lambda = \lambda^{*}_\text{reg}} = 0$, 
    i.e., $\lambda^{*}_\text{reg}$ is a stationary point of $e_\textup{sd}(\lambda)$ also. It is a \textbf{local maximum} point of $e_\textup{sd}(\lambda)$ when:
    \begin{equation}
        \label{eq-11-jan6}
        \sum_{k=1}^r\sum_{j=1}^{k-1}\frac{\sigma_j^2 \sigma_k^2 (\sigma_j^2 - \sigma_k^2) (\theta_k^{*} - \theta_j^{*})}{(\lambda^{*}_\text{reg} + \sigma_j^2)^4 (\lambda^{*}_\text{reg} + \sigma_k^2)^4} < 0,
    \end{equation}
    with $\theta_j^{*} := (\langle \bm{\theta}^{*}, \bm{u}_j \rangle)^2$. When the above holds, optimal self-distillation is better than optimal $\ell_2$-regularization.
\end{theorem}
{\noindent The detailed version and proof of \Cref{thm-opt-lambda} appear in \Cref{pf-opt-lambda}.

One case when \cref{eq-11-jan6} holds is $\theta_1^{*} > \ldots > \theta_r^{*}$ (since $\sigma_1 \geq \ldots \geq \sigma_r$). In general, when the squared projections of $\bm{\theta}^{*}$ along the most significant left singular vectors of $\bm{X}$ (i.e., the ones with \enquote{large} singular values) follow the same ordering as the corresponding singular values and the noise variance is large enough, $\lambda^{*}_\text{reg}$ will be a local maximum point of $e_\textup{sd}(\lambda)$. We formalize this next.
\begin{theorem}
\label{cond-simpler}
Without loss of generality, let $\|\bm{\theta}^{*}\| = 1$ and $\sigma_1 = 1$. Further, suppose $\sigma_j \leq \delta$ for $j \in \{q+1,\ldots,r\}$ and $\theta_1^{*} > \ldots > \theta_{q}^{*}$. Then, $\lambda = \lambda^{*}_\text{reg}$ is a \textbf{local maximum} point of $e_\textup{sd}(\lambda)$ when $\delta \leq \mathcal{O}(\frac{1}{r})$ and $\gamma^2 \geq \frac{\max_{j \in \{1,\ldots,r\}} \theta_{j}^{*}}{r-1}$.
\end{theorem}
\noindent The detailed statement and proof of \Cref{cond-simpler} appear in \Cref{app=cond-simpler}. 
In practice, $\bm{X}$ is usually low rank and only a few of its singular values are large. So, the assumption of \Cref{cond-simpler} is realistic and that too with $q \ll r$. 

{To the best of our knowledge, \textbf{there are no results} comparable to Theorems \ref{thm-opt-lambda} and \ref{cond-simpler} quantifying when \textbf{optimal} SD is better than \textbf{optimal} $\ell_2$ regularization.}
Now we consider a synthetic example to verify the previous discussion. Suppose $\bm{\theta}^{*} = \frac{1}{\sqrt{2}}\big(\bm{u}_1 + \bm{u}_2\big)$, $n > d = 100$ and $\sigma_j = \frac{1}{j}$ for $j \in \{1,\ldots,d\}$ (so only few singular values are large). Note that \cref{eq-11-jan6} is satisfied. We consider 3 values of $\gamma = \{0.125, 0.25, 0.5\}$ \& 10 values of $\lambda = \{2^{i-3} \gamma^2\}$ with $i \in \{1,\ldots,10\}$. In \Cref{fig:0}, we plot $e_\textup{reg}(\lambda)$ and $e_\textup{sd}(\lambda)$ for these values of $\gamma$ and $\lambda$; see the figure caption for discussion.

\begin{figure*}[!htb]
\centering 
\begin{subfigure}[b]{0.48\textwidth}
 \centering
 \includegraphics[width=\textwidth]{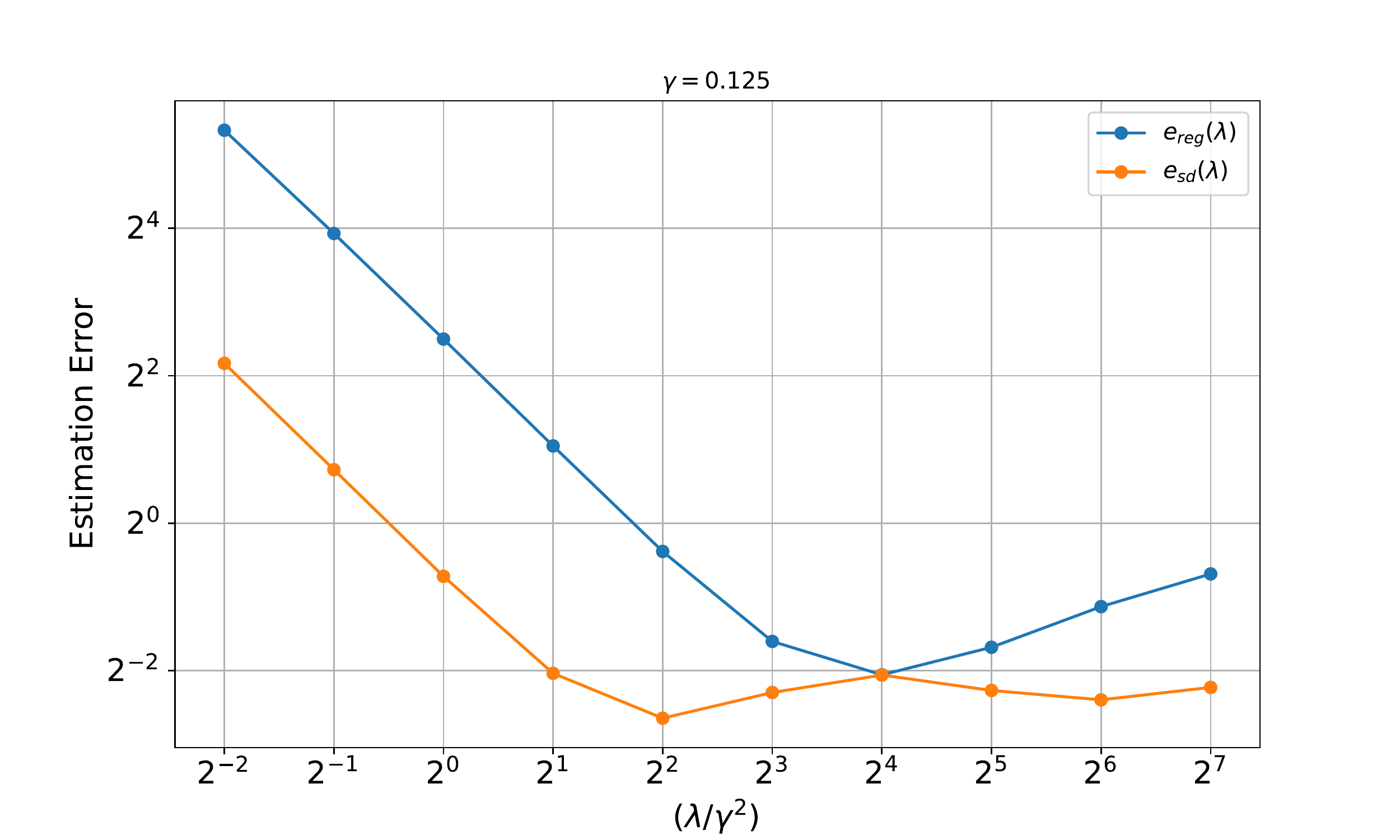}
 \caption{$\gamma=0.125$}
 \label{fig:gamma-0.125}
\end{subfigure}
\begin{subfigure}[b]{0.48\textwidth}
 \centering
 \includegraphics[width=\textwidth]{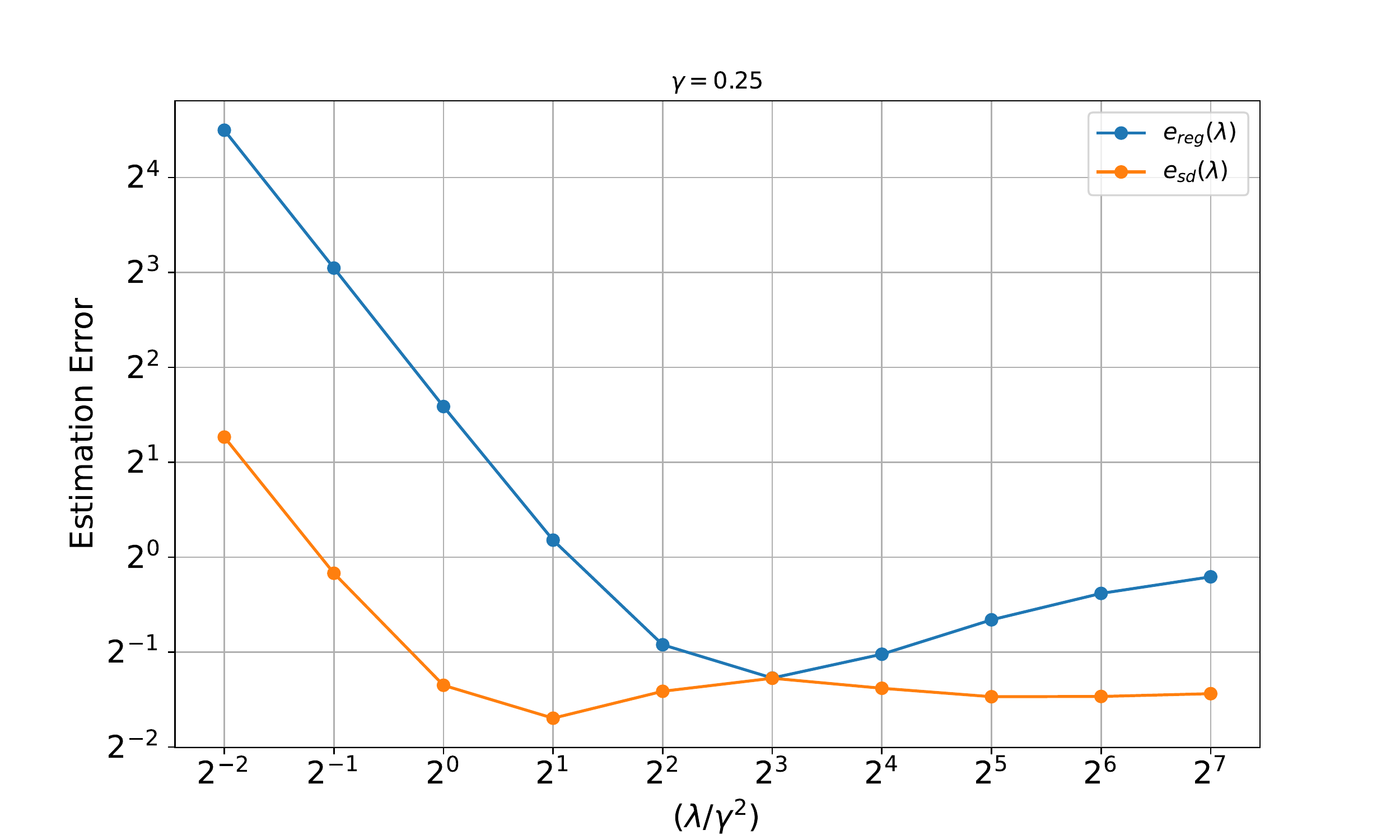}
 \caption{$\gamma=0.25$}
 \label{fig:gamma-0.25}
\end{subfigure}
%\hfill
\begin{subfigure}[b]{0.48\textwidth}
 \centering
 \includegraphics[width=\textwidth]{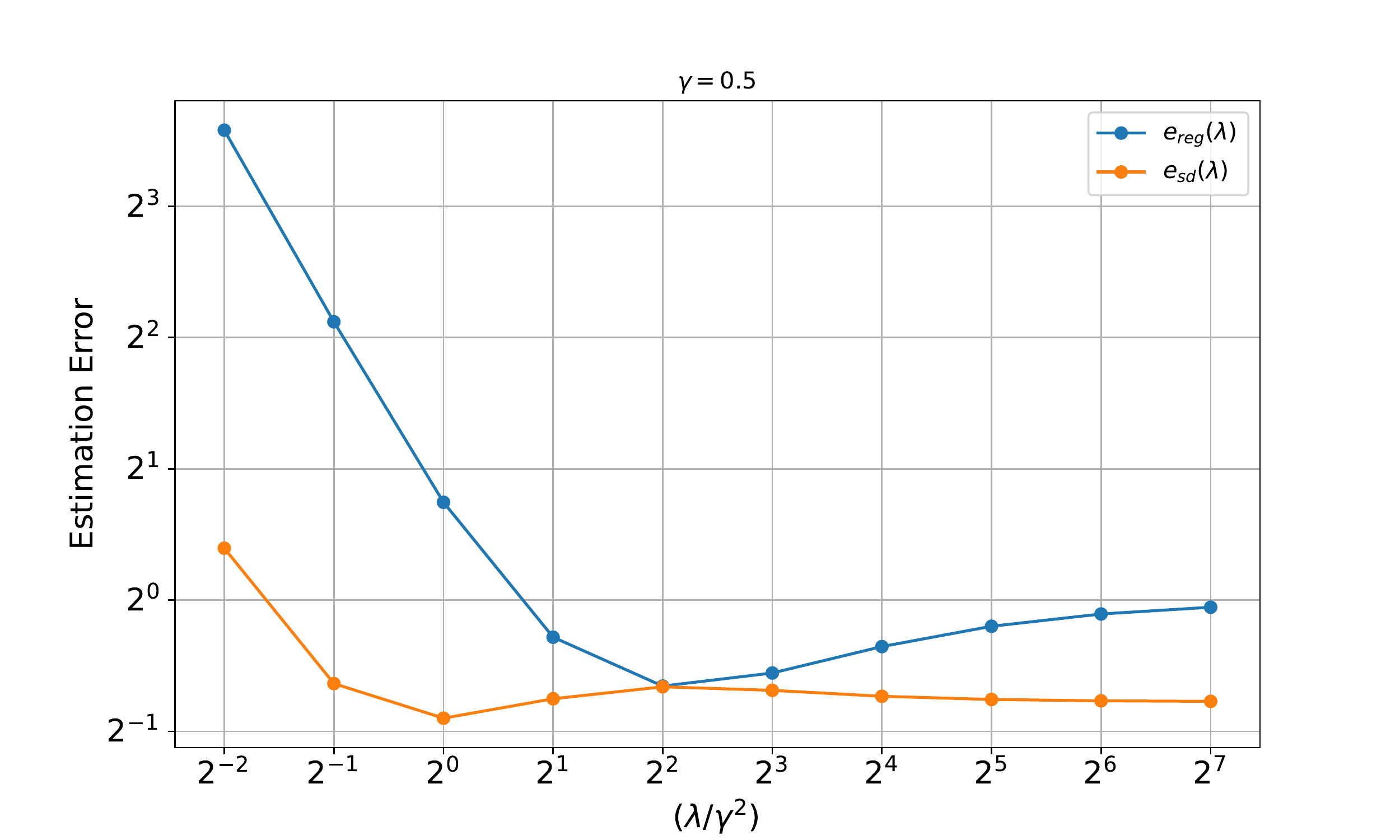}
 \caption{$\gamma=0.5$}
 \label{fig:gamma-0.5}
\end{subfigure}
\caption{Estimation errors of vanilla $\ell_2$ regularization $e_\textup{reg}(\lambda)$ and SD $e_\textup{sd}(\lambda)$ vs. $\lambda$ for the synthetic example at the end of \Cref{lin-reg}. {As per \Cref{thm-opt-lambda}, note that the global minimum of $e_\textup{reg}(\lambda)$ is a local maximum of $e_\textup{sd}(\lambda)$. Observe that $\min_{\lambda} e_\text{sd}(\lambda) < \min_{\lambda} e_\text{reg}(\lambda)$. So, optimal SD does better than optimal $\ell_2$-regularization here.}}
\label{fig:0}
\end{figure*}

If $e_\textup{sd}(\lambda)$ does not have a {local maximum} at $\lambda^{*}_\text{reg}$, it is difficult to say whether $\lambda^{*}_\text{reg}$ is a sub-optimal \textit{local} minimum point or the \textit{global} minimum point of $e_\textup{sd}(\lambda)$; also see \Cref{pf-opt-lambda}. If $\lambda^{*}_\text{reg}$ is the global minimum point of $e_\textup{sd}(\lambda)$, then optimal SD is \textbf{not} better than (i.e., does not yield any improvement over) optimal regularization because $e_\textup{sd}(\lambda^{*}_\text{reg}) = e_\textup{reg}(\lambda^{*}_\text{reg})$.} {To complement this, we present the following result (proved in \Cref{thm-ex-pf}).
\begin{theorem}
    \label{thm-ex}
    There exists $\bm{\theta}^{*}$ and $\bm{X}$ s.t. for any noise variance $\gamma^2$, $\lambda^{*}_\text{reg}$ is the \textbf{global minimum} point of $e_\textup{sd}(\lambda)$.
\end{theorem}
\noindent So there are cases when optimal SD does not yield any improvement over optimal regularization.
}

\section{Logistic Regression}
\label{sec:log_reg}
We now move onto logistic regression with the cross-entropy loss. Note that linear probing \cite{alain2016understanding,kumar2022fine} is the same as logistic regression with features obtained from a pre-trained model. It is also worth mentioning here that our analysis for logistic regression is significantly different from and harder than linear regression.
\\
\\
\textbf{Setting:} We consider a binary classification problem where each sample $\bm{x} \in \mathcal{X}$ has a discrete label ${y}(\bm{x}) \in \{0,1\}$. Let the marginal distribution of the sample space (with support $\mathcal{X}$) be denoted by $\mathcal{P}$.
We assume that there is a feature map $\phi: \mathcal{X} \xrightarrow{} \widetilde{\mathcal{X}}$ and we have access to a sample in terms of its features. We are given $2n$ pairs of data points in terms of features and \textbf{corrupted} labels $\{({\phi(\bm{x}_i)}, \hat{y}_i)\}_{i=1}^{2n}$, where each $\hat{y}_i \in \{0,1\}$ and $\bm{x}_i \underset{\text{iid}}{\sim} \mathcal{P}$. Let the corresponding \textbf{actual} labels be $\{y_i\}_{i=1}^{2n}$; we assume that the dataset is balanced, i.e., $|i: y_i=1| = |i: y_i=0| = n$. Specifically, without loss of generality (w.l.o.g.), let $y_i = 1$ for $i \in \mathcal{S}_1 := \{1,\ldots,n\}$ and $y_i = 0$ for $i \in \mathcal{S}_0 := \{n+1,\ldots,2n\}$; our training algorithms are not privy to this. We consider the following corruption model: $\hat{n} < n/2$ samples of \textit{each} class, chosen \textit{randomly}, are provided to us with flipped labels (again, our training algorithms are not privy to this). Specifically, w.l.o.g., let: 
\begin{equation*}
    \hat{y}_i = 
    \begin{cases}
    1 - {y}_i \text{ for } i \in \underbrace{\{1,\ldots,\hat{n}\}}_{:=\mathcal{S}_{1,\text{bad}}} \cup \underbrace{\{n+1,\ldots,n+\hat{n}\}}_{:=\mathcal{S}_{0,\text{bad}}},
    \\
    {y}_i \text{ for } i \in \underbrace{\{\hat{n}+1,\ldots,n\}}_{:=\mathcal{S}_{1,\text{good}}} \cup \underbrace{\{n+\hat{n}+1,\ldots,2n\}}_{:=\mathcal{S}_{0,\text{good}}}.
    \end{cases}
\end{equation*}
Define $p := \frac{\hat{n}}{n}$ as the \textbf{label corruption fraction}; note that $p < \frac{1}{2}$. 

Our goal is to learn a separator for the data \textit{w.r.t. the actual labels} by training a logistic regression model on $\{({\phi(\bm{x}_i)}, \hat{y}_i)\}_{i=1}^{2n}$. Specifically, for a sample $\bm{x}$ with feature $\phi(\bm{x}) \in \widetilde{\mathcal{X}}$, the prediction for the label $y(\bm{x})$ is modeled as:
\begin{equation}
    \mathbb{P}(y(\bm{x}) = 1) = \sigma(\langle {\bm{\theta}}, \phi(\bm{x}) \rangle),\footnote{The bias term can be absorbed within the feature vector $\phi(.)$ itself.}
\end{equation}
where $\bm{\theta} \in \widetilde{\mathcal{X}}$ is the parameter that we wish to learn, and $\sigma(z) = \frac{1}{1 + e^{-z}}$ for $z \in \mathbb{R}$ is the sigmoid function. We use the binary cross-entropy loss for training; we denote this by $\text{BCE}:[0,1] \times (0,1) \xrightarrow{} \mathbb{R}_{\geq 0}$ and it is defined as:
\begin{equation}
    \text{BCE}(q, \hat{q}) = -\big(q \log(\hat{q}) + (1-q)\log(1-\hat{q})\big).
\end{equation}
Next, we state our assumptions on the feature map $\phi(.)$.

\begin{assumption}[\textbf{Orthonormality}]
\label{a-ortho}
The features have unit norm, i.e., $\|\phi(\bm{x})\|_2 = 1$ $\forall$ $\bm{x} \in \mathcal{X}$. Further, the space of samples in feature space with labels $0$ and $1$ are orthogonal, i.e., $\langle \phi(\bm{x}), \phi(\bm{x}') \rangle = 0$ $\forall$ $\bm{x} \in \mathcal{X}, \bm{x}' \in \mathcal{X}$ with \textbf{different} labels.
\end{assumption}
\noindent \Cref{a-ortho} ensures that the data is separable and indeed there exists a separator.

\begin{assumption}[\textbf{Feature Correlation in the Training Set}]
\label{a3}
$\langle \phi(\bm{x}_i), \phi(\bm{x}_{i'}) \rangle = c \in (0, 1)$ $\forall$ $i \neq i'$ such that  $y_i = y_{i'}$.
\end{assumption}
\noindent It is true that at face value, \Cref{a3} seems strong. Instead, an assumption in expectation like $\mathbb{E}_{\bm{x}, \bm{x}'}\Big[\langle \phi(\bm{x}), \phi(\bm{x}') \rangle \Big| \bm{x} \text{ and } \bm{x}' \text{ have the same label}\Big] = c$ is more realistic; let us call this \Cref{a3}$'$ for the sake of discussion. For $n \to \infty$ and when the labels are corrupted randomly, we hypothesize that the average\footnote{This is taken over the training set.} prediction (i.e., soft score $\in (0,1)$ assigned to a particular class) of a model under \Cref{a3}$'$ is the same as that under \Cref{a3}. We provide empirical evidence to support this hypothesis in \Cref{a3-motiv}. Thus, for large $n$, we argue that \Cref{a3} is reasonable and an important case to analyze.
\\
\\
\noindent \textbf{Teacher Model:} To learn the logistic regression parameter, the teacher minimizes the $\ell_2$-regularized binary cross-entropy loss with the provided labels as its targets, i.e., the teacher's objective is:
\begin{equation}
    \label{eq:57}
    %f_{\text{T}}(\bm{\theta}) = \frac{1}{2n} \sum_{i=1}^{2n}-\Big(\hat{y}_i \log\big(\sigma(\langle \bm{\theta}, \phi(\bm{x}_i) \rangle)\big) + (1-\hat{y}_i)\log\big(1 - \sigma(\langle \bm{\theta}, \phi(\bm{x}_i) \rangle)\big)\Big) + \frac{\lambda \|\bm{\theta}\|^2}{2}.
    f_{\text{T}}(\bm{\theta}) = \frac{1}{2n} \sum_{i=1}^{2n} \text{BCE}\Big(\hat{y}_i, \sigma\big(\langle \bm{\theta}, \phi(\bm{x}_i) \rangle\big)\Big) + \frac{\lambda \|\bm{\theta}\|^2}{2}.
\end{equation}
In \cref{eq:57}, $\lambda > 0$ is the $\ell_2$-regularization parameter. The teacher's estimated parameter is $\bm{\theta}_{\text{T}}^{\ast} := \text{arg min}_{\bm{\theta}} f_{\text{T}}(\bm{\theta})$. The teacher's predicted \textit{soft} label for the $i^{\text{th}}$ sample is $y_i^{\textup{(T)}} := \sigma(\langle \bm{\theta}_{\text{T}}^{\ast}, \phi(\bm{x}_i) \rangle)$; these are used to train the student. 
\\
\\
\noindent \textbf{Student Model Trained Only with Teacher's Soft Labels:} Here we set the imitation parameter $\xi=1$ in \cref{eq:1-intro}. Thus, the student minimizes the $\ell_2$-regularized binary cross-entropy loss with the teacher's predicted \textit{soft} labels as its targets, i.e., the student's objective is:
\begin{equation}
    \label{eq:83}
    %f_{\text{S}}(\bm{\theta}) = \frac{1}{2n} \sum_{i=1}^{2n}-\Big(y_i^{\text{(T)}} \log\big(\sigma(\langle \bm{\theta}, \phi(\bm{x}_i) \rangle)\big) + (1 - y_i^{\text{(T)}})\log\big(1 - \sigma(\langle \bm{\theta}, \phi(\bm{x}_i) \rangle)\big)\Big) + \frac{\lambda \|\bm{\theta}\|^2}{2}.
    f_{\text{S}}(\bm{\theta}) = \frac{1}{2n} \sum_{i=1}^{2n} \text{BCE}\Big(y_i^{\text{(T)}}, \sigma\big(\langle \bm{\theta}, \phi(\bm{x}_i) \rangle\big)\Big) + \frac{\lambda \|\bm{\theta}\|^2}{2}.
\end{equation}
In \cref{eq:83}, $\lambda$ is the same $\ell_2$-regularization parameter that is used by the teacher. The student's estimated parameter is $\bm{\theta}_{\text{S}}^{\ast} := \text{arg min}_{\bm{\theta}} f_{\text{S}}(\bm{\theta})$. 

\subsection{Comparison of Student and Teacher}
\label{comp}
We shall now characterize the conditions under which the student outperforms the teacher w.r.t. classification accuracy; to our knowledge, this is the \textbf{first result} of its kind. For the sake of avoiding any ambiguity, the teacher's population accuracy is defined as $100*\mathbb{E}_{\bm{x} \sim \mathcal{P}}\Big[\mathds{1}\Big(y(\bm{x}) = \mathds{1}\Big(\sigma(\langle \bm{\theta}_{\text{T}}^{\ast}, \phi(\bm{x}) \rangle) > \frac{1}{2}\Big)\Big)\Big]$\%\footnote{$\mathds{1}(.)$ is the indicator function. Specifically, $\mathds{1}(z) = 1$ if $z$ is true and 0 if $z$ is false.}. The student's accuracy is defined similarly with $\bm{\theta}_{\text{S}}^{\ast}$ replacing $\bm{\theta}_{\text{T}}^{\ast}$.
\begin{theorem}[\textbf{When is Student's Accuracy > Teacher's Accuracy?}]
\label{thm1}
Suppose we have access to the population, i.e., $n \to \infty$. Further, let Assumptions \ref{a-ortho} and \ref{a3} hold with $c = \Theta(1)$ in \Cref{a3} (recall that $c < 1$). Define $\hat{\lambda} := 2 n  \lambda$ and $r := \frac{(1-c)}{4 \hat{\lambda}}$. Suppose $\lambda$ is chosen so that $\hat{\lambda} \in \Big[\frac{1-c}{2.16}, \frac{1-c}{0.40}\Big]$, which corresponds to $r \in [0.10,0.54]$. If the label corruption fraction 
\begin{equation*}
    p \in \Bigg(\max \Big(\frac{1.08-r}{2.08}, \frac{1+r}{3.7}\Big), 1 -  \frac{0.51 (1+r)^2}{1 + 2r} \Bigg),
\end{equation*}
then the student achieves 100\% population accuracy (w.r.t. the true labels), while the teacher only achieves a population accuracy of 100(1-p)\% (again, w.r.t. the true labels). 
\end{theorem}
\noindent{\textbf{Discussion:} In our setup, there exists $0< p_\text{low} < p_\text{high} < 0.5$ such that (i) when $p \leq p_\text{low}$, the teacher attains $100$\% accuracy and so there is no need for SD, (ii) when $p \in (p_\text{low}, p_\text{high})$, the student attains $100$\% accuracy while the teacher attains $100(1-p)$\% accuracy, and (iii) when $p \geq p_\text{high}$, both the teacher and student attain $100(1-p)$\% accuracy.} The range of $p$ in \Cref{thm1} $\subseteq (p_\text{low}, p_\text{high})$; our range is more conservative than the actual range because we had to impose some more restrictions on $p$ in order to control certain error terms in our analysis. 
In \Cref{fig:1}, we plot the teacher's and student's accuracies as a function of $p$ for $r=\{0.2,0.3,0.4\}$ obtained by exactly solving for $\bm{\theta}_{\text{T}}^{\ast}$ and $\bm{\theta}_{\text{S}}^{\ast}$ (through a computer). In all the cases, it can be seen that the range of $p$ where the student outperforms the teacher as per \Cref{thm1} falls within the actual range of $p$ where the student outperforms the teacher.
\\
\begin{figure*}[t]
\centering 
\subfloat[$r=0.2$. Derived bound in \Cref{thm1}: \\ $p \in (0.423,  0.475)$.]{
	\includegraphics[width=0.46\textwidth]{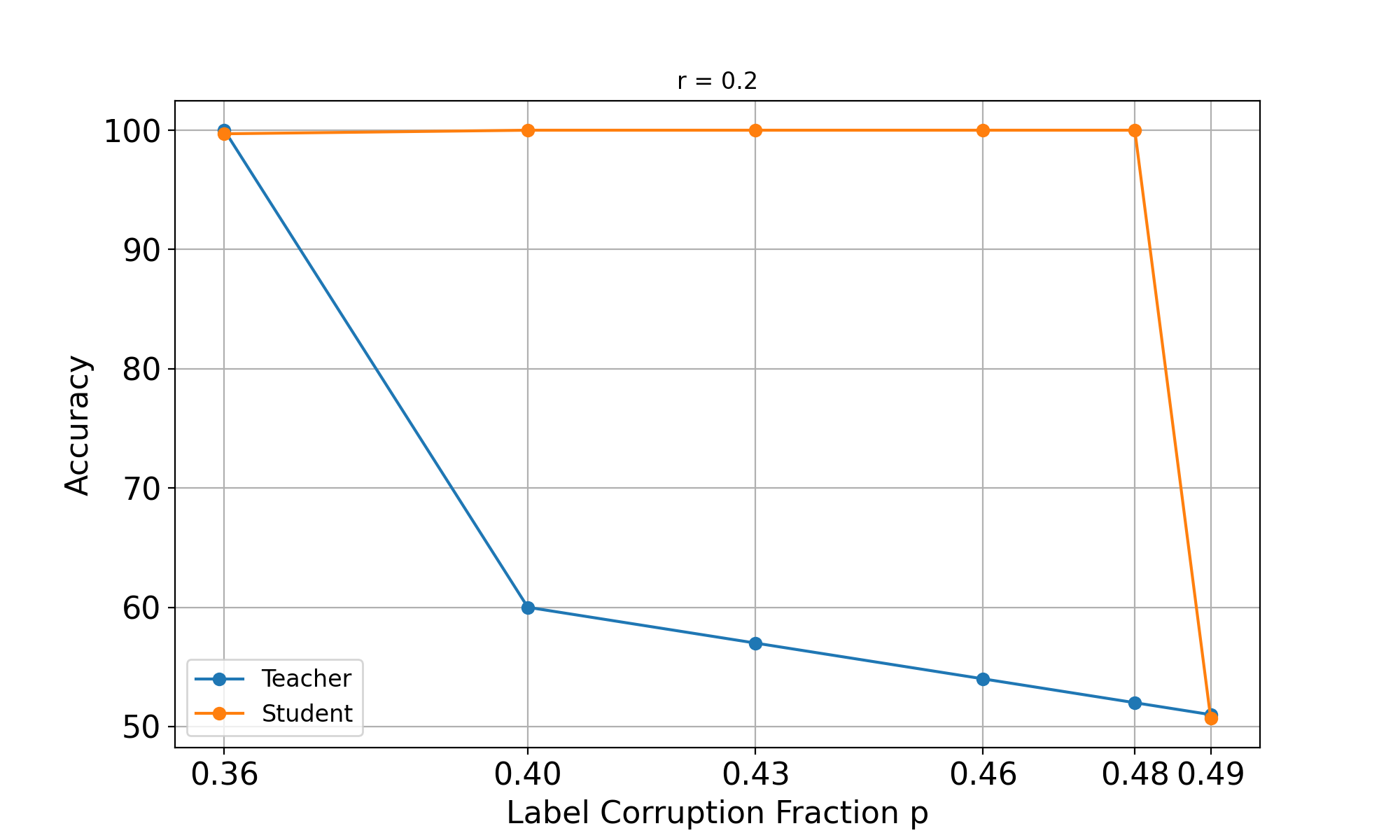}
}
\subfloat[$r=0.3$. Derived bound in \Cref{thm1}: \\ $p \in (0.375,  0.461)$.]{
	\includegraphics[width=0.46\textwidth]{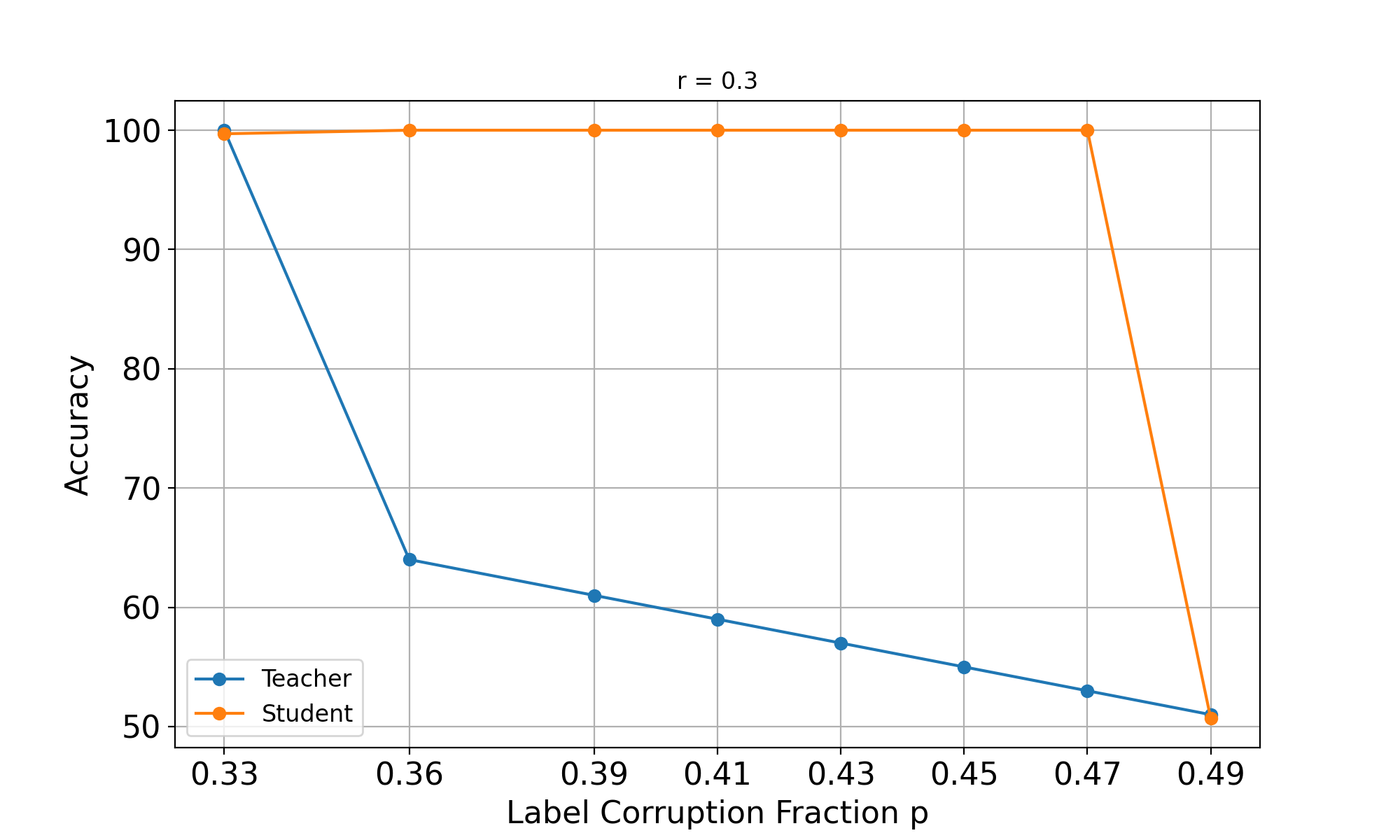}
	}
\\
\subfloat[$r=0.4$. Derived bound in \Cref{thm1}: \\ $p \in (0.378,  0.444)$.]{
	\includegraphics[width=0.46\textwidth]{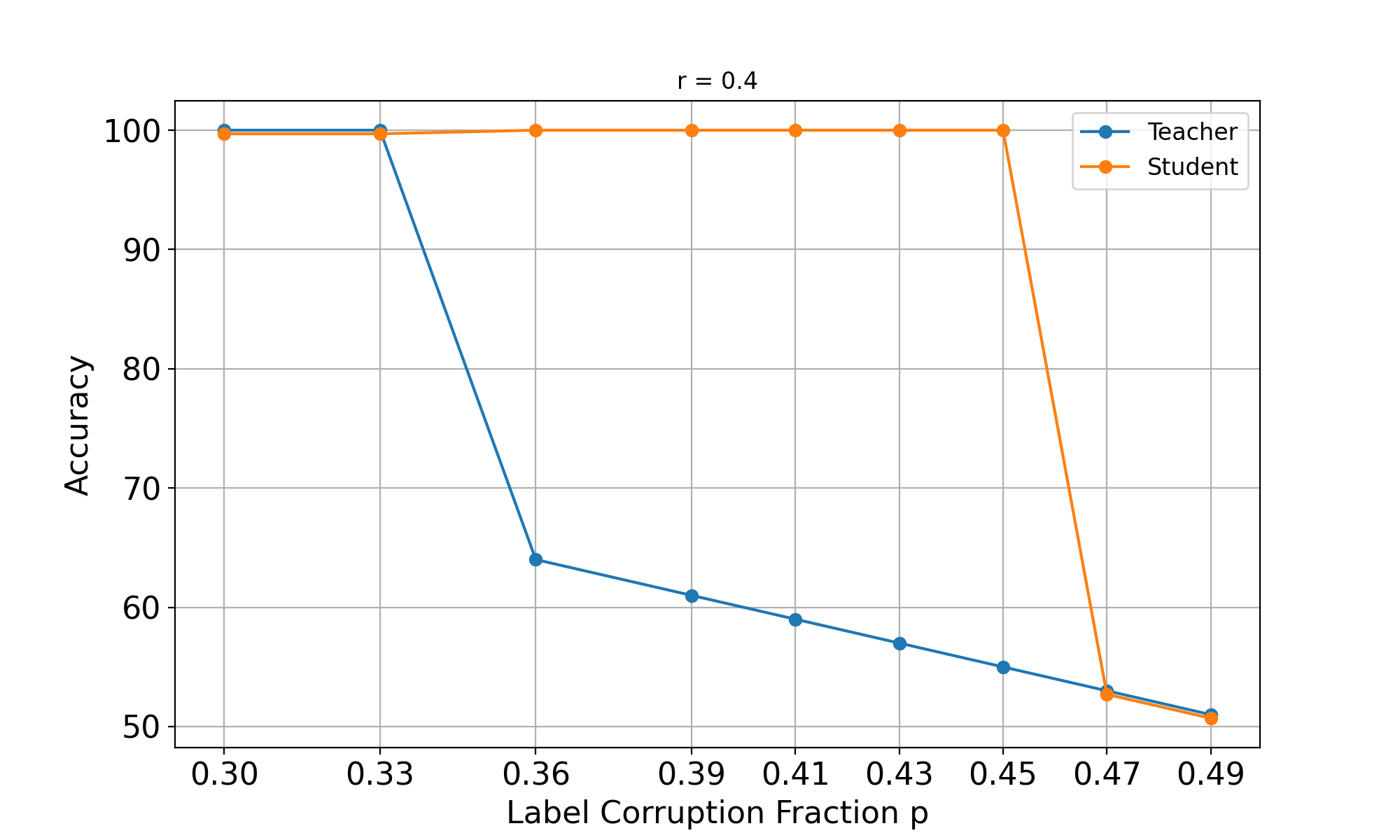}
}
\caption{Comparison of student's and teacher's accuracies for different values of label corruption fraction $p$ obtained by exactly solving \cref{eq:71-jan6} and \cref{eq:72-jan6} for the teacher and \cref{eq:87-jan6} and \cref{eq:88-jan6} for the student. We set $c = 0.1$ and $n = 5000$ here. In all the cases, note that our predicted range of $p$ where the student outperforms the teacher as per \Cref{thm1} falls within the actual range of $p$ where the student outperforms the teacher.}
\label{fig:1}
\end{figure*}

\noindent The detailed proof of \Cref{thm1} can be found in \Cref{pf-thm1}; we now outline the \textbf{key steps in the proof}.
\\
\textbf{Step 1 (Details in \Cref{step-1}).}
It can be shown that the teacher's learned parameter $\bm{\theta}_{\text{T}}^{\ast} = \text{arg min}_{\bm{\theta}} f_{\text{T}}(\bm{\theta}) = \sum_{i=1}^{2n} \alpha_i \phi(\bm{x}_i)$ for some real numbers $\{\alpha_i\}_{i=1}^{2n}$ which are known as the teacher's dual-space coordinates. In \Cref{thm-teacher}, we obtain expressions for $\{\alpha_i\}_{i=1}^{2n}$ which then enables us to obtain the teacher's predicted soft labels $\Big\{y_i^{\textup{(T)}}\Big\}_{i=1}^{2n}$. Specifically, we get:
\begin{equation}
    \label{eq:80-jan6}
    y_i^{\textup{(T)}} = 
    \begin{cases}
    \hat{\lambda} \hat{\alpha} \text{ for } i \in \mathcal{S}_{1,\text{bad}}, %\{1,\ldots,\hat{n}\},
    \\
    1 - \hat{\lambda} {\alpha} \text{ for } i \in \mathcal{S}_{1,\text{good}}, %\{\hat{n}+1,\ldots,n\},
    \\
    1 - \hat{\lambda} \hat{\alpha} \text{ for } i \in \mathcal{S}_{0,\text{bad}}, %\{n+1,\ldots,n+\hat{n}\},
    \\
    \hat{\lambda} {\alpha} \text{ for } i \in \mathcal{S}_{0,\text{good}}, %\{n+\hat{n}+1,\ldots,2n\}.
    \end{cases}
\end{equation}
where $\alpha \geq 0$ and $\hat{\alpha} \geq 0$ are obtained by jointly solving:
\begin{equation}
    \label{eq:71-jan6}
     \sigma\Big(c n \big(\alpha - (\alpha + \hat{\alpha})p) - (1-c)\hat{\alpha}\Big) = \hat{\lambda} \hat{\alpha},
\end{equation}
and
\begin{equation}
    \label{eq:72-jan6}
     \sigma\Big(c n \big(\alpha - (\alpha + \hat{\alpha})p) + (1-c){\alpha}\Big) = 1 - \hat{\lambda} {\alpha}.
\end{equation}
We focus on the interesting case of:
\\
\textbf{(a)} $p$ being large enough so that the teacher misclassifies the incorrectly labeled points ($\mathcal{S}_{1,\text{bad}} \cup \mathcal{S}_{0,\text{bad}}$) because otherwise, there is no need for SD, and
\\
\textbf{(b)} $\hat{\lambda}$ being chosen sensibly so that the teacher at least correctly classifies the correctly labeled points ($\mathcal{S}_{1,\text{good}} \cup \mathcal{S}_{0,\text{good}}$) because otherwise, SD is hopeless.
\\
\\
Later in Step 3, we impose conditions on $p$ (a lower bound) and $\hat{\lambda}$ such that \textbf{(a)} and \textbf{(b)} hold by requiring $\hat{\lambda} \hat{\alpha} < \frac{1}{2}$ and $\hat{\lambda} {\alpha} < \frac{1}{2}$.
\\
\\
\textbf{Step 2 (Details in \Cref{step-2}).} Similar to the teacher in Step 1, in \Cref{thm-student}, we show that the student's predicted soft label for the $i^{\text{th}}$ sample, $y_i^{\textup{(S)}}$, turns out to be:
\begin{equation}
    \label{eq:97-jan6}
    y_i^{\textup{(S)}} = 
    \begin{cases}
    \hat{\lambda} \hat{\alpha} + \hat{\lambda} \hat{\beta} \text{ for } i \in \mathcal{S}_{1,\text{bad}}, %\{1,\ldots,\hat{n}\},
    \\
    1 - \hat{\lambda} {\alpha} - \hat{\lambda} {\beta} \text{ for } i \in \mathcal{S}_{1,\text{good}}, %\{\hat{n}+1,\ldots,n\},
    \\
    1 - \hat{\lambda} \hat{\alpha} - \hat{\lambda} \hat{\beta} \text{ for } i \in \mathcal{S}_{0,\text{bad}}, %\{n+1,\ldots,n+\hat{n}\},
    \\
    \hat{\lambda} {\alpha} + \hat{\lambda} {\beta} \text{ for } i \in \mathcal{S}_{0,\text{good}}, %\{n+\hat{n}+1,\ldots,2n\},
    \end{cases}
\end{equation}
where $\beta \geq 0$ and $\hat{\beta} \geq 0$ (assuming $\hat{\lambda} \hat{\alpha} < \frac{1}{2}$ and $\hat{\lambda} {\alpha} < \frac{1}{2}$) are obtained by jointly solving:
\begin{equation}
    \label{eq:87-jan6}
     \sigma\Big(c n \big(\beta - (\beta + \hat{\beta})p) - (1-c)\hat{\beta}\Big) = \hat{\lambda} \hat{\alpha} + \hat{\lambda} \hat{\beta},
\end{equation}
and
\begin{equation}
    \label{eq:88-jan6}
     \sigma\Big(c n \big(\beta - (\beta + \hat{\beta})p) + (1-c){\beta}\Big) = 1 - \hat{\lambda} {\alpha} - \hat{\lambda} {\beta}.
\end{equation}
Now note that if $\hat{\lambda} \hat{\alpha} + \hat{\lambda} \hat{\beta} > \frac{1}{2}$ and $\hat{\lambda} {\alpha} + \hat{\lambda} {\beta} < \frac{1}{2}$, then the student has managed to correctly classify all the points in the training set; we ensure this in Step 3 by upper bounding $p$.
The tradeoff here is that the (1-0) accuracy of the student increases at the cost of decreased confidence in classifying the correctly labeled points compared to the teacher.
\\
\\
\textbf{Step 3 (Details in \Cref{step-3}).} Now we come to the \textit{challenging} part of the proof. To obtain a range for $p$, we need to analytically solve \cref{eq:71-jan6} and \cref{eq:72-jan6} for the teacher and then \cref{eq:87-jan6} and \cref{eq:88-jan6} for the student, which is particularly \textit{challenging} due to the non-linearity of the sigmoid function present in these equations. Our novel proof technique involves employing the first-order Maclaurin series expansion of the sigmoid function which enables us to bound $\alpha, \hat{\alpha}, \beta$ and $\hat{\beta}$ as a function of $p$, $\hat{\lambda}$ and $c$ in a small range (while imposing some conditions on $p$ and $\hat{\lambda}$ to ensure the range is small). Using this, we can bound the teacher's and student's predictions, and then impose conditions on $p$ and $\hat{\lambda}$ such that the teacher only correctly classifies the correctly labeled points and errs on all the incorrectly labeled points (i.e., $\hat{\lambda} {\alpha} < \frac{1}{2}$ and $\hat{\lambda} \hat{\alpha} < \frac{1}{2}$; see Step 1) but the student correctly classifies all the points (i.e., $\hat{\lambda} {\alpha} + \hat{\lambda} {\beta} < \frac{1}{2}$ and $\hat{\lambda} \hat{\alpha} + \hat{\lambda} \hat{\beta} > \frac{1}{2}$; see Step 2). {Finally, since $n \to \infty$, population accuracy $\to$ training accuracy (we formalize this at the end in \Cref{step-3}).}

\subsection{Variability of Predictions of Student and Teacher}
\begin{corollary}[\textbf{Variability of predictions of points within the same class}]
\label{rmk-3}
Define $\Delta_{\textup{T}} := \max_{i \neq i', y_i = y_{i'}} |y_i^{\textup{(T)}} - y_{i'}^{\textup{(T)}}|$ as the teacher's variability, i.e., the maximum difference between the teacher's predictions on two points having the same ground truth label. Similarly, $\Delta_{\textup{S}} := \max_{i \neq i', y_i = y_{i'}} |y_i^{\textup{(S)}} - y_{i'}^{\textup{(S)}}|$ is defined as the student's variability. Under the conditions of \Cref{thm1}, %the student's variability is strictly less than the teacher's variability.
$\Delta_{\textup{S}} < \Delta_{\textup{T}}$.
\end{corollary}
\noindent In other words, the student's predictions are more homogeneous than the teacher's predictions as per \Cref{rmk-3}. This is analogous to \textit{SD mitigating the variance term due to label noise} in linear regression (\Cref{dist-tradeoff}) leading to smaller variability.
\\
\\
We prove \Cref{rmk-3} in \Cref{pf-rmk-3} and corroborate it with empirical evidence in \Cref{sec-expts-1}.

\section{Empirical Results}
\label{sec-expts}
For our experiments, we consider multi-class classification with the cross-entropy loss on several vision datasets available in PyTorch's torchvision, namely, CIFAR-100 with 100 classes, Caltech-256 \cite{griffin2007caltech} with 257 classes, Food-101 \cite{bossard14} with 101 classes, StanfordCars \cite{KrauseStarkDengFei-Fei_3DRR2013} with 196 classes and Flowers-102 \cite{nilsback2008automated} with 102 classes. Since Caltech-256 does not have any train/test split provided by default, we pick 25k random images from the full dataset to form the training set, while the remaining images form the test set. For all the datasets, we train a softmax layer on top of a pre-trained ResNet-34/VGG-16 model on ImageNet which is kept fixed, i.e., we do \textit{linear probing} on ResNet-34/VGG-16. No data augmentation is involved. Next, we describe the different types of label corruption that we experiment on. 
\\
\\
\textbf{Label Corruption Type 1 (Random Corruption):} Suppose the set of labels is $[C] := \{1,\ldots,C\}$. Consider a sample whose true label is $c \in [C]$. A corruption level of $100p$ \% means we observe this sample’s label as $c$ with a probability of $(1-p)$ or some random $i \in [C] \setminus c$ with a probability of $p/(C-1)$ for each such $i \neq c$. We call this \textbf{random} corruption\footnote{This has been also called symmetric noise in prior work; see for e.g., \cite{chen2019understanding}}.
\\
\\
\noindent \textbf{Label Corruption Type 2 (Hierarchical Corruption \cite{hendrycks2018using}):} Here, the label corruption only occurs between semantically similar classes. This is a more realistic type of corruption compared to random corruption. By default, CIFAR-100 comes with 20 super-classes each containing 5 semantically similar classes; for e.g., the super-class \enquote{fish} consists of aquarium fish, flatfish, ray, shark and trout, while the super-class \enquote{small mammals} consists of hamster, mouse, rabbit, shrew and squirrel. Unfortunately, the other datasets do not have any semantically similar classes provided by default.

Now, we describe the exact corruption scheme. Consider a sample whose true class is $c$ and super-class is $S = \{c_1,\ldots,c_{|S|}\}$. A corruption level of $100p$ \% means we observe this sample’s label as $c$ with a probability of $(1-p)$ or some random $c' \in S \setminus c$ with a probability of $p/(|S|-1)$ (for each such $c' \neq c$). Following \cite{hendrycks2018using}, we call this \textbf{hierarchical} corruption.
\\
\\
\noindent \textbf{Label Corruption Type 3 (Adversarial Corruption):} Instead of semantically similar classes, we determine \enquote{hard} classes for each class by looking at the output of the teacher in the noiseless case (i.e., when there is no corruption) and induce label corruption only among these hard classes. Specifically, in the noiseless case, for a sample $\bm{x}$, %with true label given by $c(\bm{x}) \in \{1,\ldots,C\}$, 
let $p_{\text{T}}(\bm{x}, c)$ be the teacher's predicted probability of $\bm{x}$ belonging to class $c  \in \{1,\ldots,C\}$. Also, let $\mathcal{X}_c$ be the set of samples in the training set belonging to class $c$. Now, for each class $c$, we compute $\bm{\nu}_c = \Big[\frac{1}{|\mathcal{X}_c|}\sum_{\bm{x} \in \mathcal{X}_c} p_{\text{T}}(\bm{x}, 1), \ldots, \frac{1}{|\mathcal{X}_c|}\sum_{\bm{x} \in \mathcal{X}_c} p_{\text{T}}(\bm{x}, C)\Big] \in \mathbb{R}^C$, and define the $k$ hardest classes for class $c$ to be the indices in $\{1,\ldots,C\} \setminus c$ corresponding to the $k$ largest values in $\bm{\nu}_c$. For our experiments, we take $k = 5$.

Now, we describe the corruption scheme. Consider a sample whose true class is $c$ and the set of hardest 5 classes for $c$ is $S$. A corruption level of $100p$ \% means we observe this sample’s label as $c$ with a probability of $(1-p)$ or some random $c' \in S$ with a probability of $p/5$. We call this \textbf{adversarial} corruption. 

\subsection{Verifying Remark~\ref{rmk-xi}}
\label{sec-expts-2}
In \Cref{rmk-xi}, we advocated trying $\xi > 1$ in the high noise regime. We shall now test our recommendation on several noisy datasets. The teacher is trained with the $\ell_2$-regularized cross-entropy loss and the student's per-sample loss is given by \cref{eq:1-intro} where $\ell$ is the $\ell_2$-regularized cross-entropy loss. {Following our theory setting, the teacher and student are both trained with the {same} $\ell_2$-regularization parameter; the common weight decay value (PyTorch's $\ell_2$-regularization parameter) is set to $5 \times 10^{-4}$. Note that this weight decay value was the first one that we tried (i.e., it was not cherry-picked); in fact, we show results with other weight decay values in \Cref{new-expts=jan25}.} We defer the remaining experimental details to \Cref{expt-detail}. In \Cref{tab-expts2}, we list the student's improvement over the teacher (i.e., student's test accuracy - teacher's test accuracy)\footnote{The individual accuracies of the teacher and student can be found in \Cref{expt-detail}; we omit them in the main text for brevity.} averaged across 3 different runs for different values of $\xi$ with ResNet-34 and VGG-16 in the case of 50\% random, hierarchical and adversarial corruption. In all these experiments, note that the value of $\xi$ yielding the biggest improvement is $> 1$. \Cref{tab-expts2-supp-jan11} (in \Cref{sec-expts-2-supp}) shows results with 30\% corruption in Stanford Cars and Flowers-102; even there, $\xi > 1$ does better than $\xi \leq 1$. 

\subsection{Verifying Remark~\ref{dist-utility}}
\label{dist-utility-ver}
{In \Cref{dist-utility}, we claimed that the utility of the teacher's predictions increases with the amount of label noise. To demonstrate this, we train the student with $\xi=1$ which corresponds to setting the teacher's predicted \textit{soft} labels as the student's targets (just as we did in \Cref{sec:log_reg}) and completely ignoring the provided labels. All other experimental details (including weight decay) are the same as in \Cref{sec-expts-2}. In \Cref{tab-expts1}, we show the student's improvement over the teacher averaged across 3 different runs for varying degrees and types of label corruption with ResNet-34; see the table caption for discussion.}

\subsection{Verifying Corollary~\ref{rmk-3}}
\label{sec-expts-1}
We now provide empirical evidence for our claim of the student's predictions being more homogeneous than the teacher's predictions in \Cref{rmk-3}. Since our experiments are for the multi-class (and not binary) case, we look at a slightly different metric to quantify variability which we introduce next. For a sample $\bm{x}$ belonging to class $c(\bm{x})$, let $\hat{p}_{\text{T}}(\bm{x})$ and $\hat{p}_{\text{S}}(\bm{x})$ be the teacher's and student's predicted probability of $\bm{x}$ belonging to $c(\bm{x})$, respectively. Also, let $\mathcal{X}_c'$ be the set of samples in the test set belonging to class $c$. To quantify the variability of the teacher and student for class $c$, we look at $\max_{\bm{x}_1, \bm{x}_2 \in \mathcal{X}_c'} |\hat{p}_{\text{T}}(\bm{x}_1) - \hat{p}_{\text{T}}(\bm{x}_2)|$ and $\max_{\bm{x}_1, \bm{x}_2 \in \mathcal{X}_c'} |\hat{p}_{\text{S}}(\bm{x}_1) - \hat{p}_{\text{S}}(\bm{x}_2)|$, i.e., the range of $\hat{p}_{\text{T}}(\bm{x})$ and $\hat{p}_{\text{S}}(\bm{x})$ w.r.t. $\bm{x} \in \mathcal{X}_c'$, respectively. {In \Cref{fig:2}, we plot the per-class variability as defined here for three of the cases of \Cref{tab-expts1} covering all three types of label corruption; please see the caption for discussion.}

\begin{table}[!htb]
{
\begin{subtable}{.5\linewidth}\centering
{\begin{tabular}{|c|c|}
\hline
$\xi$ & \makecell{Improvement of \\ student over teacher}
\\
\hline
0.2 &  $2.22 \pm 0.12$ \% 
\\
%\hline
0.5 & $5.18 \pm 0.03$ \%
\\
%\hline
0.7 & $6.84 \pm 0.06$ \%
\\
%\hline
1.0 & $8.54 \pm 0.29$ \%
\\
%\hline
{1.2} & ${9.66 \pm 0.23}$ \%
\\
%\hline
\rowcolor{Gray}
\textbf{1.5} & $\bm{10.04 \pm 0.51}$ \% 
\\
%\hline
\rowcolor{Gray}
\textbf{1.7} & $\bm{9.81 \pm 0.55}$ \% 
\\
%\hline
2.0 & $8.56 \pm 0.73$ \% 
\\
\hline
\end{tabular}}
\caption{50\% Random Corruption in Caltech-256 \\ with {ResNet-34}}
\label{tab:2a}
\vspace{0.2 cm}
\end{subtable}%
}
{
\begin{subtable}{.5\linewidth}\centering
{\begin{tabular}{|c|c|}
\hline
$\xi$ & \makecell{Improvement of \\ student over teacher}
\\
\hline
0.5 & $0.89 \pm 0.10$ \%
\\
%\hline
1.0 & $2.01 \pm 0.14$ \%
\\
%\hline
1.5 & $3.13 \pm 0.11$ \%
\\
%\hline
2.0 & $4.22 \pm 0.20$ \%
\\
%\hline
2.5 & $5.28 \pm 0.13$ \%
\\
%\hline
\rowcolor{Gray}
\textbf{3.0} & $\bm{5.78 \pm 0.12}$ \%
\\
%\hline
\rowcolor{Gray}
\textbf{3.5} & $\bm{5.86 \pm 0.18}$ \%
\\
%\hline
4.0 & $5.32 \pm 0.33$ \%
\\
\hline
\end{tabular}}
\caption{50\% Random Corruption in Caltech-256 \\ with {VGG-16}}
\label{tab:3a}
\vspace{0.2 cm}
\end{subtable}%
}
\\
\\
%\vspace{0.5 cm}
{
\begin{subtable}{.5\linewidth}\centering
{\begin{tabular}{|c|c|}
\hline
$\xi$ & \makecell{Improvement of \\ student over teacher}
\\
\hline
0.2 & $0.98 \pm 0.12$ \%
\\
%\hline
0.5 & $2.46 \pm 0.11$ \%
\\
%\hline
0.7 & $3.38 \pm 0.02$ \%
\\
%\hline
{1.0} & ${4.19 \pm 0.09}$ \%
\\
%\hline
\rowcolor{Gray}
\textbf{1.2} & $\bm{4.46 \pm 0.19}$ \%
\\
%\hline
\rowcolor{Gray}
\textbf{1.5} & $\bm{4.46 \pm 0.17}$ \%
\\
%\hline
\rowcolor{Gray}
\textbf{1.7} & $\bm{4.32 \pm 0.18}$ \%
\\
%\hline
2.0 & $3.52 \pm 0.23$ \%
\\
\hline
\end{tabular}}
\caption{50\% Hierarchical Corruption in \\ CIFAR-100 with {ResNet-34}}
\label{tab:2b}
\vspace{0.2 cm}
\end{subtable}%
}
{
\begin{subtable}{.5\linewidth}\centering
{\begin{tabular}{|c|c|}
\hline
$\xi$ & \makecell{Improvement of \\ student over teacher}
\\
\hline
0.2 & $1.10 \pm 0.09$ \%
\\
%\hline
0.5 & $2.69 \pm 0.02$ \%
\\
%\hline
0.7 & $3.72 \pm 0.05$ \%
\\
%\hline
1.0 & $5.29 \pm 0.11$ \%
\\
%\hline
1.2 & $6.26 \pm 0.09$ \%
\\
%\hline
\rowcolor{Gray}
\textbf{1.5} & $\bm{7.20 \pm 0.14}$ \%
\\
%\hline
\rowcolor{Gray}
\textbf{1.7} & $\bm{7.23 \pm 0.17}$ \%
\\
%\hline
2.0 & $6.42 \pm 0.26$ \%
\\
\hline
\end{tabular}}
\caption{50\% Hierarchical Corruption in \\ CIFAR-100 with VGG-16}
\label{tab:3b}
\vspace{0.2 cm}
\end{subtable}%
}
\\
\\
%\vspace{0.5 cm}
{
\begin{subtable}{.5\linewidth}\centering
\begin{tabular}{|c|c|}
%\toprule
\hline
$\xi$ & \makecell{Improvement of \\ student over teacher}
\\
\hline
0.2 & $0.13 \pm 0.08$ \%
\\
%\hline
0.5 & $0.97 \pm 0.04$ \% 
\\
%\hline
0.7 & $1.45 \pm 0.01$ \%
\\
%\hline
\rowcolor{Gray}
\textbf{1.0} & $\bm{1.85 \pm 0.09}$ \% 
\\
%\hline
\rowcolor{Gray}
\textbf{1.2} & $\bm{1.87 \pm 0.06}$ \%
\\
%\hline
\rowcolor{Gray}
\textbf{1.5} & $\bm{1.86 \pm 0.08}$ \%
\\
%\hline
\rowcolor{Gray}
\textbf{1.7} & $\bm{1.80 \pm 0.05}$ \%
\\
%\hline
2.0 & $1.53 \pm 0.02$ \%
\\
\hline
\end{tabular}
\caption{50\% Adversarial Corruption in Food-101 \\ with ResNet-34}
\label{tab:2c}
\vspace{0.2 cm}
\end{subtable}%
}
{
\begin{subtable}{.5\linewidth}\centering
\begin{tabular}{|c|c|}
%\toprule
\hline
$\xi$ & \makecell{Improvement of \\ student over teacher}
\\
\hline
0.2 & $0.79 \pm 0.23$ \%
\\
%\hline
0.5 & $2.14 \pm 0.09$ \%
\\
%\hline
0.7 & $2.96 \pm 0.04$ \%
\\
%\hline
{1.0} & $3.85 \pm 0.05$ \%
\\
%\hline
\rowcolor{Gray}
\textbf{1.2} & $\bm{4.22 \pm 0.15}$ \%
\\
%\hline
\rowcolor{Gray}
\textbf{1.5} & $\bm{4.39 \pm 0.29}$ \%
\\
%\hline
\rowcolor{Gray}
\textbf{1.7} & $\bm{4.20 \pm 0.34}$ \%
\\
%\hline
2.0 & $3.53 \pm 0.49$ \%
\\
\hline
\end{tabular}
\caption{50\% Adversarial Corruption in Food-101 \\ with VGG-16}
\label{tab:3c}
\vspace{0.2 cm}
\end{subtable}%
}
\caption{Average ($\pm$ 1 std.) improvement of student over teacher (i.e., student's test set accuracy - teacher's test set accuracy) with different values of the imitation parameter $\xi$; recall that $\xi=0$ corresponds to the teacher. Observe that in all cases, the value of $\xi$ yielding the biggest improvement is more than 1 (although in Food-101 with ResNet-34, $\xi=1$ does just as well as $\xi > 1$). This is consistent with our message in \Cref{rmk-xi}, where we advocate trying $\xi > 1$ in the high noise regime.}
\label{tab-expts2}
\end{table}

\clearpage

\begin{table}[!htb]
{
\begin{subtable}{\linewidth}\centering
{\begin{tabular}{|c|c|c|}
\hline
\makecell{Corruption level} & \makecell{\textbf{Random} corruption: \\ Improvement of student} & \makecell{\textbf{Adversarial} corruption: \\ Improvement of student} 
\\
\hline 
0\% & $-0.04 \pm 0.02$ \% & $-0.04 \pm 0.02$ \%
\\
%\hline 
10\% & $2.51 \pm 0.11$ \% & $2.32 \pm 0.10$ \%
\\
%\hline
30\% & $6.14 \pm 0.16$ \% & $5.08 \pm 0.25$ \%
\\
%\hline
50\% & $8.54 \pm 0.29$ \% & $5.77 \pm 0.19$ \%
\\
\hline
\end{tabular}}
\caption{Caltech-256 (Random and Adversarial Corruption)}
\vspace{0.2 cm}
\label{tab:1a}
\end{subtable}%
}
{
\begin{subtable}{\linewidth}\centering
{\begin{tabular}{|c|c|c|}
\hline
\makecell{Corruption level} & \makecell{\textbf{Random} corruption: \\ Improvement of student} & \makecell{\textbf{Hierarchical} corruption: \\ Improvement of student}
\\
\hline 
0\% & $-0.23 \pm 0.06$ \% & $-0.23 \pm 0.06$ \%
\\
%\hline 
10\% & $0.63 \pm 0.11$ \% & $1.19 \pm 0.08$ \%
\\
%\hline
30\% & $1.34 \pm 0.13$ \% & $2.80 \pm 0.06$ \%
\\
%\hline
50\% & $2.11 \pm 0.15$ \% & $4.19 \pm 0.09$ \%
\\
\hline
\end{tabular}}
\caption{CIFAR-100 (Random and Hierarchical Corruption)}
\vspace{0.2 cm}
\label{tab:1c}
\end{subtable}%
}
{
\begin{subtable}{\linewidth}\centering
{\begin{tabular}{|c|c|c|}
\hline
\makecell{Corruption level} & \makecell{\textbf{Random} corruption: \\ Improvement of student} & \makecell{\textbf{Adversarial} corruption: \\ Improvement of student}
\\
\hline 
0\% & $-0.37 \pm 0.10$ \% & $-0.37 \pm 0.10$ \%
\\
%\hline 
10\% & $0.10 \pm 0.04$ \% & $0.25 \pm 0.05$ \%
\\
%\hline
30\% & $0.47 \pm 0.04$ \% & $0.77 \pm 0.06$ \%
\\
%\hline
50\% & $1.12 \pm 0.08$ \% & $1.85 \pm 0.09$ \%
\\
\hline
\end{tabular}}
\caption{Food-101 (Random and Adversarial Corruption)}
\label{tab:1b}
\end{subtable}%
}
\caption{\textbf{ResNet-34 with $\xi=1$:} Average ($\pm$ 1 std.) improvement of student over teacher (i.e., student's test set accuracy - teacher's test set accuracy) with different kinds and varying levels of label corruption. Observe that \textit{as the corruption level increases, so does the improvement of the student over the teacher} for all types of corruption. This shows that the utility of the teacher's predictions (which is the core idea of SD) increases with the amount of label noise corroborating our claim in \Cref{dist-utility}.}
\label{tab-expts1}
\end{table}

\begin{figure*}[!htb]
\centering 
\subfloat[Caltech-256 with 50\% \\ random corruption]{
	\includegraphics[width=0.39\textwidth]{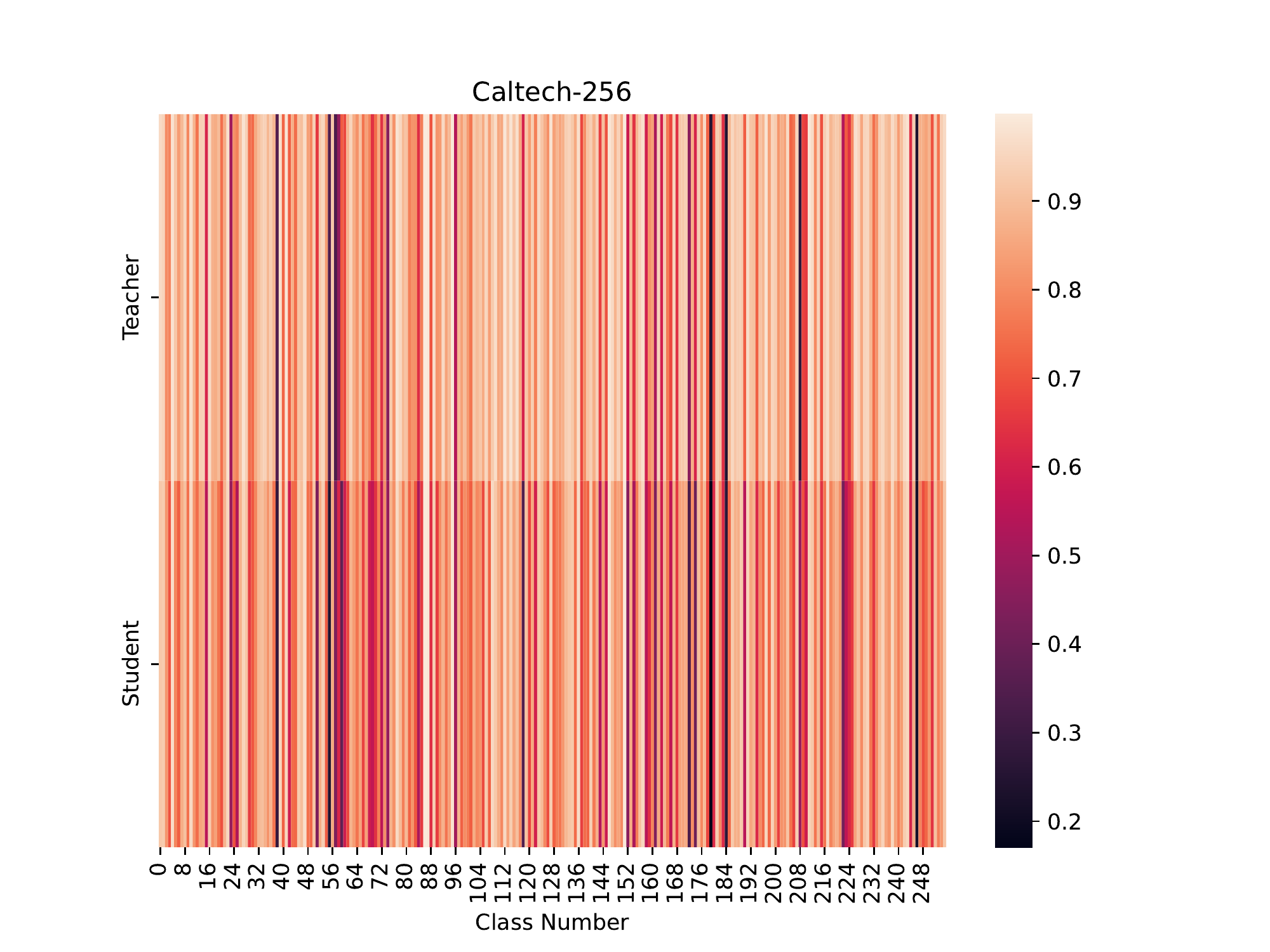}
}
\kern-2em
\subfloat[CIFAR-100 with 50\% \\ hierarchical corruption]{
	\includegraphics[width=0.39\textwidth]{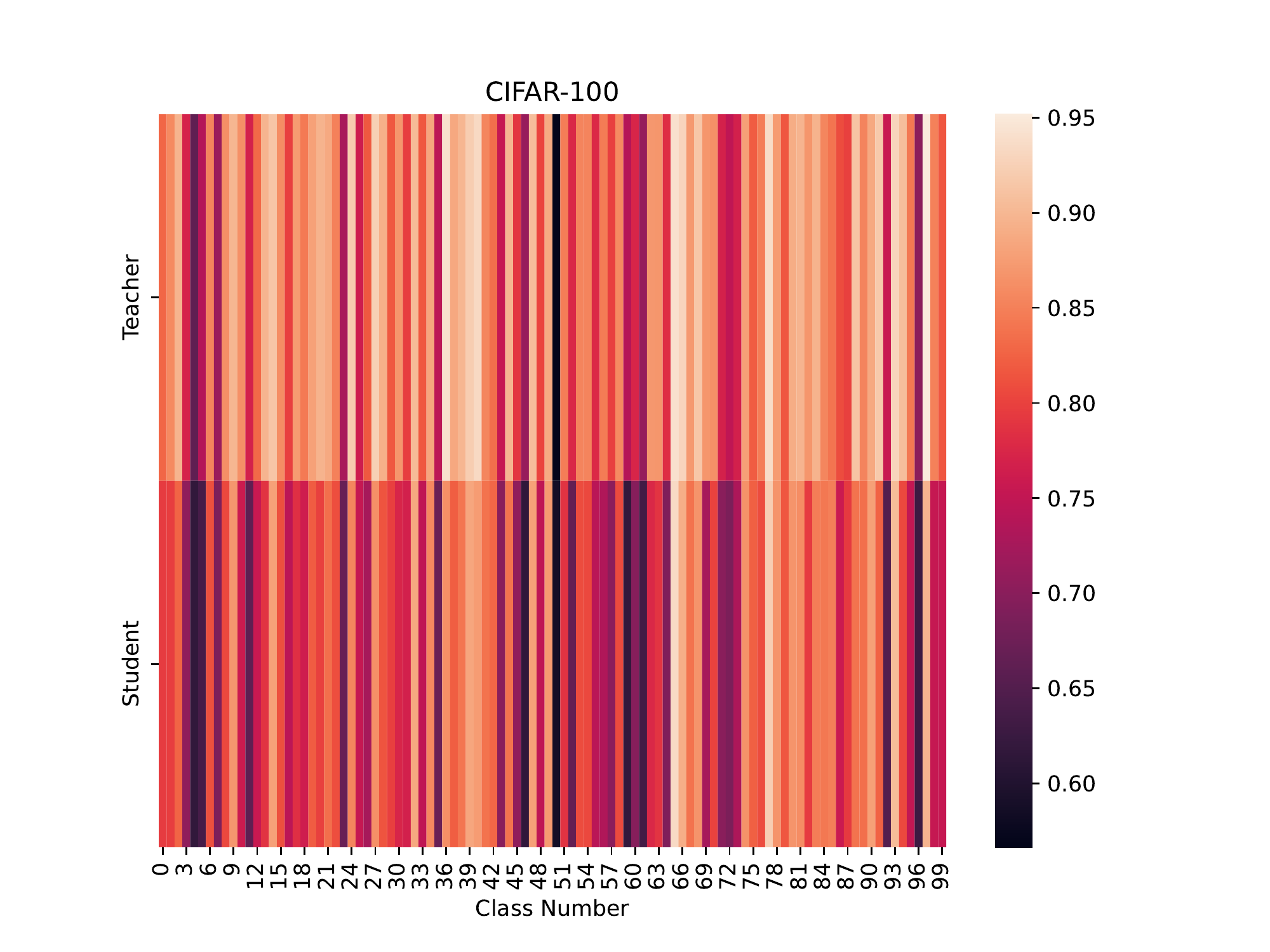}
}
\kern-2em
\subfloat[Food-101 with 50\% \\ adversarial corruption]{
	\includegraphics[width=0.39\textwidth]{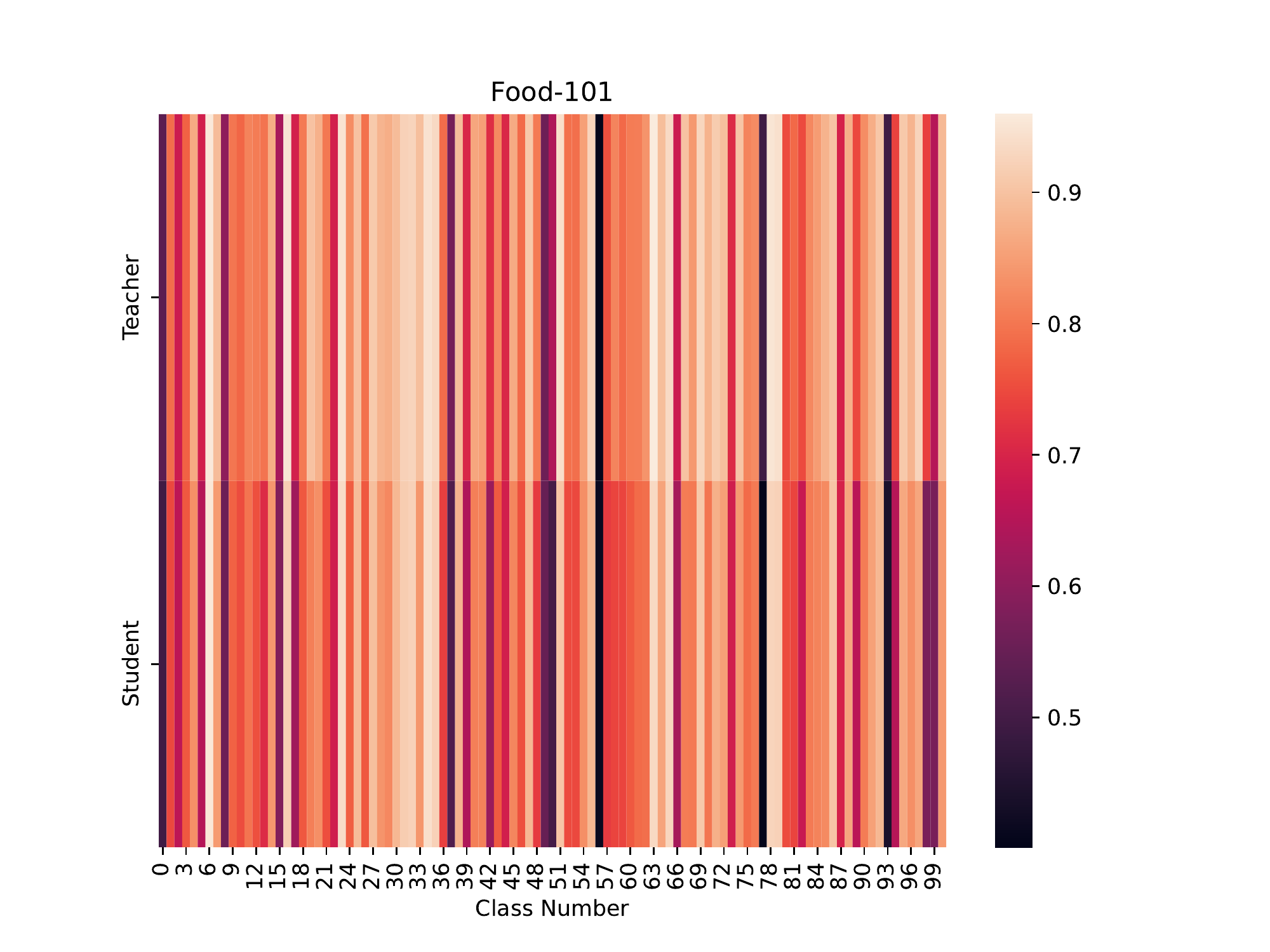}
}
\caption{
{\textbf{ResNet-34 with $\xi=1$:} Comparison of the per-class variability of the teacher and student (i.e., range of the teacher's and student's predictions of belonging to the correct class, as defined in \Cref{sec-expts-1}) for three of the cases of \Cref{tab-expts1} as a heat map. Note that a darker shade corresponds to a lower value; in all the cases, the student's heat map has a darker shade than the teacher's heat map which means that the student has a smaller variability than the teacher. This is consistent with the claim in  \Cref{rmk-3}.}
}
\label{fig:2}
\end{figure*}

\clearpage

\section{Conclusion}
In this work, we analyzed the utility of self-distillation (SD) in supervised learning with noisy labels. Our main algorithmic contribution was introducing the idea of trying $\xi > 1$ in the high label noise regime. {On the theoretical side, for a binary classification problem where some fraction of the sample's labels are flipped, we quantified the range of label corruption fraction in which the student outperforms the teacher under some assumptions on the data. We also characterized when optimal SD is better than optimal regularization in linear regression.}

There are some limitations of our work which pave the way for interesting directions of future work. Our results in \Cref{sec:log_reg} for logistic regression are under \Cref{a3}; it would be nice to derive similar results under a weaker assumption such as in expectation (see \Cref{a3}$'$ in the discussion after \Cref{a3}) or by assuming that the feature inner products are bounded in some range. Also, our results for logistic regression are with $\xi = 1$; one could try to obtain results with a general $\xi$ to shed some light on how to better tune $\xi$ for noisy datasets, like we did for linear regression. Further, our empirical results are with linear probing; experiments with full network fine-tuning are left for future work.

\section{Acknowledgement}
This work was supported by NSF TRIPODS grant 1934932.

%\clearpage

\bibliography{refs}
\bibliographystyle{apalike}

\clearpage

\appendix

\begin{center}
    \textbf{\Large Appendix}\vspace{5mm}
\end{center}

\section*{Contents}
\begin{itemize}
    \item \textbf{\Cref{pf-thm-est-err}:} Proof of Theorem~\ref{thm-est-err}
    \item \textbf{\Cref{xi_increase}:} Behavior of {$\xi^{*}$} w.r.t. {$\gamma^2$}
    \item \textbf{\Cref{pf-opt-lambda}:} Detailed Version and Proof of Theorem~\ref{thm-opt-lambda}
    \item \textbf{\Cref{app=cond-simpler}:} Detailed Version and Proof of \Cref{cond-simpler}
    \item \textbf{\Cref{thm-ex-pf}:} Proof of \Cref{thm-ex}
    \item \textbf{\Cref{a3-motiv}:} Empirical Motivation for Assumption~\ref{a3}
    \item \textbf{\Cref{pf-thm1}:} Proof of Theorem~\ref{thm1}
    \item \textbf{\Cref{pf-rmk-3}:} Proof of Corollary~\ref{rmk-3}
    \item \textbf{\Cref{more-expts-supp}:} {More Empirical Results}
    \item \textbf{\Cref{expt-detail}:} {Detailed Empirical Results}
\end{itemize}

\clearpage

\section{Proof of Theorem~\ref{thm-est-err}}
\label{pf-thm-est-err}
With the SVD notation of $\bm{X}$, we can rewrite $\hat{\bm{\theta}}_{S}(\xi)$ (from \cref{prop-2}) as:
\begin{equation}
    \label{student-svd}
    \hat{\bm{\theta}}_{S}(\xi) = {\sum_{j=1}^r \frac{\langle \bm{\theta}^{*}, \bm{u}_j \rangle}{\big(1 + \nicefrac{\lambda}{\sigma_j^2}\big)} {\Bigg(1 - \xi\Big(\frac{\nicefrac{\lambda}{\sigma_j^2}}{1+\nicefrac{\lambda}{\sigma_j^2}}\Big) \Bigg)} \bm{u}_j} + {\sum_{j=1}^r \frac{{\langle \bm{\eta}, \bm{v}_j \rangle}/{\sigma_j}}{{\big(1 + \nicefrac{\lambda}{\sigma_j^2}\big)}} {\Bigg(1 - \xi\Big(\frac{\nicefrac{\lambda}{\sigma_j^2}}{1+\nicefrac{\lambda}{\sigma_j^2}}\Big)\Bigg)} \bm{u}_j}.
\end{equation}
Also, since $\{\bm{u}_1,\ldots,\bm{u}_d\}$ forms an orthonormal basis for $\mathbb{R}^d$, we have: 
\[\bm{\theta}^{*} = \sum_{j=1}^d \langle \bm{\theta}^{*}, \bm{u}_j \rangle \bm{u}_j.\]
So, using \cref{student-svd}:
\begin{multline}
    \bm{\epsilon}_{S}(\xi) = -{\sum_{j=1}^r \langle \bm{\theta}^{*}, \bm{u}_j \rangle \Bigg(\frac{\nicefrac{\lambda}{\sigma_j^2}}{1 + \nicefrac{\lambda}{\sigma_j^2}}\Bigg) \Bigg(1 + \frac{\xi}{1 + \nicefrac{\lambda}{\sigma_j^2}}\Bigg) \bm{u}_j} - \sum_{j=r+1}^d \langle \bm{\theta}^{*}, \bm{u}_j \rangle \bm{u}_j 
    \\
    + {\sum_{j=1}^r \frac{{\langle \bm{\eta}, \bm{v}_j \rangle}/{\sigma_j}}{{\big(1 + \nicefrac{\lambda}{\sigma_j^2}\big)^2}} {\Bigg(1 - \xi\Big(\frac{\nicefrac{\lambda}{\sigma_j^2}}{1+\nicefrac{\lambda}{\sigma_j^2}}\Big)\Bigg)} \bm{u}_j}.
\end{multline}
Using \Cref{asmp-noise}, we have:
\begin{equation}
    \mathbb{E}_{\bm{\eta}}[\bm{\epsilon}_{S}(\xi)] = -{\sum_{j=1}^r \langle \bm{\theta}^{*}, \bm{u}_j \rangle \Bigg(\frac{\nicefrac{\lambda}{\sigma_j^2}}{1 + \nicefrac{\lambda}{\sigma_j^2}}\Bigg) \Bigg(1 + \frac{\xi}{1 + \nicefrac{\lambda}{\sigma_j^2}}\Bigg) \bm{u}_j} - \sum_{j=r+1}^d \langle \bm{\theta}^{*}, \bm{u}_j \rangle \bm{u}_j.
\end{equation}
Thus, using the orthonormality of $\{\bm{u}_1,\ldots,\bm{u}_d\}$, we get:
\begin{flalign}
    \label{eq:37-oct23}
    \big\|\mathbb{E}_{\bm{\eta}}[\bm{\epsilon}_{S}(\xi)]\big\|^2 & = \sum_{j=1}^r \big(\langle \bm{\theta}^{*}, \bm{u}_j \rangle\big)^2 \Bigg(\frac{\nicefrac{\lambda}{\sigma_j^2}}{1 + \nicefrac{\lambda}{\sigma_j^2}}\Bigg)^2 \Bigg(1 + \frac{\xi}{1 + \nicefrac{\lambda}{\sigma_j^2}}\Bigg)^2 + \sum_{j=r+1}^d \big(\langle \bm{\theta}^{*}, \bm{u}_j \rangle\big)^2.
\end{flalign}
Next:
\begin{flalign}
    \mathbb{E}_{\bm{\eta}}\Big[\big\|\bm{\epsilon}_{S}(\xi) - \mathbb{E}_{\bm{\eta}}[\bm{\epsilon}_{S}(\xi)]\big\|^2\Big] &= \mathbb{E}_{\bm{\eta}}\Bigg[\Bigg\|{\sum_{j=1}^r \frac{{\langle \bm{\eta}, \bm{v}_j \rangle}/{\sigma_j}}{{\big(1 + \nicefrac{\lambda}{\sigma_j^2}\big)^2}} {\Bigg(1 - \xi\Big(\frac{\nicefrac{\lambda}{\sigma_j^2}}{1+\nicefrac{\lambda}{\sigma_j^2}}\Big)\Bigg)} \bm{u}_j}\Bigg\|^2\Bigg]
    \\
    \label{eq:41-oct23}
    &= \sum_{j=1}^r \frac{\mathbb{E}_{\bm{\eta}}\big[\big(\langle \bm{\eta}, \bm{v}_j \rangle\big)^2\big]}{\sigma_j^2 \big(1 + \nicefrac{\lambda}{\sigma_j^2}\big)^2} {\Bigg(1 - \xi\Big(\frac{\nicefrac{\lambda}{\sigma_j^2}}{1+\nicefrac{\lambda}{\sigma_j^2}}\Big)\Bigg)}^2
    \\
    \label{eq:42-oct23}
    &= \sum_{j=1}^r \frac{\bm{v}_j^T \mathbb{E}_{\bm{\eta}}\big[\bm{\eta} \bm{\eta}^T\big] \bm{v}_j}{\sigma_j^2 \big(1 + \nicefrac{\lambda}{\sigma_j^2}\big)^2} {\Bigg(1 - \xi\Big(\frac{\nicefrac{\lambda}{\sigma_j^2}}{1+\nicefrac{\lambda}{\sigma_j^2}}\Big)\Bigg)}^2
    \\
    \label{eq:43-oct23}
    & = \gamma^2 \Bigg\{\sum_{j=1}^r \frac{1}{\sigma_j^2 \big(1 + \nicefrac{\lambda}{\sigma_j^2}\big)^2}{\Bigg(1 - \xi\Big(\frac{\nicefrac{\lambda}{\sigma_j^2}}{1+\nicefrac{\lambda}{\sigma_j^2}}\Big)\Bigg)}^2\Bigg\}.
\end{flalign}
\Cref{eq:41-oct23} follows from the orthonormality of the $\bm{u}_j$'s, \cref{eq:42-oct23} follows because the $\bm{v}_j$'s are independent of $\bm{\eta}$ from \Cref{asmp-noise}, and \cref{eq:43-oct23} follows because $\mathbb{E}_{\bm{\eta}}\big[\bm{\eta} \bm{\eta}^T\big] = \gamma^2 \text{I}_n$ from \Cref{asmp-noise} and because $\bm{v}_j^T \bm{v}_j = 1$ for all $j \in \{1,\ldots,r\}$.
Rewriting \cref{eq:43-oct23} slightly differently, we get:
\begin{equation}
    \label{eq:40-jan2}
    \mathbb{E}_{\bm{\eta}}\Big[\big\|\bm{\epsilon}_{S}(\xi) - \mathbb{E}_{\bm{\eta}}[\bm{\epsilon}_{S}(\xi)]\big\|^2\Big] = {\frac{\gamma^2}{\lambda} \Bigg\{\sum_{j=1}^r \frac{\nicefrac{\lambda}{\sigma_j^2}}{\big(1 + \nicefrac{\lambda}{\sigma_j^2}\big)^2} \Bigg(1 - \xi \Big(\frac{\nicefrac{\lambda}{\sigma_j^2}}{1 + \nicefrac{\lambda}{\sigma_j^2}}\Big)\Bigg)^2\Bigg\}}.
\end{equation}
%Adding \cref{eq:37-oct23} and \cref{eq:40-jan2} gives us $\mathbb{E}_{\bm{\eta}}\big[\|\bm{\epsilon}_{S}(\xi)\|^2\big]$.

\section{Behavior of \texorpdfstring{$\xi^{*}$}{TEXT} w.r.t. \texorpdfstring{$\gamma^2$}{TEXT}}
\label{xi_increase}
\begin{proposition}
    \label{prop_xi_increase}
    $\xi^{*}$ (in \Cref{opt-xi}) is an increasing function of $\gamma^2$.
\end{proposition}
\begin{proof}
Let $\rho = \gamma^2$. Then from \Cref{opt-xi}:
\begin{equation}
\xi^{*} = \frac{\sum_{j=1}^r\big(\frac{\rho}{\lambda} - \theta_j^{*}\big)\frac{c_j^2}{(1+c_j)^3}}{\sum_{j=1}^r\big(\frac{\rho}{\lambda}c_j + \theta_j^{*}\big)\frac{c_j^2}{(1+c_j)^4}}.
\end{equation}
Now,
\begin{equation}
    \frac{\partial \xi^{*}}{\partial \rho} = \frac{\Big(\sum_{j=1}^r\frac{c_j^2}{(1+c_j)^3} \Big)\Big(\sum_{j=1}^r\frac{\theta_j^{*} c_j^2}{(1+c_j)^4}\Big) + \Big(\sum_{j=1}^r\frac{c_j^3}{(1+c_j)^4} \Big)\Big(\sum_{j=1}^r\frac{\theta_j^{*} c_j^2}{(1+c_j)^3}\Big)}{\lambda\Big(\sum_{j=1}^r\big(\frac{\rho}{\lambda}c_j + \theta_j^{*}\big)\frac{c_j^2}{(1+c_j)^4}\Big)^2} > 0.
\end{equation}
Thus, $\xi^{*}$ is an increasing function of $\rho$, i.e., $\gamma^2$.
\end{proof}

\section{Detailed Version and Proof of Theorem~\ref{thm-opt-lambda}}
\label{pf-opt-lambda}
\begin{theorem}[\textbf{Detailed Version of Theorem~\ref{thm-opt-lambda}}]
    \label{thm-opt-lambda-2}
    The following hold with $\theta_j^{*} := \big(\langle \bm{\theta}^{*}, \bm{u}_j \rangle\big)^2$ (and with $'$ denoting the derivative w.r.t. $\lambda$):
    \begin{multline}
    \label{eq:31-jan7}
    e_\textup{sd}(\lambda) = e_\textup{reg}(\lambda) - \frac{\big(e_\text{reg}'(\lambda)\big)^2}{h(\lambda)} \text{ and } e_\textup{sd}'(\lambda) = e_\text{reg}'(\lambda)\Big(1 - \frac{2 e_\text{reg}''(\lambda)}{h(\lambda)} + \frac{e_\text{reg}'(\lambda) h'(\lambda)}{(h(\lambda))^2}\Big), \\ \text{ where } e_\textup{reg}(\lambda) = \sum_{j=1}^r \frac{\lambda^2 \theta_j^{*}}{(\lambda + \sigma_j^2)^2} + \sum_{j=r+1}^d  \theta_j^{*} + \sum_{j=1}^r \frac{{\gamma^2} \sigma_j^2}{(\lambda + \sigma_j^2)^2} \text{ and } h(\lambda) = 4 {\sum_{j=1}^r\Big(\frac{\gamma^2}{\sigma_j^2} + \theta_j^{*}\Big)\frac{\sigma_j^4}{(\lambda + \sigma_j^2)^4}}.
    \end{multline}
    Let $\lambda^{*}_\text{reg} := \text{arg min}_{\lambda} e_\textup{reg}(\lambda)$. Then, $e_\textup{sd}(\lambda^{*}_\text{reg}) = e_\textup{reg}(\lambda^{*}_\text{reg})$ and $e_\textup{sd}'(\lambda^{*}_\text{reg}) = 0$, i.e., $\lambda = \lambda^{*}_\text{reg}$ is a stationary point of $e_\textup{sd}(\lambda)$ also. It is a \textbf{local maximum} point of $e_\textup{sd}(\lambda)$ when:
    \begin{equation}
        \sum_{k=1}^r\sum_{j=1}^{k-1}\frac{\sigma_j^2 \sigma_k^2 (\sigma_j^2 - \sigma_k^2) (\theta_k^{*} - \theta_j^{*})}{(\lambda^{*}_\text{reg} + \sigma_j^2)^4 (\lambda^{*}_\text{reg} + \sigma_k^2)^4} < 0.
    \end{equation}
    When the above holds\footnote{Also, assume that $\lambda^{*}_\text{reg} \geq 0$ as the $\ell_2$-regularization parameter is supposed to be non-negative.}, optimal self-distillation is better than optimal $\ell_2$-regularization.
\end{theorem}
\noindent Note that if $\lambda = \lambda^{*}_\text{reg}$ is {not} a {local maximum} point of $e_\textup{sd}(\lambda)$, it could be a sub-optimal \textit{local} minimum point or the \textit{global} minimum point of $e_\textup{sd}(\lambda)$. 
The other stationary points of $e_\textup{sd}(\lambda)$ are obtained by solving (this follows from \cref{eq:31-jan7}):
\begin{equation}
    \label{eq:43-jan6}
    1 - \frac{2 e_\text{reg}''(\lambda)}{h(\lambda)} + \frac{e_\text{reg}'(\lambda) h'(\lambda)}{(h(\lambda))^2} = 0.
\end{equation}
Unfortunately, it seems difficult to determine whether a root of \cref{eq:43-jan6} or $\lambda^{*}_\text{reg}$ will be the global minimum point of $e_\textup{sd}(\lambda)$. If $\lambda^{*}_\text{reg}$ is the global minimum point of $e_\textup{sd}(\lambda)$, then optimal SD is \textbf{not} better than (i.e., does not yield any improvement over) optimal $\ell_2$-regularization as $e_\textup{sd}(\lambda^{*}_\text{reg}) = e_\textup{reg}(\lambda^{*}_\text{reg})$.

\begin{proof}
Using \cref{eq:bias} and \cref{eq:var} in \cref{eq:13-jan2} while using our notation of $c_j = \nicefrac{\lambda}{\sigma_j^2}$ and $\theta_j^{*} = \big(\langle \bm{\theta}^{*}, \bm{u}_j \rangle\big)^2$ from \Cref{opt-xi}, we get:
\begin{equation}
    e(\lambda, \xi) = \sum_{j=1}^r \theta_j^{*} \Bigg(\frac{c_j}{1 + c_j}\Bigg)^2 \Bigg(1 + \frac{\xi}{1 + c_j}\Bigg)^2 + \sum_{j=r+1}^d  \theta_j^{*} 
    + {\frac{\gamma^2}{\lambda} \Bigg\{\sum_{j=1}^r \frac{c_j}{(1 + c_j)^2} \Bigg(1 - \xi \Bigg(\frac{c_j}{1 + c_j}\Bigg)\Bigg)^2\Bigg\}}.
\end{equation}
Thus,
\begin{equation}
    \label{eq:31-jan6}
    e_\text{reg}(\lambda) :=  e(\lambda, 0) = \sum_{j=1}^r \theta_j^{*} \Bigg(\frac{c_j}{1 + c_j}\Bigg)^2 + \sum_{j=r+1}^d  \theta_j^{*} + \frac{\gamma^2}{\lambda} \sum_{j=1}^r \frac{c_j}{(1 + c_j)^2}.
\end{equation}
Next, we compute $e_\text{sd}(\lambda) := e(\lambda, \xi^{*})$.
\begin{lemma}
    \label{prop-sd-err}
    \begin{equation}
        \label{prop-sd-err-eq}
        e_\textup{sd}(\lambda) = e_\textup{reg}(\lambda) - \frac{\Big(\sum_{j=1}^r\big(\theta_j^{*} - \frac{\gamma^2}{\lambda}\big)\frac{c_j^2}{(1+c_j)^3}\Big)^2}{\sum_{j=1}^r\big(\frac{\gamma^2}{\lambda}c_j + \theta_j^{*}\big)\frac{c_j^2}{(1+c_j)^4}}.
    \end{equation}
\end{lemma}
\noindent \Cref{prop-sd-err} involves a little bit of algebra; we prove it in \Cref{prop-sd-err-pf}. 
\\
\\
Since the $c_j$'s depend on $\lambda$, let us substitute $c_j$ in \cref{eq:31-jan6} and \cref{prop-sd-err-eq} and rewrite them.
\begin{equation}
    \label{eq:31-jan6-2}
    e_\text{reg}(\lambda) = \sum_{j=1}^r \frac{\lambda^2 \theta_j^{*}}{(\lambda + \sigma_j^2)^2} + \sum_{j=r+1}^d  \theta_j^{*} + \sum_{j=1}^r \frac{{\gamma^2} \sigma_j^2}{(\lambda + \sigma_j^2)^2}.
\end{equation}
\begin{equation}
    \label{prop-sd-err-eq-2}
    e_\textup{sd}(\lambda) = e_\textup{reg}(\lambda) - \Bigg(\Bigg(\underbrace{\sum_{j=1}^r\big(\lambda \theta_j^{*} - {\gamma^2}\big)\frac{\sigma_j^2}{(\lambda + \sigma_j^2)^3}}_{:=g(\lambda)}\Bigg)^2\Bigg)\Bigg/\Bigg({\sum_{j=1}^r\big(\frac{\gamma^2}{\sigma_j^2} + \theta_j^{*}\big)\frac{\sigma_j^4}{(\lambda + \sigma_j^2)^4}}\Bigg).
\end{equation}
Interestingly, it can be checked that $g(\lambda) = \frac{1}{2}{e_\text{reg}'(\lambda)}$; here $'$ indicates the derivative w.r.t. $\lambda$. Plugging this in \cref{prop-sd-err-eq-2}, we get:
\begin{equation}
    \label{prop-sd-err-eq-3}
    e_\textup{sd}(\lambda) = e_\textup{reg}(\lambda) - \frac{\big(e_\text{reg}'(\lambda)\big)^2}{h(\lambda)}, \text{ where } h(\lambda) = 4 {\sum_{j=1}^r\Big(\frac{\gamma^2}{\sigma_j^2} + \theta_j^{*}\Big)\frac{\sigma_j^4}{(\lambda + \sigma_j^2)^4}}.
\end{equation}
Now note that:
\begin{equation}
    \label{eq:der}
    e_\textup{sd}'(\lambda) = e_\text{reg}'(\lambda)\Big(1 - \frac{2 e_\text{reg}''(\lambda)}{h(\lambda)} + \frac{e_\text{reg}'(\lambda) h'(\lambda)}{(h(\lambda))^2}\Big).
\end{equation}
Thus, $e_\text{reg}'(\lambda) = 0 \implies e_\textup{sd}'(\lambda) = 0$, i.e., any stationary point of $e_\text{reg}(\lambda)$ is also a stationary point of $e_\text{sd}(\lambda)$.

Next, $\lambda^{*}_\text{reg} := \text{arg min}_{\lambda} e_\textup{reg}(\lambda)$ satisfies:
\begin{equation}
    \label{eq:37-jan9}
    e_\text{reg}'(\lambda^{*}_\text{reg}) = 2\sum_{j=1}^r\big(\lambda^{*}_\text{reg} \theta_j^{*} - {\gamma^2}\big)\frac{\sigma_j^2}{(\lambda^{*}_\text{reg} + \sigma_j^2)^3} =  0.
\end{equation}
From \cref{eq:der}, $e_\textup{sd}'(\lambda^{*}_\text{reg}) = 0$, i.e., $\lambda = \lambda^{*}_\text{reg}$ is a stationary point of $e_\textup{sd}(\lambda)$ also. We shall now show that $\lambda = \lambda^{*}_\text{reg}$ can be a \textit{local maximum} point of $e_\textup{sd}(\lambda)$ in many cases. For that, we need to check the sign of $e_\textup{sd}''(\lambda^{*}_\text{reg})$. Note that:
\begin{equation}
    e_\textup{sd}''(\lambda^{*}_\text{reg}) = e_\textup{reg}''(\lambda^{*}_\text{reg}) \Bigg(1 - \frac{2e_\textup{reg}''(\lambda^{*}_\text{reg})}{h(\lambda^{*}_\text{reg})}\Bigg).
\end{equation}
The above follows by just differentiating \cref{eq:der} and evaluating it at $\lambda = \lambda^{*}_\text{reg}$ while using the fact that $e_\text{reg}'(\lambda^{*}_\text{reg}) = 0$. Also note that $e_\textup{reg}''(\lambda^{*}_\text{reg}) > 0$ as $\lambda = \lambda^{*}_\text{reg}$ is a minimizer of $e_\textup{reg}(\lambda)$. Let us now examine the sign of $t = \Big(1 - \frac{2e_\textup{reg}''(\lambda^{*}_\text{reg})}{h(\lambda^{*}_\text{reg})}\Big)$.  
After a bit of algebra:
\begin{equation}
    t = 1 - \frac{\sum_{j=1}^r \frac{\sigma_j^2}{(\lambda^{*}_\text{reg} + \sigma_j^2)^4} \big(\theta_j^{*} \sigma_j^2 + 3 \gamma^2 - 2\lambda^{*}_\text{reg} \theta_j^{*}\big)}{ {\sum_{j=1}^r \frac{\sigma_j^2}{(\lambda^{*}_\text{reg} + \sigma_j^2)^4} \Big({\gamma^2} + \theta_j^{*}{\sigma_j^2}\Big)}} = \frac{2 {\sum_{j=1}^r \frac{\sigma_j^2}{(\lambda^{*}_\text{reg} + \sigma_j^2)^4}\big(\lambda^{*}_\text{reg} \theta_j^{*} - {\gamma^2}\big)}}{{\sum_{j=1}^r \frac{\sigma_j^2}{(\lambda^{*}_\text{reg} + \sigma_j^2)^4} \Big({\gamma^2} + \theta_j^{*}{\sigma_j^2}\Big)}}
\end{equation}
The denominator of $t$ is positive so we only need to analyze the sign of the numerator, $ {\sum_{j=1}^r \frac{\sigma_j^2}{(\lambda^{*}_\text{reg} + \sigma_j^2)^4}\big(\lambda^{*}_\text{reg} \theta_j^{*} - {\gamma^2}\big)}$; let us refer to it as $t_2$ for brevity. From \cref{eq:37-jan9}, we have that:
\begin{equation}
    \label{eq:jan21-45}
    \lambda^{*}_\text{reg} = \frac{{\gamma^2} \sum_{j=1}^r\frac{\sigma_j^2}{(\lambda^{*}_\text{reg} + \sigma_j^2)^3}}{\sum_{j=1}^r\frac{\theta_j^{*} \sigma_j^2}{(\lambda^{*}_\text{reg} + \sigma_j^2)^3}}.
\end{equation}
Using this, we get:
\begin{equation}
    t_2 = \underbrace{\Bigg(\frac{\gamma^2}{\sum_{j=1}^r\frac{\theta_j^{*} \sigma_j^2}{(\lambda^{*}_\text{reg} + \sigma_j^2)^3}}\Bigg)}_{>0} \underbrace{\Bigg(\sum_{j,k}\frac{\sigma_j^2 \sigma_k^2 (\theta_j^{*} - \theta_k^{*})}{(\lambda^{*}_\text{reg} + \sigma_j^2)^4 (\lambda^{*}_\text{reg} + \sigma_k^2)^3}\Bigg)}_{:=t_3}
\end{equation}
Simplifying $t_3$ a bit, we get:
\begin{equation}
    t_3 = \sum_{k=1}^r\sum_{j=1}^{k-1}\frac{\sigma_j^2 \sigma_k^2 (\sigma_j^2 - \sigma_k^2) (\theta_k^{*} - \theta_j^{*})}{(\lambda^{*}_\text{reg} + \sigma_j^2)^4 (\lambda^{*}_\text{reg} + \sigma_k^2)^4}.
\end{equation}
So, $t_3 < 0 \implies t_2 < 0 \implies t < 0 \implies e_\textup{sd}''(\lambda^{*}_\text{reg}) < 0$; but this means $\lambda = \lambda^{*}_\text{reg}$ is a {local maximum} point of $e_\textup{sd}(\lambda)$.  
\end{proof}

\subsection{Proof of Lemma~\ref{prop-sd-err}}
\label{prop-sd-err-pf}
\begin{proof}
Note that $e(\lambda, \xi)$ is a quadratic function of $\xi$; specifically, it is of the form $a \xi^2 + b \xi + c$, where:
\begin{equation}
    \label{eq:32-jan9}
    a = \sum_{j=1}^r\big(\frac{\gamma^2}{\lambda}c_j + \theta_j^{*}\big)\frac{c_j^2}{(1+c_j)^4},  b = 2 \sum_{j=1}^r\Big(\theta_j^{*} - \frac{\gamma^2}{\lambda}\Big)\frac{c_j^2}{(1+c_j)^3}, \text{ and } c = e_\text{reg}(\lambda).
\end{equation}
By simple differentiation, $\xi^{*} = \text{arg min}_{\bm{\xi} \in \mathbb{R}} e(\lambda, \xi) = -\frac{b}{2a}$ (which is what we obtained in \Cref{opt-xi}). A little bit of algebra gives us:
\begin{equation}
    \label{eq:33-jan9}
    e(\lambda, \xi^{*}) = c - \frac{b^2}{4 a}. 
\end{equation}
Plugging in the values of $a$, $b$ and $c$ from \cref{eq:32-jan9} in  yields:
\begin{equation}
    e_\text{sd}(\lambda) := e(\lambda, \xi^{*}) = e_\text{reg}(\lambda) - \frac{\Big(\sum_{j=1}^r\big(\theta_j^{*} - \frac{\gamma^2}{\lambda}\big)\frac{c_j^2}{(1+c_j)^3}\Big)^2}{\sum_{j=1}^r\big(\frac{\gamma^2}{\lambda}c_j + \theta_j^{*}\big)\frac{c_j^2}{(1+c_j)^4}}.
\end{equation}
This finishes the proof.
\end{proof}

\section{Detailed Version and Proof of Theorem~\ref{cond-simpler}}
\label{app=cond-simpler}
\begin{theorem}[\textbf{Detailed Version of Theorem~\ref{cond-simpler}}]
\label{cond-simpler-full}
Without loss of generality, let $\|\bm{\theta}^{*}\| = 1$ and $\sigma_1 = 1$. Further, suppose $\sigma_j \leq \delta$ for $j \in \{q+1,\ldots,r\}$ and $\theta_1^{*} > \ldots > \theta_{q}^{*}$. Also, suppose $\lambda^{*}_\text{reg} > 0$. %Then, $\lambda = \lambda^{*}_\text{reg}$ is a \textbf{local maximum} point of $e_\textup{sd}(\lambda)$ when $\delta \leq \frac{1}{\sqrt{2} r} \sqrt{{\min_{k \in \{1,\ldots,q\}} \big(\sigma_k^2 (1 - \sigma_k^2) (\theta_1^{*} - \theta_k^{*})\big)}}$ and $\gamma^2 \geq \frac{\max_{j \in \{1,\ldots,r\}} \theta_{j}^{*}}{r-1}$.
For any $\nu > 1$, if $\delta \leq \frac{1}{\sqrt{2 \nu r}} \sqrt{{\min_{k \in \{1,\ldots,q\}} \big(\sigma_k^2 (1 - \sigma_k^2) (\theta_1^{*} - \theta_k^{*})\big)}}$ and $\gamma^2 \geq \frac{\max_{j \in \{1,\ldots,r\}} \theta_{j}^{*}}{\nu-1}$, then $\lambda = \lambda^{*}_\text{reg}$ is a \textbf{local maximum} point of $e_\textup{sd}(\lambda)$.
\end{theorem}
\noindent \Cref{cond-simpler} is obtained by using $\nu = r$ in \Cref{cond-simpler-full}.

\begin{proof}
Define $v_k := \sum_{j=1}^{k-1}\frac{\sigma_j^2 \sigma_k^2 (\sigma_j^2 - \sigma_k^2) (\theta_k^{*} - \theta_j^{*})}{(\lambda^{*}_\text{reg} + \sigma_j^2)^4 (\lambda^{*}_\text{reg} + \sigma_k^2)^4}$. For $\lambda = \lambda^{*}_\text{reg}$ to be a {local maximum} point of $e_\textup{sd}(\lambda)$, we must have $\sum_{k=1}^r v_k < 0$ as per \Cref{thm-opt-lambda}.

Let us analyze $v_k$ for $k > q$ first. Using $\sigma_k \leq \delta$ for $k > q$, $(\sigma_j^2 - \sigma_k^2) \leq \sigma_j^2 \leq \sigma_1^2 = 1$ for $j < k$ and $|\theta_k^{*} - \theta_j^{*}| \leq \|\bm{\theta}^{*}\|^2 = 1$, we get for $k > q$:
\begin{equation}
    \label{eq:jan21-51}
    |v_k| \leq \delta^2 \sum_{j=1}^{k-1} \frac{\sigma_j^2}{(\lambda^{*}_\text{reg} + \sigma_j^2)^4 (\lambda^{*}_\text{reg} + \sigma_k^2)^4}. 
\end{equation}
{Now since $\lambda^{*}_\text{reg} > 0$, we can further simplify \cref{eq:jan21-51}:}
\begin{equation}
    \label{eq:jan21-52}
    |v_k| \leq \delta^2 \sum_{j=1}^{k-1} \frac{\sigma_j^2}{(\lambda^{*}_\text{reg})^8} = \frac{\delta^2}{(\lambda^{*}_\text{reg})^8} \Bigg(\sum_{j=1}^{q}\underbrace{\sigma_j^2}_{\leq 1} + \sum_{j=q+1}^{r} \underbrace{\sigma_j^2}_{\leq \delta^2}\Bigg) \leq \frac{\delta^2(q + r \delta^2)}{(\lambda^{*}_\text{reg})^8}.
\end{equation}
Summing up \cref{eq:jan21-52} from $k=q+1$ through to $k=r$, we get:
\begin{equation}
    \label{eq:jan21-53}
    \sum_{k=q+1}^r v_k \leq \sum_{k=q+1}^r |v_k| \leq \frac{r \delta^2(q + r \delta^2)}{(\lambda^{*}_\text{reg})^8}.
\end{equation}
Let us now look at $k \leq q$. Since $\theta_1^{*} > \ldots > \theta_{q}^{*}$, we have that $v_k < 0$ for all $k \leq q$. Note that for each $k \leq q$:
\begin{equation}
    v_k \leq \frac{\sigma_k^2 (1 - \sigma_k^2) (\theta_k^{*} - \theta_1^{*})}{(\lambda^{*}_\text{reg} + 1)^4 (\lambda^{*}_\text{reg} + \sigma_k^2)^4} \leq \frac{\sigma_k^2 (1 - \sigma_k^2) (\theta_k^{*} - \theta_1^{*})}{(\lambda^{*}_\text{reg} + 1)^8},
\end{equation}
{where the last step follows using $\lambda^{*}_\text{reg} > 0$}. Thus,
\begin{equation}
    \label{eq:jan21-55}
    \sum_{k=1}^{q}v_k \leq \frac{1}{(\lambda^{*}_\text{reg} + 1)^8}\sum_{k=1}^{q}\sigma_k^2 (1 - \sigma_k^2) (\theta_k^{*} - \theta_1^{*}) \leq \frac{-q \min_{k \in \{1,\ldots,q\}} \big(\sigma_k^2 (1 - \sigma_k^2) (\theta_1^{*} - \theta_k^{*})\big)}{(\lambda^{*}_\text{reg} + 1)^8}.
\end{equation}
Using \cref{eq:jan21-53} and \cref{eq:jan21-55}, we get:
\begin{equation}
    \sum_{k=1}^r v_k = \sum_{k=1}^{q}v_k + \sum_{k=q+1}^r v_k \leq - \frac{q \min_{k \in \{1,\ldots,q\}} \big(\sigma_k^2 (1 - \sigma_k^2) (\theta_1^{*} - \theta_k^{*})\big)}{(\lambda^{*}_\text{reg} + 1)^8} + \frac{r \delta^2(q + r \delta^2)}{(\lambda^{*}_\text{reg})^8}.
\end{equation}
So to ensure $\sum_{k=1}^r v_k < 0$, ensuring:
\begin{equation}
    \frac{r \delta^2(q + r \delta^2)}{(\lambda^{*}_\text{reg})^8} < \frac{q \min_{k \in \{1,\ldots,q\}} \big(\sigma_k^2 (1 - \sigma_k^2) (\theta_1^{*} - \theta_k^{*})\big)}{(\lambda^{*}_\text{reg} + 1)^8}
\end{equation}
suffices. This implies:
\begin{equation}
    \frac{\lambda^{*}_\text{reg} + 1}{\lambda^{*}_\text{reg}} < \underbrace{\Bigg(\frac{q \min_{k \in \{1,\ldots,q\}} \big(\sigma_k^2 (1 - \sigma_k^2) (\theta_1^{*} - \theta_k^{*})\big)}{r \delta^2(q + r \delta^2)}\Bigg)^{1/8}}_{:=z}.
\end{equation}
For any $\nu > 1$, note that $z > \nu$ for $\delta^2 < \frac{1}{2 \nu r}{\min_{k \in \{1,\ldots,q\}} \big(\sigma_k^2 (1 - \sigma_k^2) (\theta_1^{*} - \theta_k^{*})\big)}$. In that case, we must have $\lambda^{*}_\text{reg} > \frac{1}{z-1}$, which can be ensured by having:
\begin{equation}
    \lambda^{*}_\text{reg} > \frac{1}{\nu-1}.
\end{equation}

From \cref{eq:37-jan9}, recall that $\lambda^{*}_\text{reg} = \frac{{\gamma^2} \sum_{j=1}^r\frac{\sigma_j^2}{(\lambda^{*}_\text{reg} + \sigma_j^2)^3}}{\sum_{j=1}^r\frac{\theta_j^{*} \sigma_j^2}{(\lambda^{*}_\text{reg} + \sigma_j^2)^3}}$. {Now since $\lambda^{*}_\text{reg} > 0$, we have that:}
\begin{equation}
    \lambda^{*}_\text{reg} \geq \frac{{\gamma^2} \sum_{j=1}^r\frac{\sigma_j^2}{(\lambda^{*}_\text{reg} + \sigma_j^2)^3}}{\theta_{\text{max}}^{*}\sum_{j=1}^r\frac{\sigma_j^2}{(\lambda^{*}_\text{reg} + \sigma_j^2)^3}} \geq \frac{\gamma^2}{\theta_{\text{max}}^{*}},
\end{equation}
where $\theta_{\text{max}}^{*} = \max_{j \in \{1,\ldots,r\}} \theta_{j}^{*}$. Using this, if $\gamma^2 > \frac{\theta_{\text{max}}^{*}}{\nu-1}$, then $\lambda^{*}_\text{reg} \geq \frac{\gamma^2}{\theta_{\text{max}}^{*}} > \frac{1}{\nu-1} > \frac{1}{z-1}$. This completes the proof.
\end{proof}

\section{Proof of Theorem~\ref{thm-ex}}
\label{thm-ex-pf}
\begin{proof}
    We provide a 2-dimensional example, i.e., $d=2$. Suppose $n > 2$. Take $\bm{\theta}^{*} = \frac{1}{\sqrt{2}}(\bm{u}_1 + \bm{u}_2)$; so, $\theta^{*}_1 = \theta^{*}_2 = \frac{1}{2}$. Also, suppose $\sigma_1 = 1$ and $\sigma_2 = \frac{1}{2}$. For this case, we get (by using the formulas in \Cref{thm-opt-lambda-2}):
    \begin{equation}
        e_\text{reg}(\lambda) = \frac{\lambda^2}{2} \Bigg(\frac{1}{(\lambda + 1)^2} + \frac{16}{(4\lambda + {1})^2}\Bigg) + \gamma^2 \Bigg(\frac{1}{(\lambda+1)^2} + \frac{4}{(4\lambda + {1})^2}\Bigg),
    \end{equation}
    and
    \begin{equation}
        \label{eq:62-jan24}
        e_\text{reg}'(\lambda) = (\lambda - 2 \gamma^2) \Bigg(\frac{1}{(\lambda+1)^3} + \frac{16}{(4\lambda+1)^3}\Bigg).
    \end{equation}
    From \cref{eq:62-jan24}, we have that $\lambda^{*}_\text{reg} = \text{arg min}_{{\lambda > 0}}e_\text{reg}(\lambda) = 2\gamma^2$.

    From \Cref{thm-opt-lambda-2}, we have that:
    \begin{equation}
        e_\textup{sd}'(\lambda) = e_\text{reg}'(\lambda)\Bigg(1 - \frac{2 e_\text{reg}''(\lambda)}{h(\lambda)} + \frac{e_\text{reg}'(\lambda) h'(\lambda)}{(h(\lambda))^2}\Bigg), \text{ where } h(\lambda) = \Bigg(\frac{4\gamma^2 + 2}{(\lambda+1)^4} + \frac{256 \gamma^2 + 32}{(4\lambda+1)^4}\Bigg).
    \end{equation}
    After a lot of algebraic heavy lifting, we get:
    \begin{equation}
        \label{eq:64-jan24}
        e_\textup{sd}'(\lambda) = \frac{288(\lambda - 2\gamma^2)^3}{(\lambda+1)^5 (4\lambda+1)^5 \Big(\frac{2\gamma^2 + 1}{(\lambda+1)^4} + \frac{128\gamma^2 + 16}{(4\lambda+1)^4}\Big)^2}\Bigg(\frac{1}{(\lambda+1)^3} + \frac{16}{(4\lambda+1)^3}\Bigg).
    \end{equation}
    Using \cref{eq:64-jan24}, we can conclude that $\text{arg min}_{{\lambda > 0}}e_\text{sd}(\lambda) = 2\gamma^2 = \lambda^{*}_\text{reg}$.
\end{proof}

\clearpage

\section{Empirical Motivation for Assumption~\ref{a3}}
\label{a3-motiv}
We consider the same logistic regression setting as \Cref{sec:log_reg}. 
Note that the Gram matrix $\bm{K} \in \mathbb{R}^{2n \times 2n}$ (w.r.t. $\phi(.)$) is of the form $\bm{K} = \left[ {\begin{array}{cc} 
\bm{K}_1 & \bm{0}_{n \times n} \\
\bm{0}_{n \times n} & \bm{K}_0 \\
\end{array} } \right]$, where $\bm{0}_{n \times n}$ is the $n \times n$ matrix of all 0's and $\bm{K}_1$ and $\bm{K}_0$ are both PSD matrices with diagonal entries = 1. For our simulations, the diagonal elements of $\bm{K}_1$ are set equal to 1 and the off-diagonal elements are set equal to the corresponding off-diagonal element of $\frac{1}{n}\bm{Z}_1\bm{Z}_1^T$, where each element of $\bm{Z}_1 \in \mathbb{R}^{n \times n}$ is drawn i.i.d. from \textbf{(i)} $\text{Unif}[0,1]$, and \textbf{(ii)} $\text{Bernoulli}(0.8)$\footnote{If $X \sim \text{Bernoulli}(p)$, then $\mathbb{P}(X=1) = p$ and $\mathbb{P}(X=0) = 1-p$.}. 
$\bm{K}_0$ is constructed in the same way. Note that $\bm{K}$ is PSD.
In the case of \textbf{(i)} (resp., \textbf{(ii)}), the expected off-diagonal element of both $\bm{K}_1$ and $\bm{K}_0$ is 0.25 (resp., 0.64), and so we compare against \Cref{a3} with $c=0.25$ (resp., $c=0.64$). Specifically, for our two Gram matrices, we compare the average predictions (average being over the training set) of our logistic regression model against the corresponding predictions under \Cref{a3}. We consider four values of $n$, namely, 1000, 5000, 10000 and 50000. 

In \Cref{tab-a3}, we show results for \textbf{(i)} when $p=0.45$ (top) and $p=0.35$ (bottom) with $\hat{\lambda} = 1-c$ (recall that $\hat{\lambda} \in \big[\frac{1-c}{2.16}, \frac{1-c}{0.40}\big]$ as per \Cref{thm1}). In \Cref{tab-a4}, we show results for \textbf{(ii)} when $p=0.3$ (top) and $p=0.2$ (bottom) with $\hat{\lambda} = \frac{1-c}{0.50} = 2(1-c)$.
Please see the table captions for a detailed discussion, but in summary, we conclude that \Cref{a3} is a reasonable assumption to analyze the average behavior of a linear model on a large dataset under random label corruption.

\begin{table}[!htb]
\centering
\begin{subtable}{\linewidth}\centering
\begin{tabular}{|c|c|c|c|c|c|}
%\toprule
\hline
\multirow{5}{*}{ \texttt{Teacher} } & $n$ & \makecell{Avg. pred. for \\ bad points} & \makecell{Pred. for \\ bad points under \\ A\ref{a3} \& $n \to \infty$} & \makecell{Avg. pred. for \\ good points} & \makecell{Pred. for \\ good points under \\ A\ref{a3} \& $n \to \infty$}
\\
\cline{2-6} & 1k & 0.4413 & \multirow{4}{*}{ \textbf{0.4400} } & 0.6372 & \multirow{4}{*}{ \textbf{0.6400} }
\\
\cline{2-3} \cline{5-5} & 5k & 0.4399 &  & 0.6397 &  
\\
\cline{2-3} \cline{5-5} & 10k & 0.4399 &  & 0.6399 & 
\\
\cline{2-3} \cline{5-5} & \textbf{50k} & \textbf{0.4400} &  & \textbf{0.6400} & 
\\
\hline
\end{tabular}
\begin{tabular}{|c|c|c|c|c|c|}
%\toprule
\hline
\multirow{5}{*}{ \texttt{Student} } & $n$ & \makecell{Avg. pred. for \\ bad points} & \makecell{Pred. for \\ bad points under \\ A\ref{a3} \& $n \to \infty$} & \makecell{Avg. pred. for \\ good points} & \makecell{Pred. for \\ good points under \\ A\ref{a3} \& $n \to \infty$}
\\
\cline{2-6} & 1k & 0.5287 & \multirow{4}{*}{ \textbf{0.5280} } & 0.5645 & \multirow{4}{*}{ \textbf{0.5680} }
\\
\cline{2-3} \cline{5-5} & 5k & 0.5279 &  & 0.5676 &  
\\
\cline{2-3} \cline{5-5} & 10k & 0.5279 &  & 0.5679 & 
\\
\cline{2-3} \cline{5-5} & \textbf{50k} & \textbf{0.5280} &  & \textbf{0.5680} & 
\\
\hline
\end{tabular}
\caption{$p=0.45$}
\vspace{0.2 cm}
\label{tab-a3-1}
\end{subtable}
\begin{subtable}{\linewidth}\centering
\begin{tabular}{|c|c|c|c|c|c|}
%\toprule
\hline
\multirow{5}{*}{ \texttt{Teacher} } & $n$ & \makecell{Avg. pred. for \\ bad points} & \makecell{Pred. for \\ bad points under \\ A\ref{a3} \& $n \to \infty$} & \makecell{Avg. pred. for \\ good points} & \makecell{Pred. for \\ good points under \\ A\ref{a3} \& $n \to \infty$}
\\
\cline{2-6} & 1k & 0.5243 & \multirow{4}{*}{ \textbf{0.5200} } & 0.7146 & \multirow{4}{*}{ \textbf{0.7200} }
\\
\cline{2-3} \cline{5-5} & 5k & 0.5198 &  & 0.7195 &  
\\
\cline{2-3} \cline{5-5} & 10k & 0.5198 &  & 0.7198 & 
\\
\cline{2-3} \cline{5-5} & \textbf{50k} & \textbf{0.5200} &  & \textbf{0.7200} & 
\\
\hline
\end{tabular}
\begin{tabular}{|c|c|c|c|c|c|}
%\toprule
\hline
\multirow{5}{*}{ \texttt{Student} } & $n$ & \makecell{Avg. pred. for \\ bad points} & \makecell{Pred. for \\ bad points under \\ A\ref{a3} \& $n \to \infty$} & \makecell{Avg. pred. for \\ good points} & \makecell{Pred. for \\ good points under \\ A\ref{a3} \& $n \to \infty$}
\\
\cline{2-6} & 1k & 0.6264 & \multirow{4}{*}{ \textbf{0.6240} } & 0.6568 & \multirow{4}{*}{ \textbf{0.6640} }
\\
\cline{2-3} \cline{5-5} & 5k & 0.6235 &  & 0.6631 &  
\\
\cline{2-3} \cline{5-5} & 10k & 0.6236 &  & 0.6636 & 
\\
\cline{2-3} \cline{5-5} & \textbf{50k} & \textbf{0.6240} &  & \textbf{0.6640} & 
\\
\hline
\end{tabular}
\caption{$p=0.35$}
\label{tab-a3-2}
\end{subtable}
\caption{\textbf{(i) $\text{Unif}[0,1]$:} Results (up to fourth decimal point) for $p=0.45$ (top) and $p=0.35$ (bottom) with $\hat{\lambda} = 1-c$ on points with true label = 1; points with true label = 0 follow the same trend by symmetry of the problem. In the table, \enquote{bad} (resp., \enquote{good}) points mean incorrectly (resp., correctly) labeled points, and A\ref{a3} is \Cref{a3}. Also, \enquote{pred.} is the predicted probability of the label being 1 and \enquote{Avg. pred. for bad points} (resp., \enquote{Avg. pred. for good points}) is the empirical average over all bad (resp., good) points with true label = 1; please note that this is with the actual Gram matrix. 
Under \Cref{a3}, all bad/good points have the same prediction (see Equations (\ref{eq:80-jan6}) and (\ref{eq:97-jan6}) or Lemmas \ref{thm-teacher} and \ref{thm-student}) due to which the corresponding columns do not have the word \enquote{Avg.}. 
\textit{Observe that as $n$ increases, the average prediction for both good and bad points (with the actual Gram matrix) matches the corresponding predictions under \Cref{a3} (and $n \to \infty$). Thus, \Cref{a3} is a reasonable assumption to analyze the average behavior of a linear model on a large dataset under random label corruption.}}
\label{tab-a3}
\end{table}

\begin{table}[!htb]
\centering
\begin{subtable}{\linewidth}\centering
\begin{tabular}{|c|c|c|c|c|c|}
%\toprule
\hline
\multirow{5}{*}{ \texttt{Teacher} } & $n$ & \makecell{Avg. pred. for \\ bad points} & \makecell{Pred. for \\ bad points under \\ A\ref{a3} \& $n \to \infty$} & \makecell{Avg. pred. for \\ good points} & \makecell{Pred. for \\ good points under \\ A\ref{a3} \& $n \to \infty$}
\\
\cline{2-6} & 1k & 0.6213 & \multirow{4}{*}{ \textbf{0.6222} } & 0.7324 & \multirow{4}{*}{ \textbf{0.7333} }
\\
\cline{2-3} \cline{5-5} & 5k & 0.6220 &  & 0.7332 &  
\\
\cline{2-3} \cline{5-5} & 10k & 0.6221 &  & 0.7332 & 
\\
\cline{2-3} \cline{5-5} & \textbf{50k} & \textbf{0.6222} &  & \textbf{0.7333} & 
\\
\hline
\end{tabular}
\begin{tabular}{|c|c|c|c|c|c|}
%\toprule
\hline
\multirow{5}{*}{ \texttt{Student} } & $n$ & \makecell{Avg. pred. for \\ bad points} & \makecell{Pred. for \\ bad points under \\ A\ref{a3} \& $n \to \infty$} & \makecell{Avg. pred. for \\ good points} & \makecell{Pred. for \\ good points under \\ A\ref{a3} \& $n \to \infty$}
\\
\cline{2-6} & 1k & 0.6895 & \multirow{4}{*}{ \textbf{0.6913} } & 0.7018 & \multirow{4}{*}{ \textbf{0.7037} }
\\
\cline{2-3} \cline{5-5} & 5k & 0.6910 &  & 0.7033 &  
\\
\cline{2-3} \cline{5-5} & 10k & 0.6911 &  & 0.7035 & 
\\
\cline{2-3} \cline{5-5} & \textbf{50k} & \textbf{0.6913} &  & \textbf{0.7037} & 
\\
\hline
\end{tabular}
\caption{$p=0.3$}
\vspace{0.2 cm}
\label{tab-a4-1}
\end{subtable}
\begin{subtable}{\linewidth}\centering
\begin{tabular}{|c|c|c|c|c|c|}
%\toprule
\hline
\multirow{5}{*}{ \texttt{Teacher} } & $n$ & \makecell{Avg. pred. for \\ bad points} & \makecell{Pred. for \\ bad points under \\ A\ref{a3} \& $n \to \infty$} & \makecell{Avg. pred. for \\ good points} & \makecell{Pred. for \\ good points under \\ A\ref{a3} \& $n \to \infty$}
\\
\cline{2-6} & 1k & 0.7097 & \multirow{4}{*}{ \textbf{0.7111} } & 0.8208 & \multirow{4}{*}{ \textbf{0.8222} }
\\
\cline{2-3} \cline{5-5} & 5k & 0.7108 &  & 0.8219 &  
\\
\cline{2-3} \cline{5-5} & 10k & 0.7109 &  & 0.8221 & 
\\
\cline{2-3} \cline{5-5} & \textbf{50k} & \textbf{0.7111} &  & \textbf{0.8222} & 
\\
\hline
\end{tabular}
\begin{tabular}{|c|c|c|c|c|c|}
%\toprule
\hline
\multirow{5}{*}{ \texttt{Student} } & $n$ & \makecell{Avg. pred. for \\ bad points} & \makecell{Pred. for \\ bad points under \\ A\ref{a3} \& $n \to \infty$} & \makecell{Avg. pred. for \\ good points} & \makecell{Pred. for \\ good points under \\ A\ref{a3} \& $n \to \infty$}
\\
\cline{2-6} & 1k & 0.7872 & \multirow{4}{*}{ \textbf{0.7901} } & 0.7995 & \multirow{4}{*}{ \textbf{0.8024} }
\\
\cline{2-3} \cline{5-5} & 5k & 0.7895 &  & 0.8019 &  
\\
\cline{2-3} \cline{5-5} & 10k & 0.7898 &  & 0.8021 & 
\\
\cline{2-3} \cline{5-5} & \textbf{50k} & \textbf{0.7901} &  & \textbf{0.8024} & 
\\
\hline
\end{tabular}
\caption{$p=0.2$}
\label{tab-a4-2}
\end{subtable}
\caption{\textbf{(ii) $\text{Bernoulli}(0.8)$:} Same as \Cref{tab-a3} except for $p=0.3$ (top) and $p=0.2$ (bottom) with $\hat{\lambda} = 2(1-c)$. Just like in \Cref{tab-a3}, \textit{as $n$ increases, the average prediction for both good and bad points (with the actual Gram matrix) matches the corresponding predictions under \Cref{a3} (and $n \to \infty$). Thus, \Cref{a3} is a reasonable assumption to analyze the average behavior of a linear model on a large dataset under random label corruption.}}
\label{tab-a4}
\end{table}

\clearpage

\section{Proof of Theorem~\ref{thm1}}
\label{pf-thm1}
\subsection{Step 1 in Detail}
\label{step-1}
The teacher's estimated parameter $\bm{\theta}_{\text{T}}^{\ast} := \text{arg min}_{\bm{\theta}} f_{\text{T}}(\bm{\theta})$ satisfies $\nabla f_{\text{T}}(\bm{\theta}_{\text{T}}^{\ast}) = \frac{1}{2n} \sum_{i=1}^{2n} \Big(\sigma(\langle \bm{\theta}_{\text{T}}^{\ast}, \phi(\bm{x}_i) \rangle) - \hat{y}_i\Big) \phi(\bm{x}_i) + \lambda \bm{\theta}_{\text{T}}^{\ast} = \vec{0}$. From this, we get:
\begin{equation}
    \label{eq:60-oct20}
    %\label{eq:59-main}
    \bm{\theta}_{\text{T}}^{\ast} = \sum_{i=1}^{2n} \underbrace{\frac{1}{2n\lambda} \Big(\hat{y}_i - \sigma(\langle \bm{\theta}_{\text{T}}^{\ast}, \phi(\bm{x}_i) \rangle)\Big)}_{:= \alpha_i} \phi(\bm{x}_i) = \sum_{i=1}^{2n} \alpha_i \phi(\bm{x}_i),
\end{equation}
for some real numbers $\{\alpha_i\}_{i=1}^{2n}$ which are known as the teacher's dual-space coordinates. Recall that we defined $\hat{\lambda} := 2 n  \lambda$ in the theorem statement.

\begin{lemma}[\textbf{Teacher's Dual-Space Coordinates and Predictions}]
\label{thm-teacher}
Suppose Assumptions \ref{a-ortho} and \ref{a3} hold. Then:
\begin{equation}
    \alpha_i = 
    \begin{cases}
    -\hat{\alpha} \text{ for } i \in \mathcal{S}_{1,\textup{bad}}, %\{1,\ldots,\hat{n}\},
    \\
    \alpha \text{ for } i \in \mathcal{S}_{1,\textup{good}}, %\{\hat{n}+1,\ldots,n\},
    \\
    \hat{\alpha} \text{ for } i \in \mathcal{S}_{0,\textup{bad}}, %\{n+1,\ldots,n+\hat{n}\},
    \\
    -\alpha \text{ for } i \in \mathcal{S}_{0,\textup{good}}, %\{n+\hat{n}+1,\ldots,2n\},
    \end{cases}
\end{equation}
where $\alpha \geq 0$ and $\hat{\alpha} \geq 0$ are obtained by jointly solving:
\begin{equation}
    \label{eq:71-oct20}
     \sigma\Big(c n \big(\alpha - (\alpha + \hat{\alpha})p) - (1-c)\hat{\alpha}\Big) = \hat{\lambda} \hat{\alpha},
\end{equation}
and
\begin{equation}
    \label{eq:72-oct20}
     \sigma\Big(c n \big(\alpha - (\alpha + \hat{\alpha})p) + (1-c){\alpha}\Big) = 1 - \hat{\lambda} {\alpha}.
\end{equation}
Also, the teacher's prediction for the $i^{\text{th}}$ sample, $y_i^{\textup{(T)}}$, turns out to be:
\begin{equation}
    \label{eq:80-oct20}
    y_i^{\textup{(T)}} = 
    \begin{cases}
    \hat{\lambda} \hat{\alpha} \text{ for } i \in \mathcal{S}_{1,\textup{bad}}, %\{1,\ldots,\hat{n}\},
    \\
    1 - \hat{\lambda} {\alpha} \text{ for } i \in \mathcal{S}_{1,\textup{good}}, %\{\hat{n}+1,\ldots,n\},
    \\
    1 - \hat{\lambda} \hat{\alpha} \text{ for } i \in \mathcal{S}_{0,\textup{bad}}, %\{n+1,\ldots,n+\hat{n}\},
    \\
    \hat{\lambda} {\alpha} \text{ for } i \in \mathcal{S}_{0,\textup{good}}. %\{n+\hat{n}+1,\ldots,2n\}.
    \end{cases}
\end{equation}
\end{lemma}
\noindent \Cref{thm-teacher} is proved next in \Cref{pf-thm-teacher}.
\\
\\
As mentioned in the proof sketch in the main text, we shall focus on the interesting case of:
\\
\textbf{(a)} $p$ being large enough so that the teacher misclassifies the incorrectly labeled points because otherwise, there is no need for SD, and
\\
\textbf{(b)} $\hat{\lambda}$ being chosen sensibly so that the teacher at least correctly classifies the correctly labeled points because otherwise, SD is hopeless.
\\
\\
Later in \Cref{step-3}, we shall impose a lower bound on $p$ (in terms of $c$ and $\hat{\lambda}$) so that (a) is ensured. Specifically, the teacher misclassifies the incorrectly labeled points (with indices $\mathcal{S}_{1,\textup{bad}} = \{1,\ldots,\hat{n}\}$ and $\mathcal{S}_{0,\textup{bad}} = \{n+1,\ldots,n+\hat{n}\}$) when 
\begin{equation}
    \label{eq:81}
    \hat{\lambda} \hat{\alpha} < \frac{1}{2}.
\end{equation}
Moreover, in \Cref{step-3}, we shall also restrict $\hat{\lambda}$ (in terms of $c$) so that (b) is ensured. Specifically, the teacher correctly classifies the correctly labeled points (with indices $\mathcal{S}_{1,\textup{good}} = \{\hat{n}+1,\ldots,n\}$ and $\mathcal{S}_{0,\textup{good}} = \{n+\hat{n}+1,\ldots,2n\}$) when
\begin{equation}
    \label{eq:82}
    1 - \hat{\lambda} {\alpha} > \frac{1}{2} \implies \hat{\lambda} {\alpha} < \frac{1}{2}.
\end{equation}

\subsection{Proof of Lemma~\ref{thm-teacher}}
\label{pf-thm-teacher}
\begin{proof}
From \cref{eq:60-oct20}, we have:
\begin{equation}
    \label{eq:62-new}
    2n \lambda \alpha_i = \hat{y}_i - \sigma\Big(\sum_{j=1}^{2n} \alpha_j \langle \phi(\bm{x}_j), \phi(\bm{x}_i) \rangle\Big),
\end{equation}
for all $i \in \{1,\ldots,2n\}$. For ease of notation, let us define $v_i := \sum_{j=1}^{2n} \alpha_j \langle \phi(\bm{x}_j), \phi(\bm{x}_i) \rangle$. Then, the above equation can be rewritten as:
\begin{equation}
    \label{eq:63}
    2n \lambda \alpha_i = \hat{y}_i - \sigma(v_i).
\end{equation}
Note here that the teacher's predictions are:
\begin{equation}
    \label{eq:64-new}
    y_i^{\text{(T)}} := \sigma(v_i) = \hat{y}_i - 2n \lambda \alpha_i,
\end{equation}
for $i \in \{1,\ldots,2n\}$.
Next, using Assumptions \ref{a-ortho} and \ref{a3}, we have:
\begin{equation}
    \label{eq:64}
    v_i = 
    \begin{cases}
    \alpha_i + c \sum_{j \in \{1,\ldots,n\} \setminus i} \alpha_j = \alpha_i(1-c) + c \sum_{j=1}^n \alpha_j \text{ for } i \in \{1,\ldots,n\},
    \\
    \alpha_i + c \sum_{j \in \{n+1,\ldots,2n\} \setminus i}^n \alpha_j = \alpha_i(1-c) + c \sum_{j=n+1}^{2n} \alpha_j \text{ for } i \in \{n+1,\ldots,2n\}.
    \end{cases}
\end{equation}
Let us focus on $i \in \{1,\ldots,n\}$. Let $S = \sum_{j=1}^n \alpha_j$. Then, we have the following equations:
\begin{equation}
    2n \lambda \alpha_i = - \sigma(\alpha_i(1-c) + c S) \text{ for } i \in \{1,\ldots,\hat{n}\},
\end{equation}
and
\begin{equation}
    2n \lambda \alpha_i = 1 - \sigma(\alpha_i(1-c) + c S) \text{ for } i \in \{\hat{n}+1,\ldots,n\}.
\end{equation}
Using the monotonicity of the sigmoid function, we conclude that: 
\begin{equation}
\alpha_i =
    \begin{cases}
     -\hat{\alpha} \text{ for } i \in \{1,\ldots,\hat{n}\}
    \\
    \alpha \text{ for } i \in \{\hat{n}+1,\ldots,n\},
    \end{cases}
\end{equation}
for some $\alpha, \hat{\alpha} \geq 0$. Using a similar argument, we can conclude that for $i \in \{n+1,\ldots,2n\}$:
\begin{equation}
\alpha_i =
    \begin{cases}
     \hat{\alpha}_2 \text{ for } i \in \{n+1,\ldots,n+\hat{n}\}
    \\
    -\alpha_2 \text{ for } i \in \{n+\hat{n}+1,\ldots,2n\},
    \end{cases}
\end{equation}
for some $\alpha_2, \hat{\alpha}_2 \geq 0$. We further claim that:
\begin{equation}
    \label{eq:69-new}
    \alpha_2 = \alpha \text{ and } \hat{\alpha}_2 = \hat{\alpha}.
\end{equation}
%Next, by the symmetry of the problem, we claim that the solution is of the following kind:
%\begin{equation}
%\alpha_i =
%    \begin{cases}
%     -\hat{\alpha} \text{ for } i \in \{1,\ldots,\hat{n}\}
%    \\
%    \alpha \text{ for } i \in \{\hat{n}+1,\ldots,n\}
%    \\
%    \hat{\alpha} \text{ for } i \in \{n+1,\ldots,n+\hat{n}\}
%    \\
%    -\alpha \text{ for } i \in \{n+\hat{n}+1,\ldots,2n\},
%    \end{cases}
%\end{equation}
%for some $\alpha, \tilde{\alpha} > 0$. 
Let us verify if this indeed holds up. Note that with such a solution:
\begin{equation}
    \sum_{j=1}^n \alpha_j = -\sum_{j=n+1}^{2n} \alpha_j = \alpha (n-\hat{n}) - \hat{\alpha} \hat{n} = \alpha n - (\alpha + \hat{\alpha})\hat{n}.
\end{equation}
Plugging this back in \cref{eq:64} for $i \in \{1,\ldots,n\}$ and then in \cref{eq:63}, we get (after a bit of rewriting):
\begin{equation}
    \label{eq:67}
     \sigma\big(-(1-c)\hat{\alpha} + c\alpha n - c(\alpha + \hat{\alpha})\hat{n}\big) = 2n \lambda \hat{\alpha}.
\end{equation}
\begin{equation}
    \label{eq:68}
     \sigma\big((1-c){\alpha} + c\alpha n - c(\alpha + \hat{\alpha})\hat{n}\big) = 1 - 2n \lambda {\alpha}.
\end{equation}
Doing the same but for $i \in \{n+1,\ldots,2n\}$ with $\alpha_2 = \alpha \text{ and } \hat{\alpha}_2 = \hat{\alpha}$, we get (again, after a bit of rewriting):
\begin{equation}
    \label{eq:69}
     \sigma\big((1-c)\hat{\alpha} - c\alpha n + c(\alpha + \hat{\alpha})\hat{n}\big) = 1 - 2n \lambda \hat{\alpha}.
\end{equation}
\begin{equation}
    \label{eq:70}
     \sigma\big(-(1-c){\alpha} - c\alpha n + c(\alpha + \hat{\alpha})\hat{n}\big) = 2n \lambda {\alpha}.
\end{equation}
Now note that \cref{eq:67} and \cref{eq:69}, and \cref{eq:68} and \cref{eq:70} are the same -- this is because $\sigma(-z) = 1 - \sigma(z)$ for all $z \in \mathbb{R}$. Thus, our claim in \cref{eq:69-new} is true.

Hence, we can consider only \cref{eq:67} and \cref{eq:68}, and solve them to find the two unknown variables ${\alpha}$ and $\hat{\alpha}$ in order to obtain $\bm{\theta}_{\text{T}}^{*}$. Recalling $\hat{n} = n p$, we can rewrite \cref{eq:67} and \cref{eq:68} as follows:
\begin{equation}
    \label{eq:71}
     \sigma\Big(c n \big(\alpha - (\alpha + \hat{\alpha})p) - (1-c)\hat{\alpha}\Big) = 2n \lambda \hat{\alpha}.
\end{equation}
\begin{equation}
    \label{eq:72}
     \sigma\Big(c n \big(\alpha - (\alpha + \hat{\alpha})p) + (1-c){\alpha}\Big) = 1 - 2n \lambda {\alpha}.
\end{equation}
Thus, we have:
\begin{equation}
    \label{eq:28-oct20}
    \alpha_i = 
    \begin{cases}
    -\hat{\alpha} \text{ for } i \in \{1,\ldots,\hat{n}\},
    \\
    \alpha \text{ for } i \in \{\hat{n}+1,\ldots,n\},
    \\
    \hat{\alpha} \text{ for } i \in \{n+1,\ldots,n+\hat{n}\},
    \\
    -\alpha \text{ for } i \in \{n+\hat{n}+1,\ldots,2n\},
    \end{cases}
\end{equation}
where $\alpha$ and $\hat{\alpha}$ are obtained by solving \cref{eq:71} and \cref{eq:72}.
\\
\\
From \cref{eq:64-new}, recall that the teacher's predictions for the $i^{\text{th}}$ sample is:
\begin{equation}
    \label{eq:78}
    y_i^{\text{(T)}} := \hat{y}_i - 2n \lambda \alpha_i.
\end{equation}
Now using \cref{eq:28-oct20} in \cref{eq:78}, we get:
\begin{equation}
    \label{eq:80}
    y_i^{\text{(T)}} = 
    \begin{cases}
    2n \lambda \hat{\alpha} \text{ for } i \in \{1,\ldots,\hat{n}\},
    \\
    1 - 2n \lambda {\alpha} \text{ for } i \in \{\hat{n}+1,\ldots,n\},
    \\
    1 - 2n \lambda \hat{\alpha} \text{ for } i \in \{n+1,\ldots,n+\hat{n}\},
    \\
    2n \lambda {\alpha} \text{ for } i \in \{n+\hat{n}+1,\ldots,2n\}.
    \end{cases}
\end{equation}
Replacing $2 n \lambda$ with $\hat{\lambda}$ in equations (\ref{eq:71}), (\ref{eq:72}) and (\ref{eq:80}), and plugging in $\mathcal{S}_{1,\textup{bad}} = \{1,\ldots,\hat{n}\}$,  $\mathcal{S}_{1,\textup{good}} = \{\hat{n}+1,\ldots,n\}$, $\mathcal{S}_{0,\textup{bad}}= \{n+1,\ldots,n+\hat{n}\}$ and $\mathcal{S}_{0,\textup{good}} = \{n+\hat{n}+1,\ldots,2n\}$ throughout finishes the proof.
\end{proof}

\subsection{Step 2 in Detail}
\label{step-2}
Just like \cref{eq:60-oct20} for the teacher, it can be shown that:
\begin{equation}
    \label{eq:85-oct20}
    \bm{\theta}_{\text{S}}^{\ast} = \sum_{i=1}^{2n} \beta_i \phi(\bm{x}_i),
\end{equation}
for some real numbers $\{\beta_i\}_{i=1}^{2n}$ which are known as the student's dual-space coordinates.

\begin{lemma}[\textbf{Student's Dual-Space Coordinates and Predictions}]
\label{thm-student}
Suppose Assumptions \ref{a-ortho} and \ref{a3} hold, and the teacher correctly classifies the correctly labeled points but misclassifies the incorrectly labeled points, i.e., 
$\hat{\lambda} {\alpha} < \frac{1}{2}$ and $\hat{\lambda} \hat{\alpha} < \frac{1}{2}$ in \Cref{thm-teacher}. Then:
\begin{equation}
    \label{eq:86-oct20}
    \beta_i = 
    \begin{cases}
    -\hat{\beta} \text{ for } i \in \mathcal{S}_{1,\textup{bad}}, %\{1,\ldots,\hat{n}\},
    \\
    \beta \text{ for } i \in \mathcal{S}_{1,\textup{good}}, %\{\hat{n}+1,\ldots,n\},
    \\
    \hat{\beta} \text{ for } i \in \mathcal{S}_{0,\textup{bad}}, %\{n+1,\ldots,n+\hat{n}\},
    \\
    -\beta \text{ for } i \in \mathcal{S}_{0,\textup{good}}, %\{n+\hat{n}+1,\ldots,2n\},
    \end{cases}
\end{equation}
where $\beta \geq 0$ and $\hat{\beta} \geq 0$ are obtained by jointly solving:
\begin{equation}
    \label{eq:87-oct20}
     \sigma\Big(c n \big(\beta - (\beta + \hat{\beta})p) - (1-c)\hat{\beta}\Big) = \hat{\lambda} \hat{\alpha} + \hat{\lambda} \hat{\beta},
\end{equation}
and
\begin{equation}
    \label{eq:88-oct20}
     \sigma\Big(c n \big(\beta - (\beta + \hat{\beta})p) + (1-c){\beta}\Big) = 1 - \hat{\lambda} {\alpha} - \hat{\lambda} {\beta}.
\end{equation}
Also, the student's prediction for the $i^{\text{th}}$ sample, $y_i^{\textup{(S)}}$, turns out to be:
\begin{equation}
    \label{eq:97-oct20}
    y_i^{\textup{(S)}} = 
    \begin{cases}
    \hat{\lambda} \hat{\alpha} + \hat{\lambda} \hat{\beta} \text{ for } i \in \mathcal{S}_{1,\textup{bad}}, %\{1,\ldots,\hat{n}\},
    \\
    1 - \hat{\lambda} {\alpha} - \hat{\lambda} {\beta} \text{ for } i \in \mathcal{S}_{1,\textup{good}}, %\{\hat{n}+1,\ldots,n\},
    \\
    1 - \hat{\lambda} \hat{\alpha} - \hat{\lambda} \hat{\beta} \text{ for } i \in \mathcal{S}_{0,\textup{bad}}, %\{n+1,\ldots,n+\hat{n}\},
    \\
    \hat{\lambda} {\alpha} + \hat{\lambda} {\beta} \text{ for } i \in \mathcal{S}_{0,\textup{good}}. %\{n+\hat{n}+1,\ldots,2n\}.
    \end{cases}
\end{equation}
\end{lemma}
\noindent We prove \Cref{thm-student} in \Cref{pf-thm-student}.
\\
\\
Now note that if $\hat{\lambda} \hat{\alpha} + \hat{\lambda} \hat{\beta} > \frac{1}{2}$ and $\hat{\lambda} {\alpha} + \hat{\lambda} {\beta} < \frac{1}{2}$, then the student has managed to correctly classify all the points in the training set. We ensure this in \Cref{step-3} by imposing an upper bound on $p$.

\subsection{Proof of Lemma~\ref{thm-student}}
\label{pf-thm-student}
\begin{proof}
The student's estimated parameter $\bm{\theta}_{\text{S}}^{\ast} = \text{arg min}_{\bm{\theta}} f_{\text{S}}(\bm{\theta})$ satisfies $\nabla f_{\text{S}}(\bm{\theta}_{\text{S}}^{\ast}) = \vec{0}$, from which we get:
\begin{equation}
    \label{eq:84}
    \bm{\theta}_{\text{S}}^{\ast} = \sum_{i=1}^{2n} \underbrace{\frac{1}{2n\lambda} \Big(y_i^{\text{(T)}} - \sigma(\langle \bm{\theta}_{\text{S}}^{\ast}, \phi(\bm{x}_i) \rangle)\Big)}_{:= \beta_i} \phi(\bm{x}_i).
\end{equation}
Thus the student's $i^{\text{th}}$ dual coordinate $\beta_i$ (as defined in \cref{eq:85-oct20}) satisfies:
\begin{equation}
    2n\lambda \beta_i = y_i^{\text{(T)}} - \sigma(\langle \bm{\theta}_{\text{S}}^{\ast}, \phi(\bm{x}_i) \rangle).
\end{equation}
By following the same approach as the one we took in the proof of \Cref{thm-teacher} for the teacher (with hard labels replaced by soft labels), we can show that:
\begin{equation}
    \label{eq:86}
    \beta_i = 
    \begin{cases}
    -\hat{\beta} \text{ for } i \in \{1,\ldots,\hat{n}\},
    \\
    \beta \text{ for } i \in \{\hat{n}+1,\ldots,n\},
    \\
    \hat{\beta} \text{ for } i \in \{n+1,\ldots,n+\hat{n}\},
    \\
    -\beta \text{ for } i \in \{n+\hat{n}+1,\ldots,2n\},
    \end{cases}
\end{equation}
where $\beta \in \mathbb{R}$ and $\hat{\beta} \in \mathbb{R}$ are obtained by solving the following two equations:
\begin{equation}
    \label{eq:87}
     \sigma\Big(c n \big(\beta - (\beta + \hat{\beta})p) - (1-c)\hat{\beta}\Big) = 2n \lambda \hat{\alpha} + 2n \lambda \hat{\beta},
\end{equation}
and
\begin{equation}
    \label{eq:88}
     \sigma\Big(c n \big(\beta - (\beta + \hat{\beta})p) + (1-c){\beta}\Big) = 1 - 2n \lambda {\alpha} - 2n \lambda {\beta}.
\end{equation}
We shall now show that $\beta \geq 0$ and $\hat{\beta} \geq 0$. We shall prove this by contradiction -- specifically, by showing that the other cases lead to a contradiction. 
\\
\\
\textbf{Case 1: $\beta \leq 0$ and $\hat{\beta} \leq 0$.} In this case:
\begin{equation}
    c n \big(\beta - (\beta + \hat{\beta})p) - (1-c)\hat{\beta} \geq c n \big(\beta - (\beta + \hat{\beta})p) + (1-c){\beta},
\end{equation}
which implies (by the increasing nature of the sigmoid function):
\begin{equation}
    \underbrace{\sigma\Big(c n \big(\beta - (\beta + \hat{\beta})p) - (1-c)\hat{\beta}\Big)}_{= 2n \lambda \hat{\alpha} + 2n \lambda \hat{\beta} \text{ from  \cref{eq:87}}} \geq \underbrace{\sigma\Big(c n \big(\beta - (\beta + \hat{\beta})p) + (1-c){\beta}\Big)}_{= 1 - 2n \lambda {\alpha} - 2n \lambda {\beta} \text{ from  \cref{eq:88}}}.
\end{equation}
Now using \cref{eq:87} and \cref{eq:88}, we get:
\begin{equation}
    2n \lambda \hat{\alpha} + 2n \lambda \hat{\beta} \geq 1 - 2n \lambda {\alpha} - 2n \lambda {\beta} \implies 2n \lambda \hat{\alpha} \geq 1 - 2n \lambda {\alpha} \underbrace{- 2n \lambda (\beta + \hat{\beta})}_{\geq 0} \implies 2n \lambda \hat{\alpha} \geq 1 - 2n \lambda {\alpha}.
\end{equation}
But this is a contradiction because as per \cref{eq:81} and \cref{eq:82}, we had:
\begin{equation}
    2n \lambda \hat{\alpha} < \frac{1}{2} \text{ and } 1 - 2n \lambda {\alpha} > \frac{1}{2} \implies 2n \lambda \hat{\alpha} < 1 - 2n \lambda {\alpha}.
\end{equation}
Hence, $\beta \leq 0$ and $\hat{\beta} \leq 0$ is not possible.
\\
\\
\textbf{Case 2: $\beta \geq 0$ and $\hat{\beta} \leq 0$.} In this case:
\begin{equation}
    c n \big(\beta - (\beta + \hat{\beta})p) - (1-c)\hat{\beta} \geq 0 \implies \underbrace{\sigma\Big(c n \big(\beta - (\beta + \hat{\beta})p) - (1-c)\hat{\beta}\Big)}_{= 2n \lambda \hat{\alpha} + 2n \lambda \hat{\beta} \text{ from  \cref{eq:87}}} \geq \frac{1}{2}.
\end{equation}
Using the above and \cref{eq:87}, we get that:
\begin{equation}
    2n \lambda \hat{\alpha} + \underbrace{2n \lambda \hat{\beta}}_{\leq 0} \geq \frac{1}{2} \implies 2n \lambda \hat{\alpha} \geq \frac{1}{2}.
\end{equation}
But this is again a contradiction as $2n \lambda \hat{\alpha} < \frac{1}{2}$ as per \cref{eq:81}. Hence, $\beta \geq 0$ and $\hat{\beta} \leq 0$ is also ruled out.
\\
\\
\textbf{Case 3: $\beta \leq 0$ and $\hat{\beta} \geq 0$.} In this case:
\begin{equation}
    c n \big(\beta - (\beta + \hat{\beta})p) + (1-c) {\beta} \leq 0 \implies \underbrace{\sigma\Big(c n \big(\beta - (\beta + \hat{\beta})p) + (1-c) {\beta}\Big)}_{= 1 - 2n \lambda {\alpha} - 2n \lambda {\beta} \text{ from  \cref{eq:88}}} \leq \frac{1}{2}.
\end{equation}
Using the above and \cref{eq:88}, we get that:
\begin{equation}
    1 - 2n \lambda {\alpha} - \underbrace{2n \lambda {\beta}}_{\leq 0} \leq \frac{1}{2} \implies 1 - 2n \lambda {\alpha} \leq \frac{1}{2}.
\end{equation}
But this is also a contradiction as $1 - 2n \lambda {\alpha} > \frac{1}{2}$ as per \cref{eq:82}. Hence, $\beta \leq 0$ and $\hat{\beta} \geq 0$ is also ruled out.
\\
\\
So, only $\beta \geq 0$ and $\hat{\beta} \geq 0$ is possible. 
Recall that $\beta$ and $\hat{\beta}$ are solutions to:
\begin{equation}
    \label{student-1}
     \sigma\Big(c n \big(\beta - (\beta + \hat{\beta})p) - (1-c)\hat{\beta}\Big) = 2n \lambda \hat{\alpha} + 2n \lambda \hat{\beta},
\end{equation}
and
\begin{equation}
    \label{student-2}
     \sigma\Big(c n \big(\beta - (\beta + \hat{\beta})p) + (1-c){\beta}\Big) = 1 - 2n \lambda {\alpha} - 2n \lambda {\beta}.
\end{equation}
Just like we obtained the teacher's predictions $\Big\{y_i^{\text{(T)}}\Big\}_{i=1}^{2n}$, the student's predictions are:
\begin{equation}
    \label{eq:97}
    y_i^{\text{(S)}} = 
    \begin{cases}
    2n \lambda \hat{\alpha} + 2n \lambda \hat{\beta} \text{ for } i \in \{1,\ldots,\hat{n}\},
    \\
    1 - 2n \lambda {\alpha} - 2n \lambda {\beta} \text{ for } i \in \{\hat{n}+1,\ldots,n\},
    \\
    1 - 2n \lambda \hat{\alpha} - 2n \lambda \hat{\beta} \text{ for } i \in \{n+1,\ldots,n+\hat{n}\},
    \\
    2n \lambda {\alpha} + 2n \lambda {\beta} \text{ for } i \in \{n+\hat{n}+1,\ldots,2n\}.
    \end{cases}
\end{equation}
Finally, replacing $2 n \lambda$ with $\hat{\lambda}$ in equations (\ref{student-1}), (\ref{student-2}) and (\ref{eq:97}), and plugging in $\mathcal{S}_{1,\textup{bad}} = \{1,\ldots,\hat{n}\}$,  $\mathcal{S}_{1,\textup{good}} = \{\hat{n}+1,\ldots,n\}$, $\mathcal{S}_{0,\textup{bad}}= \{n+1,\ldots,n+\hat{n}\}$ and $\mathcal{S}_{0,\textup{good}} = \{n+\hat{n}+1,\ldots,2n\}$ throughout gives us the desired result.
\end{proof}

\subsection{Step 3 in Detail}
\label{step-3}
\begin{proof}
Here, we shall obtain analytical expressions for the teacher's and student's predictions by solving 
\cref{eq:71-oct20} and \cref{eq:72-oct20} (in \Cref{thm-teacher}) for the teacher and then \cref{eq:87-oct20} and \cref{eq:88-oct20} (in \Cref{thm-student}) for the student.
%\cref{eq:71} and \cref{eq:72} for the teacher and then \cref{student-1} and \cref{student-2} for the student. 
Our approach will involve employing the first-order Maclaurin series expansion of the sigmoid function; specifically, we will use:
\begin{equation}
    \label{eq:100-sept7}
    \sigma(z) = \frac{1}{2} + \frac{z}{4} + {\varepsilon(z)},
\end{equation}
where $\varepsilon(z)$ is the residual error function. Note that:
\begin{equation}
\label{eq:101-sept7}
\varepsilon(z) 
    \begin{cases}
    < 0 \text{ for } z > 0 \text{ or equivalently when } \sigma(z) > \frac{1}{2} \\
    = 0 \text{ for } z = 0 \text{ or equivalently when } \sigma(z) = \frac{1}{2}\\
    > 0 \text{ for } z < 0 \text{ or equivalently when } \sigma(z) < \frac{1}{2}.
    \end{cases}
\end{equation}
It also holds that $\varepsilon(z)$ is a decreasing function. So,
\begin{equation}
    \label{eq:101-1-sept7}
    \sup_{z \in [-1,0]} \varepsilon(z) = \varepsilon(-1) < 0.02 \text{ }\text{ or equivalently } \sup_{z: \sigma(z) \in [\sigma(-1),0.5]} \varepsilon(z) < 0.02,
\end{equation}
and
\begin{equation}
    \label{eq:101-2-sept7}
    \inf_{z \in [0,1]} \varepsilon(z) = \varepsilon(1) > -0.02 \text{ }\text{ or equivalently } \inf_{z: \sigma(z) \in [0.5,\sigma(1)]} \varepsilon(z) > -0.02.
\end{equation}
Let us start with the \textbf{teacher}. Rewriting \cref{eq:71-oct20} and \cref{eq:72-oct20} while using the Maclaurin series expansion of the sigmoid function (from \cref{eq:100-sept7}) and the fact that $\sigma(-z) = 1 - \sigma(z)$ $\forall$ $z \in \mathbb{R}$, we have:
\begin{equation}
    \label{eq:102-sept7}
    \hat{\lambda} \hat{\alpha} = \sigma\Big(c n \big(\alpha - (\alpha + \hat{\alpha})p) - (1-c)\hat{\alpha}\Big) = \frac{1}{2} + \Bigg(\frac{c n \big(\alpha - (\alpha + \hat{\alpha})p) - (1-c)\hat{\alpha}}{4}\Bigg) + \varepsilon_1,
\end{equation}
and
\begin{equation}
    \label{eq:103-sept7}
    \hat{\lambda} {\alpha} = \sigma\Big(-c n \big(\alpha - (\alpha + \hat{\alpha})p) - (1-c){\alpha}\Big) = \frac{1}{2} - \Bigg(\frac{c n \big(\alpha - (\alpha + \hat{\alpha})p) + (1-c){\alpha}}{4}\Bigg) + \varepsilon_2,
\end{equation}
for some real numbers $\varepsilon_1, \varepsilon_2$. Solving the above two equations in the limit of $n \to \infty$, when $c = \Theta(1)$ and $\hat{\lambda} < \mathcal{O}(n)$ (this will be ensured subsequently), gives us:
\begin{equation}
    \label{eq:106-sept7}
    \lim_{n \to \infty} \alpha = \frac{p (1+\varepsilon_1+\varepsilon_2)}{\hat{\lambda} + \frac{1-c}{4}} \text{ and }  \lim_{n \to \infty} \hat{\alpha} = \frac{(1-p) (1+\varepsilon_1+\varepsilon_2)}{\hat{\lambda} + \frac{1-c}{4}}.
\end{equation}
%Now since $p < \frac{1}{2}$, $\lim_{n \to \infty} \hat{\alpha} > \lim_{n \to \infty} {\alpha}$.
Henceforth, we shall drop the $\lim_{n \to \infty}$ notation, and it is implied directly.
\\
\\
Let us now bound $\varepsilon_1+\varepsilon_2$ by imposing some more constraints. First, recall from \cref{eq:81} and \cref{eq:82} that we want $\hat{\lambda} \hat{\alpha} < \frac{1}{2}$ (i.e., the teacher does \textit{not} correctly classify the incorrectly labeled points) and $\hat{\lambda} {\alpha} < \frac{1}{2}$ (i.e., the teacher correctly classifies the correctly labeled points). Now since we are solving \cref{eq:102-sept7} and \cref{eq:103-sept7}, we must have $\hat{\lambda} \hat{\alpha} = \sigma\big(c n \big(\alpha - (\alpha + \hat{\alpha})p) - (1-c)\hat{\alpha}\big) < \frac{1}{2}$ and $\hat{\lambda} {\alpha} = \sigma\big(-c n \big(\alpha - (\alpha + \hat{\alpha})p) - (1-c){\alpha}\big) < \frac{1}{2}$; in this case, we must have that $\varepsilon_1 > 0$ and $\varepsilon_2 > 0$ from \cref{eq:101-sept7}. 
Next, we shall obtain upper bounds for $\varepsilon_1$ and $\varepsilon_2$. Using \cref{eq:103-sept7}, if $\sigma\Big(-c n \big(\alpha - (\alpha + \hat{\alpha})p) - (1-c){\alpha}\Big) = \hat{\lambda} \alpha > \sigma(-1)$, then $\varepsilon_2 < 0.02$ from \cref{eq:101-1-sept7}. Note that since $\varepsilon_1 + \varepsilon_2 > 0$ and $p < \frac{1}{2}$, $\hat{\alpha} > {\alpha}$. So if $\hat{\lambda} \alpha > \sigma(-1)$ holds, then so does $\hat{\lambda} \hat{\alpha} > \sigma(-1)$, in which case $\varepsilon_1 < 0.02$. But using the fact that $\varepsilon_1 + \varepsilon_2 > 0$, having
\begin{equation}
    \frac{\hat{\lambda} p}{\hat{\lambda} + \frac{1-c}{4}} > \sigma(-1),
\end{equation}
ensures $\hat{\lambda} \alpha > \sigma(-1)$ (as well as, $\hat{\lambda} \hat{\alpha} > \sigma(-1)$). Recalling that $r = \frac{(1-c)/4}{\hat{\lambda}}$ and using the fact that $\sigma(-1) = \frac{1}{1+e}$, we get:
\begin{equation}
    p > \frac{1+r}{1+e}.
\end{equation}
But we must also have $p < \frac{1}{2}$ due to which we should have $\frac{1+r}{1+e} < \frac{1}{2}$; this holds when:
\begin{equation}
    r = \frac{(1-c)/4}{\hat{\lambda}} < \frac{e-1}{2} \implies \hat{\lambda} > \frac{1-c}{2(e-1)}.
\end{equation}
The above two conditions can be evaluated and simplified a bit more to get:
\begin{equation}
    p > \frac{1+r}{3.7} \text{ and } r < 0.85 \text{ or } \hat{\lambda} > \frac{1-c}{3.4},
\end{equation}
and under these conditions, $\varepsilon_1 < 0.02$ and $\varepsilon_2 < 0.02$. Combining all this, \cref{eq:106-sept7} can be rewritten as (while also dropping the $\lim_{n \to \infty}$ notation):
\begin{equation}
    \label{eq:107-sept7}
    \alpha = \frac{p (1+\zeta)}{\hat{\lambda} + \frac{1-c}{4}} \text{ and }  \hat{\alpha} = \frac{(1-p) (1+\zeta)}{\hat{\lambda} + \frac{1-c}{4}},
\end{equation}
where $\zeta \in (0,0.04)$. Next, recall that we want $\hat{\lambda} {\alpha} < \frac{1}{2}$ and $\hat{\lambda} \hat{\alpha} < \frac{1}{2}$. Since, $\hat{\alpha} > {\alpha}$, both these conditions can be satisfied by just ensuring $\hat{\lambda} \hat{\alpha} < \frac{1}{2}$ which itself can be ensured by imposing:
\begin{equation}
    \frac{1.04 \hat{\lambda} (1-p)}{\hat{\lambda} + \frac{1-c}{4}} = \frac{1.04 (1-p)}{1 + r} < \frac{1}{2}.
\end{equation}
The above is obtained by making use of \cref{eq:107-sept7} and the fact that $\zeta < 0.04$. This gives us:
\begin{equation}
    p > 1 - \Big(\frac{1+r}{2.08}\Big).
\end{equation}
But again, we must have $p < \frac{1}{2}$ due to which we should also have $1 - \Big(\frac{1+r}{2.08}\Big) < \frac{1}{2}$; this holds when:
\begin{equation}
    r = \frac{(1-c)/4}{\hat{\lambda}} > 0.04 \implies \hat{\lambda} < \frac{1-c}{0.16}.
\end{equation}
So to recap, for the teacher, we have:
\begin{equation}
    \label{sept8-eq:115}
    \alpha = \frac{p (1+\zeta)}{\hat{\lambda} + \frac{1-c}{4}} \text{ and }  \hat{\alpha} = \frac{(1-p) (1+\zeta)}{\hat{\lambda} + \frac{1-c}{4}},
\end{equation}
where $\zeta \in (0,0.04)$, with $\hat{\lambda} \alpha < \hat{\lambda} \hat{\alpha} < \frac{1}{2}$ for $p > \max \Big(1 - \Big(\frac{1+r}{2.08}\Big), \frac{1+r}{3.7}\Big)$. All this is valid when $r \in \big(0.04, 0.85\big)$ or equivalently when $\hat{\lambda} \in \Big(\frac{1-c}{3.4}, \frac{1-c}{0.16}\Big)$.

{Let us do a sanity check to verify that the above range of $p$ ensures $\hat{\lambda} \alpha < \hat{\lambda} \hat{\alpha} < \frac{1}{2}$. First, we shall show that $\zeta = \varepsilon_1 + \varepsilon_2 \geq 0$ by contradiction; so suppose $\zeta < 0$. Then using \cref{eq:107-sept7}, we have $\hat{\lambda} \hat{\alpha} = \frac{(1+\zeta)(1-p)}{1+r} < \frac{1.04(1-p)}{1+r} < \frac{1}{2}$, where the last step follows because $p > 1 - \big(\frac{1+r}{2.08}\big)$. But if $\hat{\lambda} \hat{\alpha} < \frac{1}{2}$, we must have $\varepsilon_1 > 0$ (using \cref{eq:101-sept7}) as we are solving $\hat{\lambda} \hat{\alpha} = \sigma\big(c n \big(\alpha - (\alpha + \hat{\alpha})p\big) - (1-c)\hat{\alpha}\big)$. Similarly, we must also have $\varepsilon_2 > 0$ as $\hat{\lambda} {\alpha}$ is also $< \frac{1}{2}$ (which is easy to see because $0 < {\alpha} < \hat{\alpha}$ since $p < \frac{1}{2}$). But then $\zeta = \varepsilon_1 + \varepsilon_2 > 0$, which is a contradiction to our earlier supposition of $\zeta < 0$. Hence, we must have $\zeta \geq 0$. But then using \cref{eq:107-sept7}, we have $\hat{\lambda} {\alpha} = \frac{(1+\zeta)p}{1+r} > \frac{p}{1+r} > \sigma(-1)$, where the last step follows because $p > \frac{1+r}{3.7}$. But if $\hat{\lambda} {\alpha} > \sigma(-1)$, we must have $\varepsilon_2 < 0.02$ (using \cref{eq:101-1-sept7}) as we are solving $\hat{\lambda} {\alpha} = \sigma\big(-c n \big(\alpha - (\alpha + \hat{\alpha})p\big) - (1-c){\alpha}\big)$. Similarly, we must also have $\varepsilon_1 < 0.02$ as $\hat{\lambda} \hat{\alpha}$ is also $> \sigma(-1)$ (again, because ${\alpha} < \hat{\alpha}$). Combining all this, we get $\zeta = \varepsilon_1 + \varepsilon_2 < 0.04$. So, $\hat{\lambda} {\alpha} < \hat{\lambda} \hat{\alpha} = \frac{(1+\zeta)(1-p)}{1+r} < \frac{1.04(1-p)}{1+r} < \frac{1}{2}$, where the last step follows because $p > 1 - \big(\frac{1+r}{2.08}\big)$. So our prescribed range of $p$ indeed ensures $\hat{\lambda} \alpha < \hat{\lambda} \hat{\alpha} < \frac{1}{2}$.}
\\
\\
Let us now move onto the \textbf{student}. Rewriting \cref{eq:87-oct20} and \cref{eq:88-oct20} while using the Maclaurin series expansion of the sigmoid function (from \cref{eq:100-sept7}) and the fact that $\sigma(-z) = 1 - \sigma(z)$ $\forall$ $z \in \mathbb{R}$, we get:
\begin{equation}
    \label{eq:117-sept8}
     \hat{\lambda} \hat{\alpha} + \hat{\lambda} \hat{\beta} = \sigma\Big(c n \big(\beta - (\beta + \hat{\beta})p) - (1-c)\hat{\beta}\Big) =  \frac{1}{2} + \Bigg(\frac{c n \big(\beta - (\beta + \hat{\beta})p) - (1-c)\hat{\beta}}{4}\Bigg) + \varepsilon_3,
\end{equation}
and
\begin{equation}
    \label{eq:118-sept8}
     \hat{\lambda} {\alpha} + \hat{\lambda} {\beta} = \sigma\Big(- c n \big(\beta - (\beta + \hat{\beta})p) - (1-c){\beta}\Big) = \frac{1}{2} - \Bigg(\frac{c n \big(\beta - (\beta + \hat{\beta})p) + (1-c){\beta}}{4}\Bigg) + \varepsilon_4,
\end{equation}
for some real numbers $\varepsilon_3$ and $\varepsilon_4$. Solving the above two equations in the limit of $n \to \infty$ (when $c = \Theta(1)$ and $\hat{\lambda} < \mathcal{O}(n)$) while using the values of $\alpha$ and $\hat{\alpha}$ from \cref{sept8-eq:115}, we get:
\begin{equation}
    \label{eq:118-2-sept8}
    \lim_{n \to \infty} \beta = \frac{p}{\hat{\lambda} + \frac{1-c}{4}} \Big(-\frac{\hat{\lambda}(1+\zeta)}{\hat{\lambda} + \frac{1-c}{4}} + (1+\zeta')\Big) \text{ and } \lim_{n \to \infty} \hat{\beta} = \frac{1-p}{\hat{\lambda} + \frac{1-c}{4}} \Big(-\frac{\hat{\lambda}(1+\zeta)}{\hat{\lambda} + \frac{1-c}{4}} + (1+\zeta')\Big),
\end{equation}
with $\zeta' := \varepsilon_3 + \varepsilon_4$. Again, we shall drop the $\lim_{n \to \infty}$ notation subsequently, and it is implied directly. 
\\
\\
Next, we get:
\begin{equation}
    \label{eq:119-sept8}
    \alpha + \beta = \frac{p}{\hat{\lambda} + \frac{1-c}{4}} \Bigg(\frac{(\frac{1-c}{4})(1+\zeta)}{\hat{\lambda} + \frac{1-c}{4}} + (1+\zeta')\Bigg),
\end{equation}
and
\begin{equation}
    \label{eq:120-sept8}
    \hat{\alpha} + \hat{\beta} = \frac{1-p}{\hat{\lambda} + \frac{1-c}{4}} \Bigg(\frac{(\frac{1-c}{4})(1+\zeta)}{\hat{\lambda} + \frac{1-c}{4}} + (1+\zeta')\Bigg).
\end{equation}
Now, recall that if $\hat{\lambda}(\hat{\alpha} + \hat{\beta}) > \frac{1}{2}$ and $\hat{\lambda}({\alpha} + {\beta}) < \frac{1}{2}$, then the student has managed to correctly classify all the points in the training set. Let us first impose $\hat{\lambda}(\hat{\alpha} + \hat{\beta}) \in \big(\frac{1}{2}, \sigma(1)\big)$. Then, since we are solving  \cref{eq:117-sept8}, $\sigma\Big(c n \big(\beta - (\beta + \hat{\beta})p) - (1-c)\hat{\beta}\Big) \in \big(\frac{1}{2}, \sigma(1)\big)$, and so $\varepsilon_3 \in (-0.02,0)$ using \cref{eq:101-2-sept7}. Now, we shall be imposing $\hat{\lambda}({\alpha} + {\beta}) < \frac{1}{2}$. Additionally, we ensured earlier that $\hat{\lambda} \alpha > \sigma(-1)$ and showed in \Cref{thm-student} that $\beta \geq 0$. Therefore, we will have $\hat{\lambda}({\alpha} + {\beta}) \in \big(\sigma(-1), \frac{1}{2}\big)$. Since we are solving  \cref{eq:118-sept8}, $\sigma\Big(- c n \big(\beta - (\beta + \hat{\beta})p) - (1-c){\beta}\Big) \in \big(\sigma(-1), \frac{1}{2}\big)$, due to which $\varepsilon_4 \in (0,0.02)$ using \cref{eq:101-1-sept7}. Thus, $\zeta' = \varepsilon_3 + \varepsilon_4 \in (-0.02,0.02)$.

Now, using \cref{eq:119-sept8} and \cref{eq:120-sept8}, and plugging in $r = \frac{(1-c)/4}{\hat{\lambda}}$, we get:
\begin{equation}
    \label{eq:122-sept8}
    \hat{\lambda}({\alpha} + {\beta}) = \frac{p}{1 + r} \Big(\frac{r(1+\zeta)}{1 + r} + (1+\zeta')\Big),
\end{equation}
and
\begin{equation}
    \label{eq:122-2-sept8}
    \hat{\lambda}(\hat{\alpha} + \hat{\beta}) = \frac{1-p}{1 + r} \Big(\frac{r(1+\zeta)}{1 + r} + (1+\zeta')\Big),
\end{equation}
with $\zeta \in (0,0.04)$ and $\zeta' \in (-0.02,0.02)$. Let us first ensure $\hat{\lambda}(\hat{\alpha} + \hat{\beta}) \in \big(\frac{1}{2}, \sigma(1)\big)$. Using the bounds on $\zeta$ and $\zeta'$, this can be ensured by having:
\begin{equation}
    \frac{1-p}{1 + r} \Big(\frac{1.04 r}{1 + r} + 1.02\Big) < \sigma(1) = \frac{e}{1+e},
\end{equation}
and 
\begin{equation}
    \frac{1-p}{1 + r} \Big(\frac{r}{1 + r} + 0.98\Big) > \frac{1}{2}.
\end{equation}
Solving and simplifying the above two equations gives us:
\begin{equation}
    p \in \Big(1 -  \frac{0.7 (1+r)^2}{1 + 2r}, 1 - \frac{0.51 (1+r)^2}{1 + 2r}\Big).
\end{equation}
Note that:
\begin{equation}
    1 -  \frac{0.7 (1+r)^2}{1 + 2r} < 1 -  \frac{0.51 (1+r)^2}{1 + 2r} < \frac{1}{2}
\end{equation}
for all $r > 0$, and so we are good here. But recall that from the teacher's analysis (see the discussion after \cref{sept8-eq:115}), we had $p > \max \Big(1 - \Big(\frac{1+r}{2.08}\Big), \frac{1+r}{3.7}\Big)$. Combining everything, our current bound on $p$ is:
\begin{equation}
    \label{eq:128-sept8}
    p \in \Bigg(\max \Big(1 - \Big(\frac{1+r}{2.08}\Big), \frac{1+r}{3.7}, 1 -  \frac{0.7 (1+r)^2}{1 + 2r}\Big), 1 -  \frac{0.51 (1+r)^2}{1 + 2r} \Bigg).
\end{equation}
But the above is only meaningful when the lower bound on $p$ is smaller than the upper bound on it. So we must find the range of $r$ for which:
\[1 - \Big(\frac{1+r}{2.08}\Big) < 1 - \frac{0.51 (1+r)^2}{1 + 2r} \text{ and } \frac{1+r}{3.7} < 1 -  \frac{0.51 (1+r)^2}{1 + 2r}.\]
$1 -  \frac{0.7 (1+r)^2}{1 + 2r}$ is trivially smaller than $1 -  \frac{0.51 (1+r)^2}{1 + 2r}$ so we do not need to worry about that. Combining the range of $r$ obtained from the above equation with the previous range of $r \in (0.04,0.85)$ (that we obtained from the teacher), we get:
\begin{equation}
    r \in [0.07,0.54] \implies \hat{\lambda} \in \Big[\frac{1-c}{2.16}, \frac{1-c}{0.28}\Big].
\end{equation}
Finally, we need to ensure $\hat{\lambda} (\alpha + \beta) < \frac{1}{2}$. Using \cref{eq:122-sept8} and the bounds on $\zeta$ and $\zeta'$, this can be ensured by imposing:
\begin{equation}
    \frac{p}{1 + r} \Big(\frac{1.04 r}{1 + r} + 1.02\Big) < \frac{1}{2}.
\end{equation}
This can be simplified to:
\[p < \frac{0.485 (1+r)^2}{1+2r}.\]
But recall that we already have an upper bound on $p$ of $1 -  \frac{0.51 (1+r)^2}{1 + 2r}$. It can be checked that $1 -  \frac{0.51 (1+r)^2}{1 + 2r} < \frac{0.485 (1+r)^2}{1+2r}$ for $r \geq 0.08$. Thus, for $r \in [0.08,0.54] \text{ or } \hat{\lambda} \in \Big[\frac{1-c}{2.16}, \frac{1-c}{0.32}\Big]$, our bound on $p$ remains the same as \cref{eq:128-sept8}, i.e.,
\begin{equation}
    p \in \Bigg(\max \Big(1 - \Big(\frac{1+r}{2.08}\Big), \frac{1+r}{3.7}, 1 -  \frac{0.7 (1+r)^2}{1 + 2r}\Big), 1 -  \frac{0.51 (1+r)^2}{1 + 2r} \Bigg).
\end{equation}
Finally, to simplify our bound on $p$ a bit, we consider $r \in [0.10,0.54]$, where:
\begin{equation}
    \max \Big(1 - \Big(\frac{1+r}{2.08}\Big), \frac{1+r}{3.7}, 1 -  \frac{0.7 (1+r)^2}{1 + 2r}\Big) = \max \Big(1 - \Big(\frac{1+r}{2.08}\Big), \frac{1+r}{3.7}\Big).
\end{equation}
Thus, our final bound on $p$ is:
\begin{equation}
    p \in \Bigg(\max \Big(1 - \Big(\frac{1+r}{2.08}\Big), \frac{1+r}{3.7}\Big), 1 -  \frac{0.51 (1+r)^2}{1 + 2r} \Bigg),
\end{equation}
for
\begin{equation}
    r \in [0.10,0.54] \text{ or } \hat{\lambda} \in \Big[\frac{1-c}{2.16}, \frac{1-c}{0.40}\Big].
\end{equation}
Finally, note that the prescribed range of $\hat{\lambda}$ is $< \mathcal{O}(n)$ (as required in \cref{eq:106-sept7} and \cref{eq:118-2-sept8}) since $c = \Theta(1)$. So we are good here.

Also, since $n \to \infty$, the generalization gap (i.e., population accuracy - training accuracy) $\to 0$; see for e.g., the margin bounds (with $\ell_2$-regularization) in \cite{kakade2008complexity} where it is shown that the generalization gap goes down as $\mathcal{O}({1}/{\sqrt{n}})$. Therefore, the population accuracy of the student (resp., teacher) is the same as the training accuracy of the student (resp., teacher).
\\
\\
This finishes the proof.
\end{proof}

\section{Proof of Corollary~\ref{rmk-3}}
\label{pf-rmk-3}
\begin{proof}
From \cref{eq:80-oct20}, we have:
\begin{equation}
    \label{eq:120-oct29}
    \Delta_{\textup{T}} = 1 - \hat{\lambda} (\alpha + \hat{\alpha}),
\end{equation}
where $\hat{\lambda} = 2 n \lambda$. Similarly, using \cref{eq:97-oct20}, we have:
\begin{equation}
    \label{eq:121-oct29}
    \Delta_{\textup{S}} = 1 - \hat{\lambda} (\alpha + \beta + \hat{\alpha} + \hat{\beta}).
\end{equation}
Next, using \cref{eq:107-sept7} in \cref{eq:120-oct29}, we get:
\begin{equation}
    \label{eq:121-new-oct29}
    \Delta_{\textup{T}} = 1 - \Big(\frac{1+\zeta}{1+r}\Big),
\end{equation}
where $\zeta \in (0,0.04)$ and $r = \frac{(1-c)}{4 \hat{\lambda}}$. Similarly, using \cref{eq:122-sept8} and \cref{eq:122-2-sept8}, we get:
\begin{equation}
    \label{eq:122-oct29}
    \Delta_{\textup{S}} = 1 - \frac{1}{1 + r} \Big(\frac{r(1+\zeta)}{1 + r} + (1+\zeta')\Big),
\end{equation}
where $\zeta' \in (-0.02,0.02)$. Rewriting \cref{eq:122-oct29} slightly, we get:
\begin{flalign}
    \Delta_{\textup{S}} & = 1 - \Big(\frac{1+\zeta}{1 + r}\Big) \Big(\frac{r}{1 + r} + \frac{1+\zeta'}{1+\zeta}\Big)
    \\
    \label{eq:124-oct29}
    & \leq  1 - \Big(\frac{1+\zeta}{1 + r}\Big) \Big(\frac{0.1}{1.1} + \frac{0.98}{1.04}\Big)
    \\
    & < 1 - \Big(\frac{1+\zeta}{1 + r}\Big)
    \\
    & = \Delta_{\textup{T}}.
\end{flalign}
In \cref{eq:124-oct29}, we have used the fact that $r \geq 0.1$ (from the condition of \Cref{thm1}), $\zeta' \geq -0.02$ and $\zeta \leq 0.04$. 
\end{proof}

\section{More Empirical Results}
\label{more-expts-supp}
\subsection{Verifying Remark~\ref{rmk-xi} (Continued)}
\label{sec-expts-2-supp}
In \Cref{sec-expts-2}, we compared the performance of different values of $\xi$ with 50\% corruption. In \Cref{tab-expts2-supp-jan11}, we show results with 30\% corruption in Stanford Cars and Flowers-102\footnote{For Flowers-102, we include the provided validation set in the training set.} with the same weight decay value as in \Cref{sec-expts} (viz., $5 \times 10^{-4}$); even here, the improvement with $\xi > 1$ is more than that with $\xi \leq 1$. Again, the individual accuracies of the teacher and student and the experimental details are in \Cref{expt-detail}.

\begin{table}[!htb]
\centering 
\begin{subtable}{.5\linewidth}\centering
\begin{tabular}{|c|c|}
%\toprule
\hline
$\xi$ & \makecell{Improvement of student \\ over teacher (i.e., $\xi=0$)}
\\
\hline
0.2 & $0.89 \pm 0.15$ \%
\\
%\hline
0.5 & $2.15 \pm 0.06$ \%
\\
%\hline
0.7 & $2.75 \pm 0.10$ \%
\\
%\hline
1.0 & $3.32 \pm 0.11$ \%
\\
%\hline
\rowcolor{Gray}
\textbf{1.2} & $\bm{3.53 \pm 0.16}$ \%
\\
%\hline
1.5 & $3.38 \pm 0.12$ \% 
\\
%\hline
1.7 & $2.96 \pm 0.24$ \%
\\
%\hline
2.0 & $1.79 \pm 0.29$ \%
\\
\hline
\end{tabular}
\caption{30\% Random Corruption in Stanford Cars \\ with ResNet-34}
\label{tab:2d}
\end{subtable}%
\begin{subtable}{.5\linewidth}\centering
\begin{tabular}{|c|c|}
%\toprule
\hline
$\xi$ & \makecell{Improvement of student \\ over teacher (i.e., $\xi=0$)}
\\
\hline
0.5 & $-0.12 \pm 0.20$ \%
\\
%\hline
1.0 & $0.54 \pm 0.02$ \%
\\
%\hline
1.5 & $0.86 \pm 0.01$ \%
\\
%\hline
2.0 & $1.57 \pm 0.34$ \%
\\
%\hline
2.5 & $2.05 \pm 0.27$ \%
\\
%\hline
{3.0} & $2.49 \pm 0.25$ \%
\\
%\hline
{3.5} & $2.62 \pm 0.12$ \%
\\
%\hline
{4.0} & $2.87 \pm 0.09$ \%
\\
%\hline
{4.5} & $3.01 \pm 0.22$ \%
\\
%\hline
\rowcolor{Gray}
\textbf{5.0} & $\bm{3.21 \pm 0.07}$ \%
\\
%\hline
{5.5} & $2.94 \pm 0.33$ \%
\\
%\hline 
{6.0} & $3.06 \pm 0.09$ \%
\\
\hline 
\end{tabular}
\caption{30\% Adversarial Corruption in Flowers-102 \\ with ResNet-34}
\label{tab:2e}
\end{subtable}%
\vspace{-0.2 cm}
\caption{Average ($\pm$ 1 std.) improvement of student over teacher (i.e., student's test set accuracy - teacher's test set accuracy) with different values of the imitation parameter $\xi$. Just like in \Cref{tab-expts2}, note that the value of $\xi$ yielding the biggest improvement is more than 1.}
\label{tab-expts2-supp-jan11}
\end{table}

\subsection{Results with Other Weight Decay Values}
\label{new-expts=jan25}
All our previous results were with weight decay $=5 \times 10^{-4}$. Here, we verify Remarks \ref{rmk-xi} and \ref{dist-utility} for two other weight decay values which are $1 \times 10^{-3}$ and $1 \times 10^{-4}$.

\vspace{0.2 cm}
\noindent \textbf{(i) Verifying Remark~\ref{rmk-xi}:} In \Cref{new-expts-jan25-1-mini}, we list the student's improvement over the teacher (i.e., student's test accuracy - teacher's test accuracy) averaged across 3 different runs for different values of $\xi$ in the case of (a) Caltech-256 with 50\% random corruption \& weight decay $= 1 \times 10^{-4}$ and (b) CIFAR-100 with 50\% hierarchical corruption \& weight decay $= 1 \times 10^{-3}$. As was the case with weight decay $= 5 \times 10^{-4}$ in Tables \ref{tab-expts2} and \ref{tab-expts2-supp-jan11}, note that the value of $\xi$ yielding the biggest improvement here is also $> 1$. 

\begin{table}[!htb]
\centering 
\begin{subtable}{.5\linewidth}\centering
\begin{tabular}{|c|c|}
%\toprule
\hline
$\xi$ & \makecell{Improvement of student \\ over teacher}
\\
\hline
0.2 & $2.04 \pm 0.16$ \%
\\
%\hline
0.5 & $5.05 \pm 0.10$ \%
\\
%\hline
0.7 & $6.82 \pm 0.16$ \%
\\
%\hline
{1.0} & $9.07 \pm 0.17$ \%
\\
%\hline
{1.2} & $10.43 \pm 0.15$ \%
\\
%\hline
{1.5} & $11.78 \pm 0.16$ \%
\\
%\hline
{1.7} & $12.30 \pm 0.19$ \%
\\
%\hline
\rowcolor{Gray}
\textbf{2.0} & $\bm{13.07 \pm 0.20}$ \%
\\
%\hline
\rowcolor{Gray}
\textbf{2.2} & $\bm{12.89 \pm 0.43}$ \%
\\
%\hline
2.5 & $11.74 \pm 0.60$ \%
\\
\hline
\end{tabular}
\caption{\textbf{ResNet-34:} 50\% Random Corruption in \\ Caltech-256 with weight decay $=1 \times 10^{-4}$}
\label{tab:jan25-a}
\end{subtable}%
\begin{subtable}{.5\linewidth}\centering
\begin{tabular}{|c|c|}
%\toprule
\hline
$\xi$ & \makecell{Improvement of student \\ over teacher}
\\
\hline
0.2 & $0.39 \pm 0.09$ \% 
\\
%\hline
0.5 & $1.90 \pm 0.08$ \%
\\
%\hline
0.7 & $2.80 \pm 0.09$ \%
\\
%\hline
{1.0} & $3.82 \pm 0.04$ \%
\\
%\hline
{1.2} & $4.17 \pm 0.05$ \%
\\
%\hline
\rowcolor{Gray}
\textbf{1.5} & $\bm{4.51 \pm 0.07}$ \%
\\
%\hline
\rowcolor{Gray}
\textbf{1.7} & $\bm{4.56 \pm 0.02}$ \%
\\
%\hline
2.0 & $4.15 \pm 0.07$ \%
\\
\hline
\end{tabular}
\caption{\textbf{ResNet-34:} 50\% Hierarchical Corruption in \\ CIFAR-100 with weight decay $=1 \times 10^{-3}$}
\label{tab:jan25-b}
\end{subtable}%
\caption{Average ($\pm$ 1 std.) improvement of student over teacher (i.e., student's test set accuracy - teacher's test set accuracy) with different values of $\xi$. Just like with weight decay $=5 \times 10^{-4}$ (Tables \ref{tab-expts2} and \ref{tab-expts2-supp-jan11}), note that the value of $\xi$ yielding the biggest improvement with both weight decay values here is more than 1. This is consistent with our message in \Cref{rmk-xi}.}
\label{new-expts-jan25-1-mini}
\vspace{0.4 cm}
\end{table}

\vspace{0.2 cm}
\noindent \textbf{(ii)Verifying Remark~\ref{dist-utility}:} The setup is the same as \Cref{dist-utility-ver}, i.e., the student is trained with $\xi=1$. In \Cref{new-expts-jan25-2-mini}, we show the student's improvement over the teacher averaged across 3 different runs for varying degrees of label corruption in the case of (a) Caltech-256 with random corruption \& weight decay $= 1 \times 10^{-4}$ and (b) CIFAR-100 with hierarchical corruption \& weight decay $= 1 \times 10^{-3}$. As was the case with weight decay $= 5 \times 10^{-4}$ in \Cref{tab-expts1}, note that the improvement of the student (trained with $\xi=1$) over the teacher increases as the corruption level increases.

\begin{table}[!htb]
\begin{subtable}{\linewidth}\centering
\begin{tabular}{|c|c|}
%\toprule
\hline
Corruption level & \makecell{Improvement of student over teacher}
\\
\hline
0\% & $0.25 \pm 0.04$ \%
\\
%\hline
10\% & $1.39 \pm 0.10$ \%
\\
%\hline
30\% & $6.31 \pm 0.11$ \%
\\
%\hline
50\% & $9.07 \pm 0.17$ \%
\\
\hline
\end{tabular}
\caption{Random Corruption in Caltech-256 with weight decay $=1 \times 10^{-4}$}
\label{tab:jan25-c}
\end{subtable}%
\vspace{0.3 cm}
\begin{subtable}{\linewidth}\centering
\begin{tabular}{|c|c|}
%\toprule
\hline
Corruption level & \makecell{Improvement of student over teacher}
\\
\hline
0\% & $-0.73 \pm 0.09$ \%
\\
%\hline
10\% & $0.03 \pm 0.10$ \%
\\
%\hline
30\% & $1.77 \pm 0.19$ \% 
\\
%\hline
50\% & $3.82 \pm 0.04$ \%
\\
\hline
\end{tabular}
\caption{Hierarchical Corruption in CIFAR-100 with weight decay $=1 \times 10^{-3}$}
\label{tab:jan25-d}
\end{subtable}%
\caption{\textbf{ResNet-34 with $\xi=1$:} Average ($\pm$ 1 std.) improvement of student over teacher (i.e., student's test set accuracy - teacher's test set accuracy) with varying levels of label corruption. Just like with weight decay $= 5 \times 10^{-4}$ (\Cref{tab-expts1}), note that the improvement of the student over the teacher increases as the corruption level increases. This is consistent with our claim in \Cref{dist-utility}.}
\label{new-expts-jan25-2-mini}
\end{table}

\vspace{0.2 cm}
\noindent The individual accuracies of the teacher and student and the experimental details appear in \Cref{expt-detail}.

\clearpage

\section{Detailed Empirical Results}
\label{expt-detail}
We list the individual accuracies of the teacher and student (along with the student's improvement) corresponding to the results of \Cref{tab-expts2} in Tables \ref{supp-tab1-oct31}-\ref{supp-tab2-dec26}, \Cref{tab-expts2-supp-jan11} in Tables \ref{new1}-\ref{new2}, \Cref{tab-expts1} in Tables \ref{supp-tab1}-\ref{supp-tab2-1}, \Cref{new-expts-jan25-1-mini} in Tables \ref{new-expt-jan25-1a-full}-\ref{new-expt-jan25-1b-full} and \Cref{new-expts-jan25-2-mini} in Tables \ref{new-expt-jan25-2a-full}-\ref{new-expt-jan25-2b-full}.
\\
\\
\textbf{Experimental Details:} In all the cases, we use SGD with momentum = 0.9 and batch size = 128 for training. Since we are training only the softmax layer (i.e., doing logistic regression), we use an exponentially decaying learning rate scheme with decay parameter = 0.98 (for every epoch) and the initial learning rate is tuned\footnote{The tuning is done by picking the learning rate which yields the lowest training loss with the observed (noisy) labels. This is consistent with our theory setup where we assume convergence to the optimum of the training loss w.r.t. the observed labels.} over $\{0.001, 0.005, 0.01, 0.05, 0.1, 0.5\}$. The maximum number of epochs is 200. 

\vspace{0.3 cm}
\begin{table}[!htb]
\begin{minipage}{\textwidth}
  \centering
\begin{tabular}{|c|c|c|c|c|c|}
%\toprule
\hline
$\xi$ & Student's test acc. & \makecell{Improvement of student \\ over teacher}
\\
\hline
0.0 (=Teacher) & $57.61 \pm 0.03$ \% & $0$ \%  
\\
\hline
0.2 & $59.83 \pm 0.12$ \% & $2.22 \pm 0.12$ \% 
\\
\hline
0.5 & $62.79 \pm 0.04$ \% & $5.18 \pm 0.03$ \%
\\
\hline
0.7 & $64.45 \pm 0.09$ \% & $6.84 \pm 0.06$ \%
\\
\hline
1.0 & $66.15 \pm 0.27$ \% & $8.54 \pm 0.29$ \%
\\
\hline
{1.2} & ${67.27 \pm 0.25}$ \% & ${9.66 \pm 0.23}$ \%
\\
\hline
\textbf{1.5} & $\bm{67.65 \pm 0.54}$ \% & $\bm{10.04 \pm 0.51}$ \% 
\\
\hline
\textbf{1.7} & $\bm{67.42 \pm 0.58}$ \% & $\bm{9.81 \pm 0.55}$ \% 
\\
\hline
2.0 & $66.17 \pm 0.77$ \% & $8.56 \pm 0.73$ \% 
\\
\hline
\end{tabular}
%\end{table}
\end{minipage}
\caption{Detailed Version of \Cref{tab:2a} (50\% Random Corruption in Caltech-256 with ResNet-34)}
\label{supp-tab1-oct31}
\end{table}

\begin{table}[!htb]
\begin{minipage}{\textwidth}
  \centering
\begin{tabular}{|c|c|c|c|c|c|}
%\toprule
\hline
$\xi$ & Student's test acc. & \makecell{Improvement of student \\ over teacher}
\\
\hline
0.0 (=Teacher) & $61.15 \pm 0.09$  \% & $0$ \%  
\\
\hline
0.5 & $62.04 \pm 0.02$ \% & $0.89 \pm 0.10$ \%
\\
\hline
1.0 & $63.16 \pm 0.06$ \% & $2.01 \pm 0.14$ \%
\\
\hline
1.5 & $64.28 \pm 0.06$ \% & $3.13 \pm 0.11$ \%
\\
\hline
2.0 & $65.37 \pm 0.12$ \% & $4.22 \pm 0.20$ \%
\\
\hline
2.5 & $66.43 \pm 0.05$ \% & $5.28 \pm 0.13$ \%
\\
\hline
\textbf{3.0} & $\bm{66.93 \pm 0.03}$ \% & $\bm{5.78 \pm 0.12}$ \%
\\
\hline
\textbf{3.5} & $\bm{67.01 \pm 0.13}$ \% & $\bm{5.86 \pm 0.18}$ \%
\\
\hline
4.0 & $66.47 \pm 0.25$ \% & $5.32 \pm 0.33$ \%
\\
\hline
\end{tabular}
%\end{table}
\end{minipage}
\caption{Detailed Version of \Cref{tab:3a} (50\% Random Corruption in Caltech-256 with VGG-16)}
\label{supp-tab1-nov8}
\end{table}

\begin{table}[!htb]
\begin{minipage}{\textwidth}
  \centering
\begin{tabular}{|c|c|c|c|c|c|}
%\toprule
\hline
$\xi$ & Student's test acc. & \makecell{Improvement of student \\ over teacher}
\\
\hline
0.0 (=Teacher) & $50.80 \pm 0.04$ \% & $0$ \%  
\\
\hline
0.2 & $51.78 \pm 0.14$ \% & $0.98 \pm 0.12$ \%
\\
\hline
0.5 & $53.26 \pm 0.14$ \% & $2.46 \pm 0.11$ \%
\\
\hline
0.7 & $54.18 \pm 0.03$ \% & $3.38 \pm 0.02$ \%
\\
\hline
{1.0} & ${54.99 \pm 0.08}$ \% & ${4.19 \pm 0.09}$ \%
\\
\hline
\textbf{1.2} & $\bm{55.26 \pm 0.18}$ \% & $\bm{4.46 \pm 0.19}$ \%
\\
\hline
\textbf{1.5} & $\bm{55.26 \pm 0.15}$ \% & $\bm{4.46 \pm 0.17}$ \%
\\
\hline
\textbf{1.7} & $\bm{55.12 \pm 0.16}$ \% & $\bm{4.32 \pm 0.18}$ \%
\\
\hline
2.0 & $54.32 \pm 0.20$ \% & $3.52 \pm 0.23$ \%
\\
\hline
\end{tabular}
%\end{table}
\end{minipage}
\caption{Detailed Version of \Cref{tab:2b} (50\% Hierarchical Corruption in CIFAR-100 with ResNet-34)}
\label{supp-tab2-oct31}
\end{table}

\begin{table}[!htb]
\begin{minipage}{\textwidth}
  \centering
\begin{tabular}{|c|c|c|c|c|c|}
%\toprule
\hline
$\xi$ & Student's test acc. & \makecell{Improvement of student \\ over teacher}
\\
\hline
0.0 (=Teacher) & $41.60 \pm 0.08$ \% & $0$ \%
\\
\hline
0.2 & $42.70 \pm 0.03$ \% & $1.10 \pm 0.09$ \%
\\
\hline
0.5 & $44.29 \pm 0.06$ \% & $2.69 \pm 0.02$ \%
\\
\hline
0.7 & $45.32 \pm 0.05$ \% & $3.72 \pm 0.05$ \%
\\
\hline
1.0 & $46.89 \pm 0.05$ \% & $5.29 \pm 0.11$ \%
\\
\hline
1.2 & $47.86 \pm 0.06$ \% & $6.26 \pm 0.09$ \%
\\
\hline
\textbf{1.5} & $\bm{48.80 \pm 0.16}$ \% & $\bm{7.20 \pm 0.14}$ \%
\\
\hline
\textbf{1.7} & $\bm{48.83 \pm 0.18}$ \% & $\bm{7.23 \pm 0.17}$ \%
\\
\hline
2.0 & $48.02 \pm 0.25$ \% & $6.42 \pm 0.26$ \%
\\
\hline
\end{tabular}
%\end{table}
\end{minipage}
\caption{Detailed Version of \Cref{tab:3b} (50\% Hierarchical Corruption in CIFAR-100 with VGG-16)}
\label{supp-tab2-nov8}
\end{table}

\begin{table}[!htb]
\begin{minipage}{\textwidth}
  \centering
\begin{tabular}{|c|c|c|c|c|c|}
%\toprule
\hline
$\xi$ & Student's test acc. & \makecell{Improvement of student \\ over teacher}
\\
\hline
0.0 (=Teacher) & $48.93 \pm 0.08$ \% & 0 \%
\\
\hline
0.2 & $49.06 \pm 0.05$ \% & $0.13 \pm 0.08$ \%
\\
\hline
0.5 & $49.90 \pm 0.04$ \% & $0.97 \pm 0.04$ \% 
\\
\hline
0.7 & $50.38 \pm 0.09$ \% & $1.45 \pm 0.01$ \%
\\
\hline
\textbf{1.0} & $\bm{50.78 \pm 0.07}$ \% & $\bm{1.85 \pm 0.09}$ \% 
\\
\hline
\textbf{1.2} & $\bm{50.80 \pm 0.06}$ \% & $\bm{1.87 \pm 0.06}$ \%
\\
\hline
\textbf{1.5} & $\bm{50.79 \pm 0.03}$ \% & $\bm{1.86 \pm 0.08}$ \%
\\
\hline
\textbf{1.7} & $\bm{50.73 \pm 0.04}$ \% & $\bm{1.80 \pm 0.05}$ \%
\\
\hline
2.0 & $50.46 \pm 0.09$ \% & $1.53 \pm 0.02$ \%
\\
\hline
\end{tabular}
%\end{table}
\end{minipage}
\caption{Detailed Version of \Cref{tab:2c} (50\% Adversarial Corruption in Food-101 with ResNet-34)}
\label{supp-tab1-dec26}
\end{table}

\begin{table}[!htb]
\begin{minipage}{\textwidth}
  \centering
\begin{tabular}{|c|c|c|c|c|c|}
%\toprule
\hline
$\xi$ & Student's test acc. & \makecell{Improvement of student \\ over teacher}
\\
\hline
0.0 (=Teacher) & $37.01 \pm 0.46$ \% & 0 \%
\\
\hline
0.2 & $37.80 \pm 0.23$ \% & $0.79 \pm 0.23$ \%
\\
\hline
0.5 & $39.15 \pm 0.37$ \% & $2.14 \pm 0.09$ \%
\\
\hline
0.7 & $39.97 \pm 0.42$ \% & $2.96 \pm 0.04$ \%
\\
\hline
{1.0} & $40.86 \pm 0.51$ \% & $3.85 \pm 0.05$ \%
\\
\hline
\textbf{1.2} & $\bm{41.23 \pm 0.60}$ \% & $\bm{4.22 \pm 0.15}$ \%
\\
\hline
\textbf{1.5} & $\bm{41.40 \pm 0.71}$ \% & $\bm{4.39 \pm 0.29}$ \%
\\
\hline
\textbf{1.7} & $\bm{41.21 \pm 0.76}$ \% & $\bm{4.20 \pm 0.34}$ \%
\\
\hline
2.0 & $40.54 \pm 0.90$ \% & $3.53 \pm 0.49$ \%
\\
\hline
\end{tabular}
%\end{table}
\end{minipage}
\caption{Detailed Version of \Cref{tab:3c} (50\% Adversarial Corruption in Food-101 with VGG-16)}
\label{supp-tab2-dec26}
\end{table}

\begin{table}[!htb]
\begin{minipage}{\textwidth}
  \centering
\begin{tabular}{|c|c|c|c|c|c|}
%\toprule
\hline
$\xi$ & Student's test acc. & \makecell{Improvement of student \\ over teacher}
\\
\hline
0.0 (=Teacher) & $25.01 \pm 0.20$ \% & 0 \%
\\
\hline
0.2 & $25.90 \pm 0.09$ \% & $0.89 \pm 0.15$ \%
\\
\hline
0.5 & $27.16 \pm 0.16$ \% & $2.15 \pm 0.06$ \%
\\
\hline
0.7 & $27.76 \pm 0.21$ \% & $2.75 \pm 0.10$ \%
\\
\hline
1.0 & $28.33 \pm 0.14$ \% & $3.32 \pm 0.11$ \%
\\
\hline
\textbf{1.2} & $\bm{28.54 \pm 0.11}$ \% & $\bm{3.53 \pm 0.16}$ \%
\\
\hline
1.5 & $28.39 \pm 0.10$ \% & $3.38 \pm 0.12$ \% 
\\
\hline
1.7 & $27.97 \pm 0.20$ \% & $2.96 \pm 0.24$ \%
\\
\hline
2.0 & $26.80 \pm 0.27$ \% & $1.79 \pm 0.29$ \%
\\
\hline
\end{tabular}
%\end{table}
\end{minipage}
\caption{Detailed Version of \Cref{tab:2d} (30\% Random Corruption in Stanford Cars with ResNet-34)}
\label{new1}
\end{table}

\begin{table}[!htb]
\begin{minipage}{\textwidth}
  \centering
\begin{tabular}{|c|c|c|c|c|c|}
%\toprule
\hline
$\xi$ & Student's test acc. & \makecell{Improvement of student \\ over teacher}
\\
\hline
0.0 (=Teacher) & $50.34 \pm 0.23$ \% & $0$ \%  
\\
\hline
0.5 & $50.22 \pm 0.23$ \% & $-0.12 \pm 0.20$ \%
\\
\hline
1.0 & $50.88 \pm 0.24$ \% & $0.54 \pm 0.02$ \%
\\
\hline
1.5 & $51.20 \pm 0.22$ \% & $0.86 \pm 0.01$ \%
\\
\hline
2.0 & $51.91 \pm 0.41$ \% & $1.57 \pm 0.34$ \%
\\
\hline
2.5 & $52.39 \pm 0.44$ \% & $2.05 \pm 0.27$ \%
\\
\hline
{3.0} & $52.83 \pm 0.28$ \% & $2.49 \pm 0.25$ \%
\\
\hline
{3.5} & $52.96 \pm 0.28$ \% & $2.62 \pm 0.12$ \%
\\
\hline
{4.0} & $53.21 \pm 0.31$ \% & $2.87 \pm 0.09$ \%
\\
\hline
{4.5} & $53.35 \pm 0.15$ \% & $3.01 \pm 0.22$ \%
\\
\hline
\textbf{5.0} & $\bm{53.55 \pm 0.25}$ \% & $\bm{3.21 \pm 0.07}$ \%
\\
\hline
{5.5} & $53.28 \pm 0.50$ \% & $2.94 \pm 0.33$ \%
\\
\hline 
{6.0} & $53.40 \pm 0.28$ \% & $3.06 \pm 0.09$ \%
\\
\hline 
\end{tabular}
%\end{table}
\end{minipage}
\caption{Detailed Version of \Cref{tab:2e} (30\% Adversarial Corruption in Flowers-102 with ResNet-34)}
\label{new2}
\end{table}

\begin{table}[!htb]
\begin{minipage}{\textwidth}
  \centering
\begin{tabular}{|c|c|c|c|c|c|}
%\toprule
\hline
Corruption level & Teacher's test acc. & Student's test acc. & \makecell{Improvement of student \\ over teacher}
\\
\hline
0\% & $83.97 \pm 0.10$ \% & $83.93 \pm 0.12$ \% & $-0.04 \pm 0.02$ \%
\\
\hline
10\% & $77.86 \pm 0.14$ \% & $80.37 \pm 0.04$ \% & $2.51 \pm 0.11$ \%
\\
\hline
30\% & $68.09 \pm 0.21$ \% & $74.23 \pm 0.08$ \% & $6.14 \pm 0.16$ \%
\\
\hline
50\% & $57.61 \pm 0.03$ \% & $66.15 \pm 0.27$ \% & $8.54 \pm 0.29$ \%
\\
\hline
\end{tabular}
%\end{table}
\end{minipage}
\caption{Detailed Version of Random Corruption in Caltech-256 with ResNet-34 and $\xi=1$ (\Cref{tab:1a})}
\label{supp-tab1}
\end{table}

\begin{table}[!htb]
\begin{minipage}{\textwidth}
  \centering
\begin{tabular}{|c|c|c|c|c|c|}
%\toprule
\hline
Corruption level & Teacher's test acc. & Student's test acc. & \makecell{Improvement of student \\ over teacher}
\\
\hline
0\% & $83.97 \pm 0.10$ \% & $83.93 \pm 0.12$ \% & $-0.04 \pm 0.02$ \%
\\
\hline
10\% & $77.01 \pm 0.23$ \% & $79.33 \pm 0.13$ \% & $2.32 \pm 0.10$ \%
\\
\hline
30\% & $64.21 \pm 0.36$ \% & $69.29 \pm 0.12$ \% & $5.08 \pm 0.25$ \%
\\
\hline
50\% & $48.66 \pm 0.10$ \% & $54.43 \pm 0.29$ \% & $5.77 \pm 0.19$ \%
\\
\hline
\end{tabular}
%\end{table}
\end{minipage}
\caption{Detailed Version of Adversarial Corruption in Caltech-256 with ResNet-34 and $\xi=1$ (\Cref{tab:1a})}
\label{supp-tab1-1}
\end{table}

\begin{table}[!htb]
\begin{minipage}{\textwidth}
  \centering
\begin{tabular}{|c|c|c|c|c|c|}
%\toprule
\hline
Corruption level & Teacher's test acc. & Student's test acc. & \makecell{Improvement of student \\ over teacher}
\\
\hline
0\% & $72.77 \pm 0.07$ \% & $72.54 \pm 0.07$ \% & $-0.23 \pm 0.06$ \%
\\
\hline
10\% & $70.57 \pm 0.14$ \% & $71.20 \pm 0.02$ \% & $0.63 \pm 0.11$ \%
\\
\hline
30\% & $66.80 \pm 0.06$ \% & $68.14 \pm 0.07$ \% & $1.34 \pm 0.13$ \%
\\
\hline
50\% & $62.47 \pm 0.10$ \% & $64.58 \pm 0.10$ \% & $2.11 \pm 0.15$ \%
\\
\hline
\end{tabular}
%\end{table}
\end{minipage}
\caption{Detailed Version of Random Corruption in CIFAR-100 with ResNet-34 and $\xi=1$ (\Cref{tab:1c})}
\label{supp-tab3}
\end{table}

\begin{table}[!htb]
\begin{minipage}{\textwidth}
  \centering
\begin{tabular}{|c|c|c|c|c|c|}
%\toprule
\hline
Corruption level & Teacher's test acc. & Student's test acc. & \makecell{Improvement of student \\ over teacher}
\\
\hline
0\% & $72.77 \pm 0.07$ \% & $72.54 \pm 0.07$ \% & $-0.23 \pm 0.06$ \%
\\
\hline
10\% & $69.39 \pm 0.09$ \% & $70.58 \pm 0.08$ \% & $1.19 \pm 0.08$ \%
\\
\hline
30\% & $62.18 \pm 0.12$ \% & $64.98 \pm 0.10$ \% & $2.80 \pm 0.06$ \%
\\
\hline
50\% & $50.80 \pm 0.04$ \% & $54.99 \pm 0.08$ \% & $4.19 \pm 0.09$ \%
\\
\hline
\end{tabular}
%\end{table}
\end{minipage}
\caption{Detailed Version of Hierarchical Corruption in CIFAR-100 with ResNet-34 and $\xi=1$ (\Cref{tab:1c})}
\label{supp-tab5}
\end{table}

\begin{table}[!htb]
\begin{minipage}{\textwidth}
  \centering
\begin{tabular}{|c|c|c|c|c|c|}
%\toprule
\hline
Corruption level & Teacher's test acc. & Student's test acc. & \makecell{Improvement of student \\ over teacher}
\\
\hline
0\% & $63.65 \pm 0.08$ \% & $63.28 \pm 0.04$ \% & $-0.37 \pm 0.10$ \%
\\
\hline
10\% & $62.44 \pm 0.03$ \% & $62.54 \pm 0.03$ \% & $0.10 \pm 0.04$ \%
\\
\hline
30\% & $59.38 \pm 0.20$ \% & $59.85 \pm 0.18$ \% & $0.47 \pm 0.04$ \%
\\
\hline
50\% & $54.76 \pm 0.13$ \% & $55.88 \pm 0.06$ \% & $1.12 \pm 0.08$ \%
\\
\hline
\end{tabular}
%\end{table}
\end{minipage}
\caption{Detailed Version of Random Corruption in Food-101 with ResNet-34 and $\xi=1$ (\Cref{tab:1b})}
\label{supp-tab2}
\end{table}

\begin{table}[!htb]
\begin{minipage}{\textwidth}
  \centering
\begin{tabular}{|c|c|c|c|c|c|}
%\toprule
\hline
Corruption level & Teacher's test acc. & Student's test acc. & \makecell{Improvement of student \\ over teacher}
\\
\hline
0\% & $63.65 \pm 0.08$ \% & $63.28 \pm 0.04$ \% & $-0.37 \pm 0.10$ \%
\\
\hline
10\% & $61.92 \pm 0.13$ \% & $62.16 \pm 0.11$ \% & $0.25 \pm 0.05$ \%
\\
\hline
30\% & $57.03 \pm 0.16$ \% & $57.80 \pm 0.22$ \% & $0.77 \pm 0.06$ \%
\\
\hline
50\% & %$48.41 \pm 0.11$ \% & $50.03 \pm 0.18$ \% & $1.62 \pm 0.09$ \%
$48.93 \pm 0.08$ \% & $50.78 \pm 0.07$ \% & $1.85 \pm 0.09$ \%
\\
\hline
\end{tabular}
%\end{table}
\end{minipage}
\caption{Detailed Version of Adversarial Corruption in Food-101 with ResNet-34 and $\xi=1$ (\Cref{tab:1b})}
\label{supp-tab2-1}
\end{table}

\begin{table}[!htb]
\begin{minipage}{\textwidth}
  \centering
\begin{tabular}{|c|c|c|c|c|c|}
%\toprule
\hline
$\xi$ & Student's test acc. & \makecell{Improvement of student \\ over teacher}
\\
\hline
0.0 (=Teacher) & $39.78 \pm 0.18$ \% & 0 \%
\\
\hline
0.2 & $41.82 \pm 0.06$ \% & $2.04 \pm 0.16$ \%
\\
\hline
0.5 & $44.83 \pm 0.11$ \% & $5.05 \pm 0.10$ \%
\\
\hline
0.7 & $46.60 \pm 0.09$ \% & $6.82 \pm 0.16$ \%
\\
\hline
{1.0} & $48.85 \pm 0.06$ \% & $9.07 \pm 0.17$ \%
\\
\hline
{1.2} & $50.21 \pm 0.03$ \% & $10.43 \pm 0.15$ \%
\\
\hline
{1.5} & $51.56 \pm 0.10$ \% & $11.78 \pm 0.16$ \%
\\
\hline
{1.7} & $52.08 \pm 0.03$ \% & $12.30 \pm 0.19$ \%
\\
\hline
\textbf{2.0} & $\bm{52.85 \pm 0.05}$ \% & $\bm{13.07 \pm 0.20}$ \%
\\
\hline
\textbf{2.2} & $\bm{52.67 \pm 0.33}$ \% & $\bm{12.89 \pm 0.43}$ \%
\\
\hline
2.5 & $51.52 \pm 0.51$ \% & $11.74 \pm 0.60$ \%
\\
\hline
\end{tabular}
%\end{table}
\end{minipage}
\caption{Detailed Version of \Cref{tab:jan25-a} (50\% Random Corruption in Caltech-256 w/ ResNet-34 and wt. decay $=1 \times 10^{-4}$)}
\label{new-expt-jan25-1a-full}
\end{table}

\begin{table}[!htb]
\begin{minipage}{\textwidth}
  \centering
\begin{tabular}{|c|c|c|c|c|c|}
%\toprule
\hline
$\xi$ & Student's test acc. & \makecell{Improvement of student \\ over teacher}
\\
\hline
0.0 (=Teacher) & $54.71 \pm 0.05$ \% & 0 \%
\\
\hline
0.2 & $55.10 \pm 0.04$ \% & $0.39 \pm 0.09$ \% 
\\
\hline
0.5 & $56.61 \pm 0.04$ \% & $1.90 \pm 0.08$ \%
\\
\hline
0.7 & $57.51 \pm 0.04$ \% & $2.80 \pm 0.09$ \%
\\
\hline
{1.0} & $58.53 \pm 0.05$ \% & $3.82 \pm 0.04$ \%
\\
\hline
{1.2} & $58.88 \pm 0.03$ \% & $4.17 \pm 0.05$ \%
\\
\hline
\textbf{1.5} & $\bm{59.22 \pm 0.11}$ \% & $\bm{4.51 \pm 0.07}$ \%
\\
\hline
\textbf{1.7} & $\bm{59.27 \pm 0.06}$ \% & $\bm{4.56 \pm 0.02}$ \%
\\
\hline
2.0 & $58.86 \pm 0.02$ \% & $4.15 \pm 0.07$ \%
\\
\hline
\end{tabular}
%\end{table}
\end{minipage}
\caption{Detailed Version of \Cref{tab:jan25-b} (50\% Hierarchical Corruption in CIFAR-100 w/ ResNet-34 and wt. decay $=1 \times 10^{-3}$)}
\label{new-expt-jan25-1b-full}
\end{table}

\clearpage

\begin{table}[!htb]
\begin{minipage}{\textwidth}
  \centering
\begin{tabular}{|c|c|c|c|c|c|}
%\toprule
\hline
Corruption level & Teacher's test acc. & Student's test acc. & \makecell{Improvement of student \\ over teacher}
\\
\hline
0\% & $82.95 \pm 0.02$ \% & $83.20 \pm 0.04$ \% & $0.25 \pm 0.04$ \%
\\
\hline
10\% & $74.29 \pm 0.12$ \% & $75.68 \pm 0.03$ \% & $1.39 \pm 0.10$ \%
\\
\hline
30\% & $54.65 \pm 0.13$ \% & $60.96 \pm 0.18$ \% & $6.31 \pm 0.11$ \%
\\
\hline
50\% & $39.78 \pm 0.18$ \% & $48.85 \pm 0.06$ \% & $9.07 \pm 0.17$ \%
\\
\hline
\end{tabular}
%\end{table}
\end{minipage}
\caption{Detailed Version of \Cref{tab:jan25-c} (Random Corruption in Caltech-256 w/ ResNet-34, $\xi=1$ and wt. decay $=1 \times 10^{-4}$)}
\label{new-expt-jan25-2a-full}
\end{table}

\begin{table}[!htb]
\begin{minipage}{\textwidth}
  \centering
\begin{tabular}{|c|c|c|c|c|c|}
%\toprule
\hline
Corruption level & Teacher's test acc. & Student's test acc. & \makecell{Improvement of student \\ over teacher}
\\
\hline
0\% & $72.99 \pm 0.09$ \% & $72.26 \pm 0.01$ \% & $-0.73 \pm 0.09$ \%
\\
\hline
10\% & $70.59 \pm 0.04$ \% & $70.62 \pm 0.07$ \% & $0.03 \pm 0.10$ \%
\\
\hline
30\% & $64.64 \pm 0.12$ \% & $66.41 \pm 0.09$ \% & $1.77 \pm 0.19$ \% 
\\
\hline
50\% & $54.71 \pm 0.05$ \% & $58.53 \pm 0.05$ \% & $3.82 \pm 0.04$ \%
\\
\hline
\end{tabular}
%\end{table}
\end{minipage}
\caption{Detailed Version of \Cref{tab:jan25-d} (Hierarchical Corruption in CIFAR-100 w/ ResNet-34, $\xi=1$ and wt. decay $=1 \times 10^{-3}$)}
\label{new-expt-jan25-2b-full}
\end{table}

\end{document}